\documentclass{article}
\usepackage{fullpage}
\usepackage{smile}
\RequirePackage{natbib}

\usepackage[pdftex,colorlinks,linkcolor=blue,citecolor=blue,filecolor=blue,urlcolor=blue]{hyperref}      
\usepackage{url}           
\usepackage{booktabs}      
\usepackage{amsfonts}       
\usepackage{nicefrac}       
\usepackage{xcolor}

\usepackage{subfigure}
\usepackage{multirow}
\usepackage{color}
\usepackage{amsmath, amsthm, amssymb, amsfonts}

\usepackage{xspace}
\usepackage{comment}
\usepackage{dsfont}
\usepackage{subfigure}

\usepackage{algorithm}
\usepackage{algorithmic}
\usepackage{multirow}
\RequirePackage{times}
\usepackage{comment}
\usepackage{xcolor}

\usepackage{authblk}
\author[1]{Yunfan Li}
\author[2]{Lin Yang}

\affil[1,2]{University of California, Los Angeles}
{
	\makeatletter
	\renewcommand\AB@affilsepx{, \protect\Affilfont}
	\makeatother
	\affil[1]{yunfanli@g.ucla.edu}
	\affil[2]{linyang@ee.ucla.edu}
	
}

\title{ On the Model-Misspecification in Reinforcement Learning}

\begin{document}

	\date{}
	\maketitle

\begin{abstract}
The success of reinforcement learning (RL) crucially depends on effective function approximation when dealing with complex ground-truth models. Existing sample-efficient RL algorithms primarily employ three approaches to function approximation: policy-based, value-based, and model-based methods. However, in the face of model misspecification—a disparity between the ground-truth and optimal function approximators— it is shown that policy-based approaches can be robust even when the policy function approximation is under a large \emph{locally-bounded} misspecification error, with which the function class may exhibit a $\Omega(1)$ approximation error in specific states and actions, but remains small on average within a policy-induced state distribution. Yet it remains an open question whether similar robustness can be achieved
with value-based and model-based approaches, especially with general function approximation. 

To bridge this gap, in this paper we present a unified theoretical framework for addressing model misspecification in RL. We demonstrate that, through meticulous algorithm design and sophisticated analysis, value-based and model-based methods employing general function approximation can achieve robustness under local misspecification error bounds. In particular, they can attain a regret bound of $\widetilde{O}\left(\text{poly}(dH)\cdot(\sqrt{K} + K\cdot\zeta) \right)$, where $d$ represents the complexity of the function class, $H$ is the episode length, $K$ is the total number of episodes, and $\zeta$ denotes the local bound for misspecification error. Furthermore, we propose an algorithmic framework that can achieve the same order of regret bound without prior knowledge of $\zeta$, thereby enhancing its practical applicability.

\end{abstract}
	
\section{Introduction}\label{intro}
Reinforcement Learning (RL) is a paradigm where agents learn to interact with an environment through state and reward feedback. In recent years, RL has seen significant success across various applications, such as control \citep{mnih2015human, gu2017deep}, board games \citep{silver2016mastering}, video games \citep{mnih2013playing}, and even the training of large language models like ChatGPT \citep{ouyang2022training}. In these applications, RL systems use deep neural networks (DNN) to approximate the policy, value, or models, thereby addressing the notorious ``curse-of-dimensionality'' issues associated with RL systems that have large state-action spaces \citep{bellman2010dynamic}.

Despite these successes, the theoretical understanding of how RL operates in practice, particularly when deep neural networks are involved, remains incomplete. A key question that arises is how approximation error, or misspecification, in function approximators can affect the performance of a deep RL system. Although deep networks are known to be universal approximators, their performance can be influenced by a range of factors, such as the training algorithm, dropout, normalization, and other engineering techniques. Therefore, the robustness of RL systems under misspecification is an important concern, particularly in risk-sensitive domains.

Theoretical advancements have begun to address the misspecification issue \citep{du2019good, jin2020provably, agarwal2020pc, zanette2021cautiously}. For example, \citet{du2019good} demonstrated that a small misspecification error in a value function approximator can lead to exponential increases in learning complexity, even when the function approximator for the optimal $Q$-value is a linear function class. In contrast, several studies have presented positive results for a relaxed model-class, where the transition probability matrix or the Bellman operator are close to a function class, making the learning system more robust to misspecification errors \citep{jin2020provably, wang2020reinforcement}.

However, these studies address misspecification errors of distinct natural. For instance, studies like \citep{jin2020provably,wang2020reinforcement,ayoub2020model} require the misspecification error to be \emph{globally bounded}, i.e., the function approximators should approximate the transition probability with a small error for all state-action pairs. Conversely, studies like \citet{agarwal2020pc, zanette2021cautiously} only require the misspecification error to be \emph{locally bounded}, where the errors need only be bounded at \emph{relevant} state-action pairs, which can be reached by a policy with a high enough probability. Notably, under a locally bounded misspecification error, the approximation error for certain state-action pairs (and therefore the global misspecification bound) can be arbitrarily large, and the analysis in \citep{jin2020provably,wang2020reinforcement,ayoub2020model} would fail to provide a meaningful learning guarantee.

Existing sample-efficient Reinforcement Learning (RL) algorithms can be categorized into three main types based on the targets they approximate: policy-based, value-based and model-based. Policy-based approaches distinguish themselves from the other two by utilizing Monte-Carlo (MC) sampling to estimate policy values during training. This approach is considered more robust compared to value bootstrapping in value-based methods. However, the reliance on MC sampling makes it challenging to reuse samples since each new policy necessitates fresh runs and data collection. Consequently, the state-of-the-art policy-based approach \citep{zanette2021cautiously} exhibits a statistical error bound scaling as $\propto 1/\epsilon^{-3}$ for an error level of $\epsilon$. In contrast, value-based and model-based approaches offer sample bounds of $\propto 1/\epsilon^{-2}$ \citep{yang2019sample, yang2020reinforcement, jin2020provably,ayoub2020model}. This disparity underscores the necessity for designing RL algorithms that are both statistically efficient and robust against misspecifications. Recent developments \citep{vial2022improved,agarwal2023provable} have enhanced the robustness and practicality of classical value-based method \citep{jin2020provably}. However, these algorithms depend on well-designed linear feature extractors, significantly limiting their applicability. In practice, algorithms often specify a function class (e.g., deep neural networks with a specific architecture) rather than a linear feature mapping. To date, a fundamental question concerning RL with general function approximation remains largely unanswered: \emph{"Is the policy-based approach inherently more robust than the value-based and model-based approaches when dealing with misspecifications?"} Or, equivalently, \emph{"Is it possible to design an RL approach with general function approximation that is both statistically efficient and robust to misspecifications?"}


In this paper, we delve into these fundamental questions and offer a comprehensive response. We present a unified and robust algorithm framework, \textbf{LBM-UCB} (\emph{\textbf{L}ocally \textbf{B}ounded \textbf{M}isspecification-\textbf{U}pper \textbf{C}onfidence \textbf{B}ound}), catering to value-based and model-based methods with general function approximation, specifically tailored to handle model misspecifications, particularly in the context of locally-bounded misspecifications. Unlike the realizable setting, where the ground-truth model is assumed to be within the function class, our robust algorithm framework meticulously designs a high-probability confidence set to encompass the best approximator within the function class. We demonstrate that under our framework, classical value-based algorithm \citep{wang2020reinforcement} and model-based algorithm \citep{ayoub2020model} can achieve a level of robustness similar to policy-based methods in the presence of locally bounded misspecifications. Importantly, they maintain a statistical rate scaling as $\propto 1/\epsilon^{-2}$. To be specific, for episodic RL with a total of $K$ episodes, a horizon length of $H$, a function class complexity of $d$, and a locally bounded misspecification error bound of $\zeta$, our regret bound is $\widetilde{O}\left(\text{poly}(dH)\cdot(\sqrt{K} + K\cdot\zeta) \right)$. This bound is almost optimal in terms of $\zeta$ and K, and provides a regret bound of $O(\zeta K)$ even when the misspecification error for certain states is on the order of $O(1)$. Furthermore, we devise a meta-algorithm within our framework that does not require prior knowledge of the misspecification parameter $\zeta.$ This enhancement increases its potential practical applicability, making it more accessible and versatile in real-world applications.

A core novelty of our analysis is that instead of making rough assumptions about all state-action pairs having a uniform upper bound with respect to misspecification, we carefully study which state-action pairs genuinely impact the algorithm's robustness. Interestingly, we find that, within the same episode, those state-action pairs that truly influence the algorithm's performance are drawn from the same policy-induced distribution. This allows us to obtain better bounds in average sense under the distribution of policies. Furthermore, since there is no global upper bound on misspecification errors, the global optimism (or near-optimism) property described in \citep{jin2020provably,wang2020reinforcement,ayoub2020model} no longer holds. To address this, we introduce a new method where we attach a virtual random process to utilize the optimal policy for data collection. This approach enables us to achieve near-optimism in the average sense, considering the distribution induced by the optimal policy.

\section{Related work}
In this section, we present the recent works that are relevant to our paper.

\paragraph{Misspecified Bandit}
In the context of misspecified linear bandit problems, where the reward function can be approximated by a linear function with some worst-case error, a number of papers, including those \citep{ghosh2017misspecified,foster2020beyond,lattimore2020learning,takemura2021parameter,zhang2023interplay}, have explored the notion of a uniform upper bound on misspecification errors across all actions. Additionally, \citet{foster2020adapting} has introduced a more lenient average-case concept of misspecification, which assesses the error associated with the specific sequence being considered. It is noteworthy that several other papers, such as those by \citep{lykouris2018stochastic,gupta2019better,zhao2021linear,wei2022model,ding2022robust,ye2023corruption}, delve into the analysis of cumulative misspecification errors, often referred to as "corruption," within the context of bandit problems.

\paragraph{RL with Function Approximations.}
Sample-efficient reinforcement learning (RL) algorithms employing function approximations can be classified into three primary categories, each based on the specific target they aim to approximate: value-based, model-based, and policy-based.

With regards to recent value-based methods, there are rich literature designing and analyzing algorithms for RL with linear function approximation \citep{du2019provably,wang2019optimism,zanette2019limiting, yang2020reinforcement, modi2020sample,wang2020reward,agarwal2020flambe,he2022nearly,zhang2023optimal}. However, these papers heavily rely on the assumption that the value function or the model can be approximated by a linear function or a generalized linear function of the feature vectors and do not discuss when model misspecification happens. On the other hand, these papers \citep{yang2019sample,jin2020provably,jia2019feature,vial2022improved} consider the model misspecification using the globally-bounded misspecification error, making their algorithms robust to a small range of misspecified linear models. Furthermore, \citet{vial2022improved} has introduced an algorithm with the notable attribute of being parameter-free in relation to the global bound of the misspecification parameter. In a similar vein, \citet{agarwal2023provable} has demonstrated that the classical algorithm LSVI-UCB \citep{jin2020provably} remains effective even in the presence of locally-bounded misspecification error. For recent general function approximations, complexity measures
are essential for non-linear function class, and \cite{russo2013eluder} proposed the concept of eluder dimension. Recent papers have extended it to more general framework 
\citep{jiang2017contextual,du2021bilinear,jin2021bellman,foster2020instance,chen2022general,zhong2022posterior,liu2023one}.  However,
the use of eluder dimension allows computational tractable
optimization methods. Based on the eluder dimension, \cite{wang2020reinforcement} describes a UCB-VI style algorithm that can explore the environment driven by a well-designed width function and \cite{kong2021online} devises an online sub-sampling method which
largely reduces the average computation time of \cite{wang2020reinforcement}. 
However, when considering the misspecified case, their work can only tolerate the approximation error between the assumed general function class and the truth model to have a uniform upper bound for all state-action pairs. In this paper, we analyze the regret bound of value-based methods under locally bounded misspecified MDP.

In the realm of model-based methods, several notable papers \citep{jia2020model, ayoub2020model, modi2020sample} have provided valuable statistical guarantees, primarily focusing on linear function approximation. These works concentrate on scenarios where the underlying transition probability kernel of the MDP is represented as a linear mixture model. Notably, \citep{zhou2021nearly} has introduced an algorithm that achieves nearly minimax optimality for linear mixture MDPs, incorporating Bernstein-type concentration techniques. Furthermore, \citep{zhou2022computationally} has designed computationally efficient horizon-free RL algorithms within the same linear mixture MDP framework. In the context of reward-free settings, \citep{zhang2021reward} has presented an algorithm based on the linear mixture MDP assumption. On a broader scale, \citep{ayoub2020model} has delved into general function approximation for model-based methods, using the concept of value-targeted regression. However, among the aforementioned papers, only \citep{jia2020model} and \citep{ayoub2020model} have ventured into the realm of model misspecification, although their assumptions remain confined to scenarios featuring globally-bounded misspecification error. In this paper, we analyze the regret bound of model-based methods under locally bounded misspecified MDP.

For recent policy-based methods with function approximation, a series of papers provide statistical guarantees \citep{cai2020provably,duan2020minimax,agarwal2021theory,feng2021provably,zanette2021cautiously}. Among them, \citet{agarwal2020pc,agarwal2021theory,feng2021provably,zanette2021cautiously} consider model misspecification using locally-bounded misspecification error but suffer from poor sample complexity due to policy evaluations. Specifically, \cite{agarwal2020pc} uses a notion called transfer error to measure the model misspecification in the linear setting, where they assume a good approximator under some policy cover has a bounded error in average sense when transferred to an arbitrary policy-induced distribution. Moreover, \cite{feng2021provably}  proposes a model-free algorithm applying the indicator of width function \citep{wang2020reinforcement} under the bounded transfer error assumption which allows the use of general function approximation. To improve the poor sample complexity of \cite{agarwal2020pc}, \cite{zanette2021cautiously} uses the doubling trick for determinant of empirical cumulative covariance and importance sampling technique.

\section{Preliminary}
In this paper, we focus on the episodic RL with setting modeled by a finite-horizon Markov Decision Process. Below we present a brief introduction of problem settings. 
\subsection{Episodic RL with Finite-Horizon Markov Decision Process}
We consider a finite-horizon Markov Decision Process (MDP) $M = (\mathcal{S},\mathcal{A},H,\mathbb{P},r,\mu)$, where $\mathcal{S}$ is the state space, $\mathcal{A}$ is the action space which has a finite size\footnote{Our approach can be extended to infinite-sized or continuous action space with an efficient optimization oracle for computing the $\arg\max$ operation.}, $\mathbb{P}: \mathcal{S} \times \mathcal{A} \rightarrow \Delta (\mathcal{S})$ is the transition operator, $r: \mathcal{S}\times \mathcal{A} \rightarrow [0,1]$ is the deterministic reward function, $H$ is the planning horizon, i.e. episode length, and $\mu$ is the initial distribution. 

An agent interacts with the environment episodically as follows. For each $H$-length episode, the agent adopts a policy $\pi$. To be specific, a policy $\pi = \{\pi_h\}_{h=1}^H $, where for each $h \in [H]$, $\pi_h:  \mathcal{S} \rightarrow \mathcal{A}$ chooses an action $a$ from the action space based on the current state $s$. The policy $\pi$ induces a trajectory $s_1,a_1,r_1,s_2,a_2,r_2,\cdots s_H,a_H,r_H$, where $s_1\sim \mu$, $a_1 = \pi_1(s_1)$, $r_1 = r(s_1,a_1)$, $s_2 \sim P(\cdot|s_1,a_1)$, $a_2=\pi_2(s_2)$, etc.

We use $V$-function and $Q$-function to evaluate the long-term expected cumulative reward under the policy $\pi$ with respect to the current state (state-action) pair. They are defined as: 
\[
Q_h^{\pi}(s,a) = \mathbb{E}\left[\sum\limits_{h'=h}^H r(s_{h'},a_{h'})| s_h=s,a_h=a ,\pi\right]\quad
\]
and
\[
V_h^{\pi}(s) = \mathbb{E}\left[\sum\limits_{h'=h}^H r(s_{h'},a_{h'})| s_h=s,\pi\right]
\]

For MDP, there always exists an optimal deterministic policy  $\pi^*$, such that $V_h^{\pi^*}(s) = \sup_{\pi} V_h^{\pi}(s)$ for all $s \in \mathcal{S}$ and all $h \in [H]$ \citep{puterman2014markov}. To simplify our notation, we denote the optimal $Q$-function and $V$-function as $Q_h^*(s,a) = Q_h^{\pi^*}(s,a)$ and $V_h^*(s) = V_h^{\pi^*}(s)$. We also denote $[\mathbb{P}V_{h+1}](s,a) := \mathbb{E}_{s'\sim \mathbb{P} (\cdot|s,a)} V_{h+1}(s') $, and the Bellman equation can be written as :
$$
Q_h^{\pi}(s,a) = r(s,a) + \mathbb{P}V_{h+1}^{\pi}(s,a)
$$
and
$$
V_h^{\pi}(s) = Q_h^{\pi}(s,\pi_h(s))
$$
Besides values, we also consider the state-action distribution generated by a policy. Without loss of generality, we assume that the agent always starts from a fixed point $s_1$ for each episode $k$. Concretely, for each time step $h \in [H]$, we define the state-action distribution induced by a policy $\pi$ as 
$$
d_{ h}^{\pi}(s,a) = \mathbb{P}^{\pi}(s_h = s , a_h =a |s_1)
$$ 
where $\mathbb{P}^{\pi}(s_h = s , a_h =a |s_1)$ is the probability of reaching $(s,a)$ at the $h$-{th} step starting from $s_1$ under policy $\pi$. We also define the average distribution $d^{\pi} = \frac{1}{H}\sum\limits_{h=1}^H d_h^{\pi}$.

The goal of the agent is to improve its performance with the environment. One way to measure the effectiveness of a learning algorithm is using the notion of regret. For $k \in [K]$, suppose the agent starts from state $s_1^k$ and chooses the policy $\pi^k$ to collect a trajectory. Then the regret is defined as 

$$
\text{Regret} (K) = \sum\limits_{k=1}^K \left(V_1^*(s_1^k) - V_1^{\pi^k}(s_1^k) \right)
$$


\subsection{Function Approximation}
In addressing optimization challenges within MDPs featuring large state-action spaces, we introduce function approximation methods. These methods can be categorized into two key groups: value-based (approximating the optimal value function) and model-based (approximating the MDP's transition kernel) function approximation. We will now provide detailed explanations of both approaches.

\paragraph{Value-based Function Approximation}
For the value-based setting, the function class $\mathcal{F}$ contains the approximators of optimal state-action value function of the MDP, which means $\mathcal{F} \subset \{f: \mathcal{S}\times \mathcal{A} \rightarrow \mathbb{R}\}$.
\begin{itemize}
    \item We denote corresponding state-action value function $Q_f = f$.
    \item We denote corresponding value function $V_f(\cdot) = \max_{a\in \mathcal{A}}Q_f(\cdot,a)$. Moreover, we denote the corresponding optimal policy $\pi_f$ with $\pi_f(\cdot) = \argmax_{a\in \mathcal{A}}Q_f(\cdot,a)$.
    \item We denote $f^*$ as the optimal state-action value function based on the ground-truth model, and it is possible that $f^*$ does not belong to the set $\mathcal{F}$.

\end{itemize}

\paragraph{Model-based Function Approximation}
For the model-based setting, the function class $\mathcal{F}$ contains the approximators of transition kernels, for which we denote $f = \mathbb{P}_f \in \mathcal{F} $.  
\begin{itemize}
    \item We denote $V_f^{\pi}$ as the value function induced by model $\mathbb{P}_f$ and policy $\pi$.
    \item We denote $V_f$ as the optimal value function under model $\mathbb{P}_f$, i.e., $V_{f} = \sup_{\pi \in \Pi}V_f^{\pi}$. Moreover, we denote $\pi_f$ as the corresponding optimal policy, i.e. $\pi_f = \argmax_{\pi \in \Pi} V_{f}^{\pi}$.
    \item We denote the ground-truth model as $f^*$, and $f^*$ may not belong to the function class $\mathcal{F}$.
\end{itemize}

\paragraph{Notation}
We use $[n]$ to represent index set $\{1,\cdots n\}$. For $x \in \mathbb{R}$, $\lfloor{x\rfloor}$ represents the largest integer not exceeding $x$ and $\lceil{x\rceil}$ represents the smallest integer exceeding $x$. Given $a,b \in \mathbb{R}^d$, we denote by $a^\top b$ the inner product between $a$ and $b$ and $||a||_2$ the Euclidean norm of $a$. Given a matrix $A$, we use $||A||_2$ for the spectral norm of $A$, and for a positive definite matrix $\Sigma$ and a vector $x$, we define $||x||_{\Sigma}=\sqrt{x^\top \Sigma x}$. We use $O$ to represent leading orders in asymptotic upper bounds and $\widetilde{O}$ to hide the polylog factors. For a finite set $\mathcal{A}$, we denote the cardinality of $\mathcal{A}$ by $|\mathcal{A}|$, all distributions over $\mathcal{A}$ by $\Delta (\mathcal{A})$, and especially 
the uniform distribution over $\mathcal{A}$ by $\text{Unif} (\mathcal{A})$. For a function $f: \mathcal{S} \times \mathcal{A} \rightarrow \mathbb{R}$, we define $\|f\|_{\infty}=\max_{(s, a) \in \mathcal{S} \times \mathcal{A}}|f(s, a)|$. Similarly, for a function $v: \mathcal{S} \rightarrow \mathbb{R}$, we define $\|v\|_{\infty}=\max _{s \in \mathcal{S}}|v(s)|$. For a set of state-action pairs $\mathcal{Z} \subseteq \mathcal{S} \times \mathcal{A}$, for a function $f: \mathcal{S} \times \mathcal{A} \rightarrow \mathbb{R}$, we define the $\mathcal{Z}$-norm of $f$ as $\|f\|_{\mathcal{Z}}=\left(\sum_{(s, a) \in \mathcal{Z}}(f(s, a))^{2}\right)^{1 / 2}$. Given a dataset $\mathcal{D} = \{(s_i,a_i,q_i)\}_{i=1}^{|\mathcal{D}|} \subset \mathcal{S} \times \mathcal{A} \times \mathbb{R}$, for a function $f: \mathcal{S}\times \mathcal{A} \rightarrow \mathbb{R}$, define
$||f||_{\mathcal{D}} = \left(\sum\limits_{t=1}^{|\mathcal{D}|}(f(s_t,a_t)-q_t)^2\right)^{1/2}$.

\section{Robust RL Algorithms with General Function Approximation} \label{Robust Algorithms}

\subsection{Generic Framework: LBM-UCB} \label{generic framework}

In this section, we provide a generic robust RL framework, \textbf{LBM-UCB} (\emph{\textbf{L}ocally \textbf{B}ounded \textbf{M}isspecification-\textbf{U}pper \textbf{C}onfidence \textbf{B}ound}), with general function approximation when locally-bounded misspecifications appear.

At the outset of each episode, denoted as $k=1,2,\cdots,K$, our algorithm identifies the optimal empirical approximator, denoted as $f^k$, from the hypothesis class $\mathcal{F}$. This selection is made by minimizing the loss function $L_{k-1}$ with the current dataset $\mathcal{D}^{k-1}$. The form of the loss function $L_{k-1}$ varies depending on the type of function approximation, typically employing the two-norm distance to measure the fitting error of the function within the hypothesis class.

In the conventional approach, a high-probability confidence set is constructed to encompass the ground-truth. However, in our misspecified setting, we cannot assume realizability, meaning that $f^*$ may not belong to the hypothesis class $\mathcal{F}$. Consequently, we construct a confidence set designed to encompass the best ground-truth approximator with high probability. This confidence set, denoted as $\mathcal{B}^k$, is centered around the best empirical approximator $f^k$. Its radius consists of two components: $\mathcal{E}_{\text{stat}}^k$ and $\mathcal{E}_{\text{bias}}^k$. Here, $\mathcal{E}_{\text{stat}}^k$ accounts for the statistical error arising from dataset randomness, while $\mathcal{E}_{\text{bias}}^k$ represents the error stemming from the mismatch between the ground-truth and the best ground-truth approximator.

Subsequently, the algorithm selects the optimistic approximator $f_{\text{op}}^k$ from the confidence set $\mathcal{B}^k$. Unlike realizable or globally-bounded misspecified settings, which achieve optimism, our locally-bounded misspecified setting only allows us to establish average optimism. The algorithm then determines the optimal policy $\pi^k$ based on the optimistic approximator $f_{\text{op}}^k$ and collects new data denoted as $\mathcal{Z}^k$ by executing that policy. Finally, the newly collected data is merged into the dataset, and the empirical loss function is updated for the subsequent training episode.

\begin{algorithm}[t] 
   \caption{LBM-UCB} \label{Algorithm Framework}
   \begin{algorithmic}[1]
   \STATE \textbf{Input: } The hypothesis class $\mathcal{F}$.
   \FOR{episode $k=1,\cdots,K$}
        \STATE Find the best empirical approximator $f^k$ in the hypothesis class $\mathcal{F}$ by solving the problem (\ref{regression_1}).
        \begin{equation} \label{regression_1}
            \begin{aligned}
                f^k = \argmin_{f \in \mathcal{F}} L_{k-1} ({\mathcal{D}}^{k-1},f)
            \end{aligned}
        \end{equation}

        \STATE Construct the high probability confidence set to contain the best ground-truth approximator.
        \begin{equation}
            \begin{aligned}
               \mathcal{B}^k = \{f \in \mathcal{F}| \ d(f,f^k) \leq \mathcal{E}_{\text{stat}}^k + \mathcal{E}_{\text{bias}}^k\}
            \end{aligned}
        \end{equation}
        \STATE Get the optimistic approximator $f^k_{\text{op}}$ in the confidence set $\mathcal{B}^k$ and the corresponding optimal policy $\pi^k = \pi_{f^k_{\text{op}}}$.
        \STATE Execute the policy $\pi^k$ to collect new data $\mathcal{Z}^k = \{\mathcal{Z}_h^k\}_{h \in [H]}$, where $\mathcal{Z}_h^k = \{(s_h^k,a_h^k,r_h^k,s_{h+1}^k)\}$.
        \STATE Update the dataset $\mathcal{D}^k \leftarrow \mathcal{D}^{k-1} \cup \mathcal{Z}^k$ and the empirical loss funtion $L_k$.
        
    \ENDFOR
\end{algorithmic} 
\end{algorithm}

\subsection{LBM-UCB for Value-based Algorithm}

In the context of value-based function approximation, our algorithm's objective is to learn the optimal state-value function. In this case, our \textbf{LBM-UCB} becomes Algorithm \ref{Algorithm general known} (\textbf{Robust-LSVI}). Consequently, the loss function takes the form: $L_{k-1}(\mathcal{D}_{k-1},f) = ||f||_{\mathcal{D}_h^k}^2$ where $\mathcal{D}_h^k = \{(s_{h'}^{k'},a_{h'}^{k'},r_{h'}^{k'} + V_{h+1}^k(s_{h'+1}^{k'}))\}_{(k',h')\in [k-1] \times [H]}$.

Before we proceed with constructing the confidence set, we adopt the sensitivity sampling technique as outlined in \citep{wang2020reinforcement}. This technique enables us to substantially reduce the size of the dataset while approximately preserving the confidence region.

In our efforts to encompass the best ground-truth approximator within the confidence set, we meticulously design the radius of the confidence set, denoted as $\beta(\mathcal{F},\delta)$. This radius is expressed as:

\begin{equation} \label{confidence_radius}
\begin{aligned}
\beta(\mathcal{F},\delta) = L(d,K,H,\delta) \cdot\left(\underbrace{\sqrt{kH}\zeta}_{\mathcal{E}_{\text{bias}}^k} + \underbrace{dH^2}_{\mathcal{E}_{\text{stat}}^k}\right)
\end{aligned}
\end{equation}

Here, $\zeta$ denotes locally
bounded misspecification error, and its specific definition is detailed in the section \ref{sec: theory}, $L(d,K,H,\delta)$ is a function that exhibits logarithmic scaling with respect to all the involved variables.

\subsection{LBM-UCB for Model-based Algorithm}

In the model-based setting, our algorithm's primary objective is to learn the transition kernel of the underlying MDP. In this context, our \textbf{LBM-UCB} becomes Algorithm \ref{Algorithm model-based} (\textbf{Robust-UCRL-VTR}). To design the loss function, we incorporate the concept of value-targeted regression from \citep{ayoub2020model}. Specifically,
\begin{equation} \label{v-t-1}
    \begin{aligned}
        \widehat{P}^{(k)} = \argmin_{P \in \mathcal{P}} L_{k-1}(\mathcal{D}_{k-1},P)
    \end{aligned}
\end{equation}
where
\begin{equation} \label{loss_model_based}
    \begin{aligned}
       L_{k-1}(\mathcal{D}_{k-1},P) &= \sum\limits_{k'=1}^{k-1} \sum\limits_{h=1}^H \left(PV_{h+1}^{k'}(s_h^{k'},a_h^{k'}) - V_{h+1}^{k'}(s_{h+1}^{k'})\right)^2
    \end{aligned}
\end{equation}
Subsequently, we define the model distance in relation to the estimated value functions as
\begin{small}
\begin{equation} \label{v-t-2}
    \begin{aligned}
        d_{k}(P,\widehat{P}^{(k)}) = \sum\limits_{k'=1}^{k-1}\sum\limits_{h=1}^H \left(PV_{h+1}^{k'}(s_h^{k'},a_h^{k'}) - \widehat{P}^{(k)}V_{h+1}^{k'}(s_h^{k'},a_h^{k'})\right)^2
    \end{aligned}
\end{equation}
\end{small}

We then define the confidence set as
\begin{equation}
    \begin{aligned}
        B_{k} = \{P' \in \mathcal{P} |\  d_{k}(P' , \widehat{P}^{(k)}) \leq \beta_{k}\}
    \end{aligned}
\end{equation}

The selection of the radius of the confidence set is similar to (\ref{confidence_radius}), where $\beta_k = L'(d,K,H,\delta) \cdot\left(\underbrace{\sqrt{kH}\zeta}_{\mathcal{E}_{\text{bias}}^k} + \underbrace{dH^2}_{\mathcal{E}_{\text{stat}}^k}\right)$.  Here, $\zeta$ denotes locally
bounded misspecification error, $L'(d,K,H,\delta)$ is a function that exhibits logarithmic scaling with respect to all the involved variables.

To obtain the optimistic approximator, the algorithm identifies the model that maximizes the optimal value. In other words,

\begin{equation} \label{optimisic model selection}
    \begin{aligned}
        P^{(k)} = \argmax_{P' \in B_k} V_{P' ,1}^* (s_1^k)
    \end{aligned}
\end{equation}

where $s_1^k$ is the initial state at the beginning of episode $k$, and $V_{P' ,1}^*$ represents the optimal value function at stage one under transition kernel $P'$. After that, the algorithm will calculate the corresponding optimal policy for $P^{(k)}$ using dynamic programming. In particular, for each $h \in [1,H+1]$, and all $(s,a) \in \mathcal{S} \times \mathcal{A}$,
\begin{equation} \label{dynamic-programming}
    \begin{aligned}
        &Q_{H+1}^k(s,a) =0 \\ & V_h^k(s) = \max_{a \in \mathcal{A}} Q_h^k(s,a) \\ & Q_h^k(s,a) = r_h(s,a) + P^{(k)}V_{h+1}^k(s,a)
    \end{aligned}
\end{equation}

\section{Theoretical Analysis of Robust RL Algorithms with General Function Approximation} \label{sec: theory}

In this section, we will provide our theoretical analysis of Robust RL Algorithms with general function approximation in Section \ref{Robust Algorithms} under the locally-bounded misspecification assumptions. 

First of all, the sample complexity of algorithms with function approximation depends on the complexity of the function class. To measure this complexity, we adopt the notion of eluder dimension which is first mentioned in \cite{russo2013eluder}.

\begin{definition}[Eluder dimension]
\label{def:inj}
 Let $\varepsilon \geq 0$ and $\mathcal{Z}=\left\{\left(s_{i}, a_{i}\right)\right\}_{i=1}^{n} \subseteq \mathcal{S} \times \mathcal{A}$ be a sequence of state-action pairs.
\begin{itemize}
\item{A state-action pair $(s, a) \in \mathcal{S} \times \mathcal{A}$ is $\varepsilon$-dependent on $\mathcal{Z}$ with respect to $\mathcal{F}$ if any $f, f^{\prime} \in \mathcal{F}$ satisfying $\left\|f-f^{\prime}\right\|_{\mathcal{Z}} \leq \varepsilon$ also satisfies $\left|f(s, a)-f^{\prime}(s, a)\right| \leq \varepsilon$.}
\item{An $(s, a)$ is $\varepsilon$-independent of $\mathcal{Z}$ with respect to $\mathcal{F}$ if $(s, a)$ is not $\varepsilon$-dependent on $\mathcal{Z}$.}
\item{The eluder dimension $\operatorname{dim}_{E}(\mathcal{F}, \varepsilon)$ of a function class $\mathcal{F}$ is the length of the longest sequence of elements in $\mathcal{S} \times \mathcal{A}$ such that, for some $\varepsilon^{\prime} \geq \varepsilon$, every element is $\varepsilon^{\prime}$-independent of its predecessors.}
\end{itemize}
\end{definition}

Next, we discuss the regret bound of these two algorithms respectively.

\subsection{Regret Bound of Robust-LSVI}

First, we give a theorectical analysis for \textbf{Robust-LSVI} with general function approximation. We assume that the function class $\mathcal{F}$ and the state-actions $\mathcal{S}\times\mathcal{A}$ have bounded covering numbers.

\begin{assumption} [$\varepsilon$-cover]\label{ass:cover}
For any $\varepsilon>0$, the following holds:

1. there exists an $\varepsilon$-cover $\mathcal{C}(\mathcal{F}, \varepsilon) \subseteq \mathcal{F}$ with size $|\mathcal{C}(\mathcal{F}, \varepsilon)| \leq \mathcal{N}(\mathcal{F}, \varepsilon)$, such that for any $f \in \mathcal{F}$, there exists $f^{\prime} \in \mathcal{C}(\mathcal{F}, \varepsilon)$ with $\left\|f-f^{\prime}\right\|_{\infty} \leq \varepsilon$;

2. there exists an $\varepsilon$-cover $\mathcal{C}(\mathcal{S} \times \mathcal{A}, \varepsilon)$ with size $|\mathcal{C}(\mathcal{S} \times \mathcal{A}, \varepsilon)| \leq \mathcal{N}(\mathcal{S} \times \mathcal{A}, \varepsilon)$, such that for any $(s, a) \in \mathcal{S} \times \mathcal{A}$, there exists $\left(s^{\prime}, a^{\prime}\right) \in \mathcal{C}(\mathcal{S} \times \mathcal{A}, \varepsilon)$ with $\max _{f \in \mathcal{F}}\left|f(s, a)-f\left(s^{\prime}, a^{\prime}\right)\right| \leq \varepsilon$.
\end{assumption}

\begin{remark}
Assumption \ref{ass:cover} is rather standard. Since our algorithm complexity depends only logarithmically on $\mathcal{N}(\mathcal{F},\cdot)$ and $\mathcal{N}(\mathcal{S} \times \mathcal{A},\cdot)$, it is even acceptable to have exponential size of these covering numbers.
\end{remark}

\begin{assumption} 
[General value function approximation with LBM]\label{Ass:general-value}

Given the ground-truth MDP $M$ with the transition model $\mathbb{P}$ and the reward function $r$, we assume that there exists a function class $\mathcal{F} \subset \{f:\mathcal{S}\times\mathcal{A} \rightarrow [0,H+1]\}$ and  a real number $\zeta \in [0,1]$, such that for any $V : \mathcal{S} \rightarrow [0,H]$, there exists a non-empty function class $\bar{\mathcal{F}}_V \subset \mathcal{F}$, which satisfies :  for all $ f \in \bar{\mathcal{F}}_V$, and all $\beta \in [4]$,

$$
\sup_{\pi} \mathbb{E}_{(s,a)\sim d^{\pi}}\left|f(s,a) - \left(r(s,a)+\mathbb{P}V(s,a)\right)\right|^{\beta} \leq \zeta^{\beta}
$$

\end{assumption}

\begin{theorem}[Regret bound with known $\zeta$]
\label{Thm 3.1}

Under our Assumption \ref{ass:cover} and \ref{Ass:general-value}, for any fixed $\delta \in (0,1)$, with probability at least $1-\delta$,  the total regret of Algorithm \ref{Algorithm general known} is at most $\widetilde{O}\left(\sqrt{d_EH^3}K\zeta \log(1/\delta)+ \sqrt{d_E^2KH^3} \log(1/\delta) \right)$, where $d_E$ represents the eluder dimension of the function class.
\end{theorem}

The comprehensive proof is presented in Appendix \ref{Sec: general value}.

\begin{remark}
    Our assumption is strictly weaker than the globally-bounded misspecification error  in \citep{wang2020reinforcement}, where they assumed that, for all $(s,a) \in \mathcal{S} \times \mathcal{A}$, and $V:\mathcal{S} \rightarrow [0,H]$, $|f(s,a) - \left(r(s,a)+ \mathbb{P}V(s,a)\right)| \leq \zeta$. 
    
    In other words, our Assumption \ref{Ass:general-value} only needs the misspecification error to be \emph{locally} bounded at relevant state-action pairs, which can be reached by some policies with sufficiently high probability, whereas \cite{wang2020reinforcement} requires the misspecification error to bounded \emph{globally} in all state-action pairs, including those not even relevant to the learning.
\end{remark}

\begin{remark}
 A classical special case of the general function approximation setting is the linear MDP \citep{yang2019sample,jin2020provably}. In this case, \citep{agarwal2023provable} initially demonstrated the effectiveness of LSVI-UCB \citep{jin2020provably} even under conditions of locally-bounded misspecification error. Our assumptions and findings serve as a more general version of \citep{agarwal2023provable}, extending the utility of function approximation from a linear setting to a more general context. For the reader's convenience, we present the result below. Under Assumption \ref{Ass1} and \ref{Ass2} , for any fixed $\delta \in (0,1)$, with probability at least $1-\delta$,  the total regret of the algorithm Robust-LSVI (Algorithm \ref{Algorithm known}) is at most $\widetilde{O}\left(dKH^2\zeta \log(1/\delta)+ \sqrt{d^3KH^4} \log(1/\delta)\right)$. Here, $d$ represents the dimensionality of the linear features. The comprehensive proof is elaborated upon in Appendix \ref{Sec A}.


\end{remark}

\subsection{Regret Bound of Robust-UCRL-VTR}

Next, we provide our assumption and the main theoretical result for \textbf{Robust-UCRL-VTR}.
Let $\mathcal{V}$ be the set of optimal value functions under some model in the hypothesis class $\mathcal{P}$: $\mathcal{V} = \{V_{P'}^* : P' \in \mathcal{P} \}$.
We define $\mathcal{X} = \mathcal{S}\times \mathcal{A} \times \mathcal{V}$, and choose 
\begin{equation}
    \begin{aligned} \label{approximation-function-class}
        \mathcal{F} = \big\{ f : \mathcal{X} \rightarrow  \mathbb{R}: \exists \widetilde{P} \in \mathcal{P}\ \  \text{s.t.}\  f(s,a,V) = \widetilde{\mathbb{P}} V(s,a), \ \ \forall (s,a,V) \in \mathcal{X}\big\}
    \end{aligned}
\end{equation}

Similar to Assumption \ref{ass:cover}, we assume the function class $\mathcal{F}$ defined in (\ref{approximation-function-class}) has bounded covering numbers.

\begin{assumption} [$\varepsilon$-cover]\label{ass:cover2}
For any $\varepsilon>0$, the following holds:

There exists an $\varepsilon$-cover $\mathcal{C}(\mathcal{F}, \varepsilon) \subseteq \mathcal{F}$ with size $|\mathcal{C}(\mathcal{F}, \varepsilon)| \leq \mathcal{N}(\mathcal{F}, \varepsilon)$, such that for any $f \in \mathcal{F}$, there exists $f^{\prime} \in \mathcal{C}(\mathcal{F}, \varepsilon)$ with $\left\|f-f^{\prime}\right\|_{\infty} \leq \varepsilon$

\end{assumption}

\begin{assumption}[General model function approximation with LBM]
\label{Ass4}

Given the ground-truth MDP $M$ with the transition model $\mathbb{P}$, we assume that there exists a real number $\zeta \in [0,1]$, and $\bar{f}\in \mathcal{F}$ (defined in \ref{approximation-function-class}), such that for any $V \in \mathcal{V}$, and any $\beta \in [4]$,

$$
\sup_{\pi} \mathbb{E}_{(s,a)\sim d^{\pi}}\left|\bar{f}(s,a,V) - \mathbb{P}V(s,a)\right|^{\beta} \leq \zeta^{\beta}
$$
\end{assumption}

Now we present the main theorem of our algorithm and the in-depth proof is showcased in Appendix \ref{general_model_based}.

\begin{theorem}[Regret bound with known $\zeta$] 
\label{Thm-model-based(zeta is known)}

Under our Assumption \ref{ass:cover2} and \ref{Ass4}, for any fixed $\delta \in (0,1)$, with probability at least $1-\delta$, the total regret of Algorithm \ref{Algorithm model-based} is at most $\widetilde{O}\left(\sqrt{d_E}KH\zeta \log(1/\delta) + \sqrt{d_E^2KH^3} \log(1/\delta)\right)$, 
where $d_E$ represents the eluder dimension of the function class.

\end{theorem}

\begin{remark}
    Our assumption is strictly weaker than that in \citep{ayoub2020model}, where they assume that given the truth model $\mathbb{P}$, there exists an approximator $\bar{P} \in \mathcal{P}$ such that for all $(s,a) \in \mathcal{S} \times \mathcal{A}$, $||\bar{P}(\cdot|s,a) - \mathbb{P}(\cdot|s,a)||_{\text{TV}} \leq \zeta$. To clarify, our Assumption \ref{Ass4} merely necessitates the misspecification error to be locally bounded at the relevant state-action pairs. These pairs are those that can be reached by certain policies with a high probability.

\end{remark}

\begin{remark}
    A direct corollary of our result is for the Linearly-Parametrized Transition Model. Specifically, we assume there are $d$ transition models, $P_1,P_2,\cdots,P_d$, $\Theta \subset \mathbb{R}^d$ is a bounded and nonempty set, and let $\mathcal{P} = \left\{\sum\limits_{j} \theta_j P_j: \theta \in \Theta\right\}$. Given the ground-truth model $\mathbb{P}$ , if there exists a $d$-dimension vector $\alpha \in \Theta$, such that $\mathbb{E}_{(s,a)\sim d^{\pi}}||\mathbb{P}(\cdot|s,a)-\sum\limits_{j=1}^d \alpha_j P_j(\cdot|s,a)||_1^{\alpha} \leq \zeta^{\alpha}$, $\forall \alpha \in [4]$. Then for any fixed $\delta \in (0,1)$, with probability at least $1-\delta$, the total regret of Algorithm \ref{Algorithm model-based} is at most $\widetilde{O}\left(\sqrt{d}KH\zeta \log(1/\delta) + \sqrt{d^2KH^3} \log(1/\delta)\right)$.
\end{remark}

\section{Meta Algorithm without Knowing the Misspecified Parameter}

{\small
\begin{algorithm}[t] 
   \caption{Meta-algorithm for unknown misspecified parameter $\zeta$} \label{Algorithm unknown}
   \begin{algorithmic}[1]
   \State \textbf{Input: } The base algorithm \textsl{Alg.}, the total number of episodes $K$, the length of one episode $H$, failure probability $\delta>0$.
   \FOR{ epoch $i=0,1,2,\cdots, \lfloor \log_2\left(\sqrt{3K+1}\right) \rfloor$}
        \State $\zeta^{(i)} \leftarrow \frac{1}{2^{i}}$, $K^{(i)} \leftarrow \frac{1}{(\zeta^{(i)})^2}$,
        \State $(\overline{V_1}^{(i)}, \pi^{(i)})  \leftarrow  \text{Algorithm} \ \ref{Algorithm Single} \ (\textsl{Alg.}, K^{(i)}, H, \delta, \zeta^{(i)})$
        \IF{$i \geq 1$}
        \IF{$|\overline{V_1}^{(i)} - \overline{V_1}^{(i-1)}| > C(d,H,\delta) \cdot \zeta^{(i)}$}
               \State $j \leftarrow i-1$,
               \State \textbf{break};
        \ENDIF
        \ENDIF
    \ENDFOR
    \FOR{ the rest episodes $t = 1,2,\cdots K-\sum\limits_{i=0}^{j+1} K^{(i)}$}
         \FOR{step $h=1,\cdots,H$}
            \State Take action $a_h^t \leftarrow \pi^{(j)}(s_h^t)$, and observe $s_{h+1}^t$.
        \ENDFOR
    \ENDFOR
\end{algorithmic} 
\end{algorithm}
}

In real-world environments, we cannot assume that the misspecified parameter $\zeta$ is provided. This issue serves as motivation for our meta algorithm (Algorithm \ref{Algorithm unknown}), which makes the base algorithm (e.g., Algorithm \ref{Algorithm general known},\ref{Algorithm model-based},\ref{Algorithm known}) have the parameter-free property by employing exponentially decreasing misspecified parameters and increasing training episodes. Without loss of generality, we assume the initial state $s_1$ remains fixed across all episodes. The entire training process is divided into multiple epochs. In each epoch, the meta algorithm interacts with the environment (using a small variation of the base algorithm,  Algorithm \ref{Algorithm Single} in Appendix \ref{sec_Algorithm}, which is almost the same as the base algorithm except that it outputs the policy and value of each round) a total of $K^{(i)} = 1/(\zeta^{(i)})^2$ times, where $\zeta^{(i)} = 1/2^i$ is the exponentially decreasing misspecified parameter. After each epoch, the real-time reward data is utilized to estimate the value function of the average policy for that round. Notably, when training with misspecified parameter that is roughly the true value (i.e., $\zeta^{(i)} \gtrsim \zeta$), the value estimates from adjacent epochs exhibit minimal variation. However, as the misspecified parameter $\zeta^{(i)}$ decreases below the ground-truth parameter $\zeta$, the obtained policy may deteriorate since the base algorithms do not guarantee optimism in this scenario. Hence, when our stability condition is violated, defined as $|\overline{V_1}^{(i)} - \overline{V_1}^{(i-1)}| > C(d,H,\delta) \cdot \zeta^{(i)}$ (where $C(d,H,\delta)$ is a constant dependent on $d,H,\delta$, see definition in Appendix \ref{Sec: meta algorithm}), we break out of the loop and execute the last average policy for the remaining episodes. According to our selection of the misspecified parameters, there must exist an accurate parameter $\zeta^{(s)}$ close to the true $\zeta$ ($\zeta \leq \zeta^{(s)} < 2\zeta$). As the last executed policy still satisfies the stability condition, it can serve as an approximate good policy for the previous policy with the parameter $\zeta^{(s)}$.

Based on the above analysis, our formal guarantee of Algorithm \ref{Algorithm unknown} for the unknown $\zeta$ case is presented as follows. The detailed proof is displayed in Appendix \ref{Sec: meta algorithm}.

\begin{theorem}[Regret bound with unknown $\zeta$]

\label{Thm 3.2}

Suppose the input base algorithm \textsl{Alg.} which needs to know the locally-bounded misspecified parameter $\zeta$ has a regret bound of $\widetilde{O}\left(d^{\alpha}H^{\beta}(\sqrt{K}+K\cdot\zeta)\right)$, then our meta-algorithm (Algorithm \ref{Algorithm unknown}) can achieve the same order of regret bound $\widetilde{O}\left(d^{\alpha}H^{\beta}(\sqrt{K}+K\cdot\zeta)\right)$ without knowing the misspecified parameter $\zeta$.
\end{theorem}


\section{Conclusion}
In this paper, we have proposed a robust RL algorithm framework for value-based and model-based methods under locally-bounded misspecification error. Through a careful design of the high-probability confidence set and a refined analysis, we have significantly improved the regret bound of \citep{wang2020reinforcement, ayoub2020model} when the misspecification error is not \emph{globally bounded}. Furthermore, we have developed a provably efficient meta algorithm to address scenarios where the misspecified parameter is unknown.

\section{Acknowledgement}
This work was supported in part by DARPA under agreement HR00112190130, NSF grant 2221871, and an Amazon Research Grant.

	\bibliography{Reference}
	\bibliographystyle{icml2023}

	\newpage
	\appendix
	\onecolumn

\section{Remaining Algorithm Pseudocodes} \label{sec_Algorithm}
We provide the remaining algorithms in this section.

\begin{algorithm}[H] 
   \caption{Robust-LSVI with general function approximation ($\zeta$ is known)} \label{Algorithm general known}
   \begin{algorithmic}[1]
   \STATE \textbf{Input: } The function class $\mathcal{F}$, the number of episodes $K$, the length of one episode $H$, failure probability $\delta>0$, misspecified parameter $\zeta$.
   \FOR{episode $k=1,\cdots,K$}
        \STATE Receive the initial state $s_1^k$.
        \STATE Initialize $Q_{H+1}^k(\cdot,\cdot) \leftarrow 0$, $V_{H+1}^k(\cdot) \leftarrow 0$.
        \STATE $\mathcal{Z}^k \leftarrow \{(s_{h'}^{k'},a_{h'}^{k'})\}_{(k',h')\in [k-1]\times [H]}$

        \FOR{step $h= H,H-1,\cdots,1$}
          \STATE Find the best empirical approximator:
           $$
           f_h^k \leftarrow \argmin_{f \in \mathcal{F}}||f||_{\mathcal{D}_h^k}^2
           $$
           where
           $$
           \mathcal{D}_h^k = \{(s_{h'}^{k'},a_{h'}^{k'},r_{h'}^{k'} + V_{h+1}^k(s_{h'+1}^{k'}))\}_{(k',h')\in [k-1] \times [H]}
           $$
          
          \STATE $(\widehat{f}_h^k,\widehat{\mathcal{Z}}^k) \leftarrow \text{Sensitivity-Sampling} (\mathcal{F},f_h^k,\mathcal{Z}^k,\delta)$
          \STATE Construct the confidence set 
          $$
            \widehat{\mathcal{F}} = \left\{f\in \mathcal{F}: ||f-\widehat{f}_h^k||_{\widehat{\mathcal{Z}}^k}  \leq \beta(\mathcal{F},\delta)  \right\}
          $$
          where $\beta(\mathcal{F},\delta) = L(d,K,H,\delta) \cdot\left(\underbrace{\sqrt{kH}\zeta}_{\mathcal{E}_{\text{bias}}^k} + \underbrace{dH^2}_{\mathcal{E}_{\text{stat}}^k}\right)$
        
          \STATE Get the optimistic approximator
          $$
          Q_h^k(\cdot,\cdot) \leftarrow \min\{f_h^k(\cdot,\cdot) + b_h^k(\cdot,\cdot),H\}, \ \ V_h^k(\cdot) = \max_{a\in\mathcal{A}}Q_h^k(\cdot,a)
          $$, where
          $$
          b_h^k(\cdot,\cdot) = \sup\limits_{f_1,f_2\in \widehat{\mathcal{F}}}|f_1(\cdot,\cdot) - f_2(\cdot,\cdot)|
          $$
          \STATE Get the corresponding optimal policy
          $$
          \pi_h^k(\cdot) \leftarrow \text{argmax}_{a\in \mathcal{A}}Q_h^k(\cdot,a)
          $$
        \ENDFOR
        \FOR{step $h=1,\cdots,H$}
            \STATE Take action $a_h^k \leftarrow \pi_h^k(s_h^k)$, and observe $s_{h+1}^k$ and $r_h^k = r(s_h^k,a_h^k)$.
        \ENDFOR
    \ENDFOR
\end{algorithmic} 
\end{algorithm}

\newpage

\begin{algorithm}[H] 
   \caption{Robust-UCRL-VTR with general function approximation ($\zeta$ is known)} \label{Algorithm model-based}
   \begin{algorithmic}[1]
   \STATE \textbf{Input: } Family of MDP models $\mathcal{P}$, The number of episodes $K$, the length of one episode $H$, failure probability $\delta>0$, misspecified parameter $\zeta$.
   \STATE $B_1 = \mathcal{P}$
   \FOR{episode $k=1,\cdots,K$}
        \STATE Receive the initial state $s_1^k$.
        Find the best empirical approximator:
        $$
        \widehat{P}^{(k)} \leftarrow \argmin_{P\in\mathcal{P}}L_{k-1}(\mathcal{D}_{k-1},P)\ \ (\ref{loss_model_based})
        $$
        
        \STATE Construct the confidence set:
        $$
         B_{k} = \{P' \in \mathcal{P} | \ d_{k}(P' , \widehat{P}^{(k)}) \leq \beta_{k}\} \ \ (\ref{v-t-2})
        $$
        where
        $$
        \beta_k = L'(d,K,H,\delta)\cdot(\underbrace{\sqrt{kH}\zeta}_{\mathcal{E}_{\text{bias}}^k} + \underbrace{H}_{\mathcal{E}_{\text{stat}}^k})
        $$
        \STATE Get the optimistic approximator 
        $$
         P^{(k)} = \text{argmax}_{P' \in B_k} V_{P',1}^*(s_1^k)
        $$
        Compute $Q_1^k,Q_2^k,\cdots,Q_H^k$ for $P^{(k)}$ using dynamic programming (\ref{dynamic-programming}).
        
       \STATE Get the corresponding optimal policy 
        $$
        \pi_h^k(\cdot) \leftarrow \argmax_{a\in \mathcal{A}}Q_h^k(\cdot,a) \ \ ,h=1,2,\cdots,H
        $$

        \FOR{step $h= 1,2,\cdots,H$}
           \STATE Take action $a_h^k \leftarrow \pi_h^k(s_h^k)$, and observe $s_{h+1}^k$ and $r_h^k = r(s_h^k,a_h^k)$.
        \ENDFOR
    \ENDFOR
\end{algorithmic} 
\end{algorithm}

\newpage

\begin{algorithm}[h] 
   \caption{Single-epoch-Algorithm} \label{Algorithm Single}
   \begin{algorithmic}[1]
   \State \textbf{Input: } The base algorithm \textsl{Alg.}, The number of episodes $K$ in a single epoch, the length of one episode $H$, failure probability $\delta>0$, misspecified parameter $\zeta$.
   \FOR{$k=1,\cdots,K$}
        \State Update the policy  $\{\pi_h^k\}_{h \in [H]}$ by using \textsl{Alg.}
        \FOR{step $h=1,\cdots,H$}
            \State Take action $a_h^k \leftarrow \pi_h^k(s_h^k)$, and observe $s_{h+1}^k$.
            \State Calculate $R^k \leftarrow R^k + r_h(s_h^k,a_h^k)$.
        \ENDFOR
    \ENDFOR
    \State \textbf{Output: } $\textbf{value} :\overline{V_1} \leftarrow \frac{1}{K}\sum\limits_{k=1}^K R^k$ 
    \State $\ \ \ \ \ \ \ \ \ \ \ \ \ \ \ \ $ $\textbf{policy}: \text{Unif} (\pi^1,\pi^2,\cdots, \pi^K)$
\end{algorithmic} 
\end{algorithm}

\newpage

\section{Analysis of Robust Value-based Algorithm with Linear Function Approximation for Locally-bounded Misspecified MDP} \label{Sec A}

\begin{algorithm}[h] 
   \caption{Robust-LSVI with linear function approximation ($\zeta$ is known)} \label{Algorithm known}
   \begin{algorithmic}[1]
   \STATE \textbf{Input: } The number of episodes $K$, the length of one episode $H$, failure probability $\delta>0$, misspecified parameter $\zeta$.
   \FOR{episode $k=1,\cdots,K$}
        \STATE Receive the initial state $s_1^k$.
        \STATE Initialize $Q_{H+1}^k(\cdot,\cdot) \leftarrow 0$, $V_{H+1}^k(\cdot) \leftarrow 0$.
        \STATE Update the bonus parameter $\beta_k \leftarrow c_{\beta}\left(4\sqrt{kd}\zeta+\sqrt{(\lambda+1)d^2\log\left(\frac{4dKH}{\delta}\right)}\right)H$.

        \FOR{step $h= H,H-1,\cdots,1$}
           \STATE $\Lambda_h^k \leftarrow \sum\limits_{\tau=1}^{k-1}\phi(s_h^{\tau},a_h^{\tau})\phi(s_h^{\tau},a_h^{\tau})^{\top} +\lambda \cdot \bf{I}$.
           \STATE $\mathbf{w}_h^k \leftarrow (\Lambda_h^k)^{-1} \sum\limits_{\tau=1}^{k-1} \phi(s_h^{\tau},a_h^{\tau})[r_h(s_h^{\tau},a_h^{\tau})+ V_{h+1}^k(s_{h+1}^{\tau})]$.
           \STATE $Q_h^k(\cdot,\cdot) \leftarrow \min\{(\mathbf{w}_h^k)^{\top}\phi(\cdot,\cdot)+\beta_k[\phi(\cdot,\cdot)^{\top}(\Lambda_h^k)^{-1}\phi(\cdot,\cdot)]^{1/2}, H\}$.
           \STATE $V_h^k(\cdot) \leftarrow \max_{a\in\mathcal{A}} Q_h^k(\cdot,a)$.
           \STATE $\pi_h^k(\cdot) \leftarrow \text{argmax}_{a\in \mathcal{A}}Q_h^k(\cdot,a)$.
        \ENDFOR
        \FOR{step $h=1,\cdots,H$}
            \STATE Take action $a_h^k \leftarrow \pi_h^k(s_h^k)$, and observe $s_{h+1}^k$.
        \ENDFOR
    \ENDFOR
\end{algorithmic} 
\end{algorithm}

In order to let readers better understand the core idea of our paper, we first give the proof for the linear case before giving the proof for the general case. We study the linear function approximation setting for MDPs introduced in \cite{yang2019sample, jin2020provably}, where the probability transition matrix can be approximated by a linear function class. 
To enable a much stronger \emph{locally bounded} misspecification error, we consider the following notion of $\zeta$-Average-Approximate Linear MDP. It is worth mentioning that \citet{agarwal2023provable} also gives a positive result of LSVI-UCB \citep{jin2020provably} under average-misspecification, and here we present another version of the proof.

\begin{assumption}
($\zeta$-Average-Approximate Linear MDP). \label{Ass1}
For any $\zeta\leq1$, we say that the MDP $M = (\mathcal{S}, \mathcal{A}, H, \mathbb{P}, r)$ is a $\zeta$-Average-Approximate Linear MDP with a feature map $\phi : \mathcal{S}\times\mathcal{A} \rightarrow \mathbb{R}^d$, if for  any $h \in [H]$, there exists a $d$-dimension measures $\boldsymbol{\mu}_h = (\mu_h^{(1)}, \cdots, \mu_h^{(d)})$ over $\mathcal{S}$, and an vector $\boldsymbol{\theta}_h \in \mathbb{R}^d$, such that for any policy $\pi$, and any $\alpha\in[4]$, we have
\vspace{-1mm}
\[
\mathbb{E}_{(s,a)\sim d_h^{\pi}}||\mathbb{P}_h(\cdot|s,a) - \left\langle \phi(s,a),\boldsymbol{\mu}_h(\cdot) \right\rangle ||_{\text{TV}}^\alpha \leq \zeta^\alpha
\quad\text{and}\quad
\mathbb{E}_{(s,a)\sim d_h^{\pi}}|r_h(s,a)-\left\langle\phi(s,a),\boldsymbol{\theta}_h\right\rangle |^\alpha \leq \zeta^\alpha
\]
\vspace{-2mm}
\end{assumption}

\begin{remark}
We note that the assumption on the boundedness of the $4$-th moments is minor: $\mathbb{E}_{(s,a)\sim d_h^{\pi}}|f(s,a)|^\alpha$ is bounded for any $\alpha>1$ as long as $f$ is bounded and $\mathbb{E}_{(s,a)\sim d_h^{\pi}}|f(s,a)|$ is bounded.
We choose a $4$-th moment bound for the ease of presentation and fair comparison with existing results.

\end{remark}

\begin{remark}
    Our assumption is strictly weaker than the $\zeta$-Approximate Linear MDP in \cite{jin2020provably}, where they assumed that, for all $ (s,a) \in \mathcal{S} \times \mathcal{A}$: $\quad ||\mathbb{P}_h(\cdot|s,a) - \left\langle \phi(s,a),\boldsymbol{\mu}_h(\cdot) \right\rangle ||_{\text{TV}} \leq \zeta,  \ \   |r_h(s,a)-\left\langle\phi(s,a),\boldsymbol{\theta}_h\right\rangle | \leq \zeta.$
    
In other words, $\zeta$-Average-Approximate Linear MDP only needs the misspecification error to be \emph{locally} bounded at relevant state-action pairs, which can be reached by some policies with sufficiently high probability, whereas $\zeta$-Approximate Linear MDP requires the misspecification error to bounded \emph{globally} in all state-action pairs, including those not even relevant to the learning.

\end{remark}

We will also need the following standard assumptions for the regularity of the feature map:
\begin{assumption}  \label{Ass2}
    (Boundness) Without loss of generality, we assume that $||\phi(s,a)||\leq 1$, and $\max\{||\boldsymbol{\mu}_h(\mathcal{S})||, ||\boldsymbol{\theta}_h||\} \leq \sqrt{d} $, $\forall (s,a,h) \in \mathcal{S}\times\mathcal{A}\times [H]$.
\end{assumption}

The following analysis in Section \ref{Sec A} are based on Assumption \ref{Ass1} and \ref{Ass2}.

To simplify our notation, for each $(s,a,h)\in \mathcal{S}\times \mathcal{A} \times [H]$, we denote 
\[
\xi_h(s,a) = ||\mathbb{P}_h(\cdot|s,a) - \widetilde{\mathbb{P}_h}(\cdot|s,a)||_{\text{TV}}
\]
\[
\eta_h(s,a) = |r_h(s,a)- \widetilde{r_h}(s,a) |
\]

where $\widetilde{\mathbb{P}_h}(\cdot|s,a): =\left\langle \phi(s,a),\boldsymbol{\mu}_h(\cdot) \right\rangle$, and $\widetilde{r_h}(s,a) := \left\langle\phi(s,a),\boldsymbol{\theta}_h\right\rangle$. 

With the above notation, $\xi_h(s,a)$, $\eta_h(s,a)$ denotes the model misspecification error for transition kernel $\mathbb{P}_h$ and reward function $r_h$ of a fixed state-action pair $(s,a)$.

Moreover,  for all $(k,h) \in [K] \times [H]$, we denote $\phi_h^k := \phi(s_h^k,a_h^k)$.

In the following analysis, we consider an auxiliary stochastic process, which, although unattainable in reality, proves valuable for our analysis.

\paragraph{Auxiliary Stochastic Process}

For each episode, denoted by $k \in [K]$, we collect the dataset $D_k = \{(s_h^k,a_h^k)\}_{h=1}^H$ using policies $\{\pi_h^k\}_{h=1}^H$ trained by the algorithm in the last $k$ episodes. Additionally, we allow the agent to gather data $D_k^* = \{(s_h^{k*},a_h^{k*})\}_{h=1}^H$ using optimal policies $\{\pi_h^{*}\}_{h=1}^H$ within the MDP.

It is worth observing that this auxiliary stochastic process closely resembles the original training process, with the sole addition being a dataset sampled under optimal policies. However, it is crucial to emphasize that this additional dataset does not influence our training course. Consequently, the results obtained from the original stochastic process remain valid in this auxiliary stochastic process. To formalize this concept, we define the following filtration: $\mathcal{F}_0 = \{\emptyset, \Omega\}$, $\mathcal{F}_1 = \sigma \left(\left\{D_1,D_1^*\right\}\right),\cdots$, $\mathcal{F}_k = \sigma \left(\left\{D_1,D_1^*,\cdots,D_k,D_k^*\right\}\right),\cdots$, $\mathcal{F}_K = \sigma \left(\left\{D_1,D_1^*,\cdots,D_K,D_K^*\right\}\right)$. Here, $\sigma({D})$ represents the filtration induced by the dataset $D$.

\subsection{Proof of Theorem \ref{Thm 3.1}}

In this section, we present the comprehensive proof of Theorem \ref{Thm 3.1}. Prior to providing the proof for the main theorem (Theorem \ref{Thm A.9}), it is necessary to establish the foundation through the following lemmas.

\begin{lemma} \label{A.1}
(Misspecification Error for Q-function).  For a $\zeta$-Average-Approximate Linear MDP (Assumption \ref{Ass1}), for any fixed policy $\pi$, any $h \in [H]$,  there exists weights $\{\mathbf{w}_h^{\pi}\}_{h \in [H]}$, where $\mathbf{w}_h^{\pi} = \boldsymbol{\theta}_h + \int V_{h+1}^{\pi}(s') d\boldsymbol{\mu}_h(s')$, such that for any $(s,a) \in \mathcal{S}\times\mathcal{A}$, 
$$
|Q_h^{\pi}(s,a)-\left\langle \phi(s,a),\mathbf{w}_h^{\pi} \right\rangle| \leq \eta_h(s,a) + H\cdot \xi_h(s,a)
$$
\end{lemma}

\begin{proof} 
This proof is straightforward by using the property of $Q$-function.

\begin{equation} \label{1}
    \begin{aligned}
    |Q_h^{\pi}(s,a)-\left\langle \phi(s,a),\mathbf{w}_h^{\pi} \right\rangle|
    &= \left|r_h(s,a)+\mathbb{P}_hV_{h+1}^{\pi}(s,a)-\left\langle \phi(s,a), \boldsymbol{\theta}_h + \int V_{h+1}^{\pi}(s') d\boldsymbol{\mu}_h(s') \right\rangle \right| \\ &\leq \left|r_h(s,a)-\left\langle \phi(s,a), \boldsymbol{\theta}_h \right\rangle\right| + \left|\mathbb{P}_hV_{h+1}^{\pi}(s,a)-\left\langle \phi(s,a), \int V_{h+1}^{\pi}(s') d\boldsymbol{\mu}_h(s') \right\rangle \right| \\ &\leq \eta_h(s,a) + H\cdot \xi_h(s,a)
    \end{aligned}
\end{equation}
\end{proof}

\begin{lemma} \label{A.2}
For any $h\in[H]$, 

\[
||\mathbf{w}_h^{\pi}|| \leq 2H\sqrt{d}
\]
\end{lemma}
\begin{proof}
For any policy $\pi$, $\mathbf{w}_h^{\pi} = \boldsymbol{\theta_h} + \int V_{h+1}^{\pi}(s') d\boldsymbol{\mu}_h(s')$, therefore, under Assumption \ref{Ass2}, we have:
\[
||\mathbf{w}_h^{\pi}|| \leq ||\boldsymbol{\theta}_h|| + ||\int V_{h+1}^{\pi}(s') d\boldsymbol{\mu}_h(s')|| \leq \sqrt{d} + H\sqrt{d} \leq 2H\sqrt{d}
\]
\end{proof}

\begin{lemma} \label{A.3}
Let $c_{\beta}$ be a constant in the definition of $\beta_k$, where $\beta_k = c_{\beta}\left(4\sqrt{kd}\zeta+\sqrt{(\lambda+1)d^2\log\left(\frac{4dKH}{\delta}\right)}\right)H$. Then under Assumption \ref{Ass1}, \ref{Ass2}, there exists an absolute constant $C$ that is independent of $c_{\beta}$ such that for any fixed $\delta \in [0,1]$, we have for all $(k,h)\in [K] \times [H]$, with probability at least $1-\delta$,
\[
\left|\left|\sum\limits_{\tau=1}^{k-1}\phi_h^{\tau}[V_{h+1}^k(s_{h+1}^{\tau})-\mathbb{P}_h V_{h+1}^k(s_h^{\tau},a_h^{\tau})] \right|\right|_{(\Lambda_h^k)^{-1}} \leq C\cdot dH\sqrt{\log\left[2(c_{\beta}+1)dKH/\delta\right]}
\]
\end{lemma}

\begin{proof} 
Lemmas \ref{C.3} and Lemma \ref{C.5} together imply that for all $(k,h) \in [K] \times [H]$,  with probability at least $1-\delta$,
\begin{equation} \label{Eq2}
    \begin{aligned}
         &\left|\left|\sum\limits_{\tau=1}^{k-1}\phi_h^{\tau}[V_{h+1}^k(s_{h+1}^{\tau})-\mathbb{P}_h V_{h+1}^k(s_h^{\tau},a_h^{\tau})] \right|\right|_{(\Lambda_h^k)^{-1}}^2  \\ &\leq  4H^2\left(\frac{d}{2}\log(\frac{k+\lambda}{\lambda})+ d\log\left(1+\frac{8H\sqrt{dk}}{\epsilon\sqrt{\lambda}}\right)+d^2\log\left[1+8d^{1/2}B^2/(\lambda\epsilon^2)\right] +\log(\frac{1}{\delta})\right) + \frac{8k^2\epsilon^2}{\lambda}
    \end{aligned}
\end{equation}
We let $\epsilon = dH/k$, and $B = \max_{k}\beta_k = c_{\beta}\left(4\sqrt{Kd}\zeta+\sqrt{(\lambda+1)d^2\log\left(\frac{4dKH}{\delta}\right)}\right)H$, then from Eq.(\ref{Eq2}), there exists a constant $C$ which is independent of $c_{\beta}$ such that 
\[
\left|\left|\sum\limits_{\tau=1}^{k-1}\phi_h^{\tau}[V_{h+1}^k(s_{h+1}^{\tau})-\mathbb{P}_h V_{h+1}^k(s_h^{\tau},a_h^{\tau})] \right|\right|_{(\Lambda_h^k)^{-1}} \leq C\cdot dH\sqrt{\log\left[2(c_{\beta}+1)dKH/\delta\right]}
\]

\end{proof}

\begin{lemma} \label{A.4}
During the course of training, we define the mixed misspecification error $\epsilon_h^{\tau} = (r_h-\widetilde{r_h})(s_h^{\tau},a_h^{\tau})+ \left(\mathbb{P}_h-\widetilde{\mathbb{P}_h}\right)V_{h+1}^k(s_h^{\tau},a_h^{\tau})$, $\ \forall (\tau,h) \in [K]\times[H]\ $,  then for any fixed policy $\pi$, conditioned on the event in Lemma \ref{A.3}, we have for all $(s,a,h,k) \in \mathcal{S}\times\mathcal{A}\times [H]\times[K]$, that
\begin{equation}
    \begin{aligned}
    &|\left\langle \phi(s,a), \mathbf{w}_h^k \right\rangle - Q_{h}^{\pi}(s,a)-\mathbb{P}_h(V_{h+1}^k-V_{h+1}^{\pi})(s,a)| \\ &\leq \lambda_h^k \sqrt{\phi(s,a)^{\top}(\Lambda_h^k)^{-1}\phi(s,a)} + 3H\cdot \xi_h(s,a)+ \eta_h(s,a)
    \end{aligned}
\end{equation}
where
$$
\lambda_h^k = 4H\sqrt{\lambda d}+C\cdot dH\sqrt{\log\left[2(c_{\beta}+1)dKH/\delta\right]}+\sqrt{d}\sqrt{\sum\limits_{\tau=1}^{k-1}(\epsilon_h^{\tau})^2}
$$
\end{lemma}

\begin{proof}
For any $(s,a)\in \mathcal{S}\times \mathcal{A}$, we have $\left\langle \phi(s,a), \mathbf{w}_h^{\pi} \right\rangle = \left\langle \phi(s,a), \boldsymbol{\theta}_h \right\rangle + \widetilde{\mathbb{P}_h}V_{h+1}^{\pi}(s,a)$. Therefore, we have
\begin{equation}
    \begin{aligned}
        \mathbf{w}_h^k-\mathbf{w}_h^{\pi}&=(\Lambda_h^k)^{-1}\sum\limits_{\tau=1}^{k-1}\phi_h^{\tau}[r_h^{\tau}+V_{h+1}^k(s_{h+1}^{\tau})]-\mathbf{w}_h^{\pi} \\ &= (\Lambda_h^k)^{-1}\{-\lambda \mathbf{w}_h^{\pi} +\sum\limits_{\tau=1}^{k-1}\phi_h^{\tau}[r_h^{\tau}+V_{h+1}^k(s_{h+1}^{\tau})-(\phi_h^{\tau})^{\top}\theta_h - \widetilde{\mathbb{P}_h}V_{h+1}^{\pi}(s_h^{\tau},a_h^{\tau})]\}\\ &=
        \underbrace{-\lambda (\Lambda_h^k)^{-1}\mathbf{w}_h^{\pi}}_{p_1} + \underbrace{(\Lambda_h^k)^{-1}\sum_{\tau=1}^{k-1}\phi_h^{\tau}[V_{h+1}^k(s_{h+1}^{\tau})-\mathbb{P}_hV_{h+1}^k(s_h^{\tau},a_h^{\tau})]}_{p_2} \\&+ \underbrace{(\Lambda_h^k)^{-1}\sum_{\tau=1}^{k-1}\phi_h^{\tau}\widetilde{\mathbb{P}}_h(V_{h+1}^k-V_{h+1}^{\pi})(s_h^{\tau}, a_h^{\tau})}_{p_3} + \underbrace{(\Lambda_h^k)^{-1}\sum_{\tau=1}^{k-1}\phi_h^{\tau} [r_h^{\tau}-(\phi_h^{\tau})^{\top}\boldsymbol{\theta}_h + \left(\mathbb{P}_h-\widetilde{\mathbb{P}_h}\right)V_{h+1}^k(s_h^{\tau},a_h^{\tau})]}_{p_4}   
        \end{aligned}
\end{equation}
For the last term, according to our definition in Lemma \ref{A.4}, $r_h^{\tau}-(\phi_h^{\tau})^{\top}\theta_h + \left(\mathbb{P}_h-\widetilde{\mathbb{P}_h}\right)V_{h+1}^k(s_h^{\tau},a_h^{\tau}) = \epsilon_h^{\tau}$, and notice that 
$|\epsilon_h^{\tau}| \leq \eta_h(s_h^{\tau},a_h^{\tau})+H\cdot\xi_h(s_h^{\tau},a_h^{\tau})$.

For the first term $p_1$, by Lemma \ref{A.2}, we have
$$
|\left\langle\phi(s,a),p_1\right\rangle| = |\lambda \left\langle\phi(s,a), (\Lambda_h^k)^{-1}\mathbf{w}_h^{\pi}\right\rangle| \leq \sqrt{\lambda}||\mathbf{w}_h^{\pi}||\sqrt{\phi(s,a)^{\top}(\Lambda_h^k)^{-1}\phi(s,a)} \leq 2H\sqrt{\lambda d} \sqrt{\phi(s,a)^{\top}(\Lambda_h^k)^{-1}\phi(s,a)}
$$
Conditioned on Lemma \ref{A.3}, we have

$$
|\left\langle\phi(s,a),p_2\right\rangle| \leq C\cdot dH\sqrt{\log\left[2(c_{\beta}+1)dKH/\delta\right]}\cdot\sqrt{\phi(s,a)^{\top}(\Lambda_h^k)^{-1}\phi(s,a)}
$$
For the third term, 
\begin{equation}
    \begin{aligned}
        \left\langle \phi(s,a),p_3\right\rangle &= \left\langle \phi(s,a),(\Lambda_h^k)^{-1}\sum\limits_{\tau=1}^{k-1}\phi_h^{\tau} \widetilde{\mathbb{P}}_h(V_{h+1}^k-V_{h+1}^{\pi})(s_h^{\tau}, a_h^{\tau}) \right\rangle \\ &= \left\langle \phi(s,a),(\Lambda_h^k)^{-1}\sum\limits_{\tau=1}^{k-1}\phi_h^{\tau}(\phi_h^{\tau})^{\top}\int(V_{h+1}^k-V_{h+1}^{\pi})(s')d\boldsymbol{\mu}_h(s') \right\rangle \\ &= \underbrace{\left\langle \phi(s,a), \int(V_{h+1}^k-V_{h+1}^{\pi})(s')d\boldsymbol{\mu}_h(s') \right\rangle}_{t_1}  \underbrace{-\lambda \left\langle \phi(s,a), (\Lambda_h^k)^{-1}\int(V_{h+1}^k-V_{h+1}^{\pi})(s')d\boldsymbol{\mu}_h(s')\right\rangle}_{t_2}  
    \end{aligned}
\end{equation}
Notice that $t_1 = \widetilde{\mathbb{P}_h}(V_{h+1}^k-V_{h+1}^{\pi})(s,a)$, $|t_2|\leq 2H\sqrt{d\lambda}\sqrt{\phi(s,a)^{\top}(\Lambda_h^k)^{-1}\phi(s,a)}$, and
$$
\left|t_1-\mathbb{P}_h(V_{h+1}^k-V_{h+1}^{\pi})(s,a)\right| = \left|(\widetilde{\mathbb{P}}_h-\mathbb{P}_h)(V_{h+1}^k-V_{h+1}^{\pi})(s,a)\right| \leq 2H\cdot \xi_h(s,a)
$$
For the last term $p_4$,
\begin{equation}
    \begin{aligned}
        |\left\langle \phi(s,a), p_4\right\rangle| &= \left|\left\langle \phi(s,a),(\Lambda_h^k)^{-1}\sum\limits_{\tau=1}^{k-1}\phi_h^{\tau}\epsilon_h^{\tau}\right\rangle\right|\\ &\leq \sum\limits_{\tau=1}^{k-1}|(\epsilon_{h}^{\tau} \phi)^{\top}(\Lambda_h^k)^{-1}\phi_h^{\tau}| \\  &\leq \sqrt{\left(\sum\limits_{\tau=1}^{k-1}(\epsilon_h^{\tau}\phi)^{\top}(\Lambda_h^k)^{-1}(\epsilon_h^{\tau}\phi)\right)\left(\sum\limits_{\tau=1}^{k-1}(\phi_h^{\tau})^{\top}(\Lambda_h^k)^{-1}\phi_h^{\tau}\right)} \\ &= \sqrt{\sum\limits_{\tau=1}^{k-1}(\epsilon_h^{\tau})^2}\cdot\sqrt{\phi(s,a)^{\top}(\Lambda_h^k)^{-1}\phi(s,a)}\cdot \sqrt{\sum\limits_{\tau=1}^{k-1}(\phi_h^{\tau})^{\top}(\Lambda_h^k)^{-1}\phi_h^{\tau}} \\ &\leq \sqrt{\sum\limits_{\tau=1}^{k-1}(\epsilon_h^{\tau})^2}\cdot\sqrt{\phi(s,a)^{\top}(\Lambda_h^k)^{-1}\phi(s,a)}\cdot \sqrt{d}  \ \ \ (\text{By Lemma \ref{C.1}})
    \end{aligned}
\end{equation}
Finally, Combined with the result in Lemma \ref{A.1}, we have 
\begin{equation}
    \begin{aligned}
    &|\left\langle \phi(s,a), \mathbf{w}_h^k \right\rangle - Q_{h}^{\pi}(s,a)-\mathbb{P}_h(V_{h+1}^k-V_{h+1}^{\pi})(s,a)| \\    & \leq |\left\langle \phi(s,a), \mathbf{w}_h^k-\mathbf{w}_h^{\pi} \right\rangle -\mathbb{P}_h(V_{h+1}^k-V_{h+1}^{\pi})(s,a)| + |Q_h^{\pi}(s,a)-\left\langle \phi(s,a),\mathbf{w}_h^{\pi} \right\rangle| \\ & \leq \underbrace{\left\{4H\sqrt{\lambda d}+C\cdot dH\sqrt{\log\left[2(c_{\beta}+1)dKH/\delta\right]}+\sqrt{d}\sqrt{\sum\limits_{\tau=1}^{k-1}(\epsilon_h^{\tau})^2}\right\}}_{\lambda_h^k}\sqrt{\phi(s,a)^{\top}(\Lambda_h^k)^{-1}\phi(s,a)} + 3H\cdot \xi_h(s,a)+ \eta_h(s,a)
     \end{aligned}
\end{equation}
\end{proof}

\begin{lemma} \label{A.5}
(Recursive formula).
We define $\delta_h^k = V_h^k(s_h^k)-V_h^{\pi_k}(s_h^k)$, and $\zeta_{h+1}^k = \mathbb{E}[\delta_{h+1}^k|s_h^k,a_h^k]-\delta_{h+1}^k$. Then, conditioned on the event in Lemma \ref{A.3}, we have for any $(k,h)\in [K]\times [H]$:
\begin{equation}
    \begin{aligned}
        \delta_h^k \leq \delta_{h+1}^k+\zeta_{h+1}^k + (\lambda_h^k+\beta_k)\sqrt{(\phi_h^k)^{\top}(\Lambda_h^k)^{-1}\phi_h^k} +  3H\cdot \xi_h(s_h^k,a_h^k) +\eta_h(s_h^k,a_h^k)
    \end{aligned}
\end{equation}
\end{lemma}

\begin{proof}
By Lemma \ref{A.4}, we have for any $(s,a,h,k)\in \mathcal{S}\times\mathcal{A}\times[H]\times[K]$:
\begin{equation}
    \begin{aligned}
        Q_h^k(s,a)-Q_h^{\pi_k}(s,a) &\leq \left\langle \phi(s,a), \mathbf{w}_h^k \right\rangle +\beta_k \sqrt{\phi(s,a)^{\top}(\Lambda_h^k)^{-1}\phi(s,a)}-Q_h^{\pi_k}(s,a) \\ &\leq \mathbb{P}_h(V_{h+1}^k - V_{h+1}^{\pi_k})(s,a) + (\lambda_h^k+\beta_k)\sqrt{\phi(s,a)^{\top}(\Lambda_h^k)^{-1}\phi(s,a)}+ 3H\cdot \xi_h(s,a) +\eta_h(s,a)
    \end{aligned}
\end{equation}
Notice that $\delta_h^k = Q_h^k(s_h^k,a_h^k) - Q_h^{\pi_k}(s_h^k,a_h^k)$, and $\mathbb{P}_h(V_{h+1}^k - V_{h+1}^{\pi_k})(s_h^k,a_h^k) = \mathbb{E}[\delta_{h+1}^k|s_h^k,a_h^k] $,

which finishes the proof.
\end{proof}

\begin{lemma} \label{A.6}
(Bound of cumulative misspecification error).
With probability at least $1-\delta$, 
\begin{equation}
    \begin{aligned}
        \sum\limits_{k=1}^K\sum\limits_{h=1}^H \xi_h(s_h^k,a_h^k) \leq\sqrt{8KH^2\log(\frac{4}{\delta})}+KH\zeta    
    \end{aligned}
\end{equation}

and

\begin{equation}
    \begin{aligned}
        \sum\limits_{k=1}^K\sum\limits_{h=1}^H \eta_h(s_h^k,a_h^k) \leq\sqrt{32dKH^2\log(\frac{4}{\delta})}+KH\zeta    
    \end{aligned}
\end{equation}
\end{lemma}

\begin{proof}
We define $X_k = \sum\limits_{h=1}^H \xi_h (s_h^k, a_h^k)$, $Z_0=0$, $Z_k = \sum\limits_{i=1}^k X_i-\sum\limits_{i=1}^k\mathbb{E}[X_i|\mathcal{F}_{i-1}]$, $k = 1,2, \cdots, K$.
Notice that $\{Z_k\}_{k=1}^K$ is a martingale, and $|Z_k-Z_{k-1}| = |X_k - \mathbb{E}[X_k | \mathcal{F}_{k-1}]|\leq 2H$, $\forall {k}\in [K]$. 

Then by Azuma-Hoeffding's inequality: For any $\epsilon >0$,

$$
\mathbb{P}\left(|Z_K-Z_0|\geq \epsilon\right) \leq 2\ \text{exp}\left\{\frac{-\epsilon^2}{2K\cdot 4H^2}\right\}
$$ 

which means that, with probability at least $1-\delta$, 
$$
|\sum\limits_{k=1}^K X_k- \sum\limits_{k=1}^K\mathbb{E}[X_k|\mathcal{F}_{k-1}]| \leq \sqrt{8KH^2\log(\frac{2}{\delta})}
$$
For the term $\sum\limits_{k=1}^K \mathbb{E}[X_k|\mathcal{F}_{k-1}]$, notice that

$$
\mathbb{E}[X_k|\mathcal{F}_{k-1}] = \sum\limits_{h=1}^H \mathbb{E}[\xi_h(s_h^k,a_h^k)|\mathcal{F}_{k-1}] = \sum\limits_{h=1}^H \mathbb{E}_{(s,a)\sim d_h^{\pi_k}}[\xi_h(s,a)] \leq H\zeta \ \ (\text{By Assumption \ref{Ass1}})
$$
Therefore, 
$$
\sum\limits_{k=1}^K \mathbb{E}[X_k|\mathcal{F}_{k-1}] \leq KH \zeta
$$
Finally,  with probablity at least $1-\delta$,

$$
 \sum\limits_{k=1}^K\sum\limits_{h=1}^H \xi_h(s_h^k,a_h^k) \leq\sqrt{8KH^2\log(\frac{2}{\delta})}+KH\zeta 
$$

and similarly, with probability at least $1-\delta$,

$$
\sum\limits_{k=1}^K\sum\limits_{h=1}^H \eta_h(s_h^k,a_h^k) \leq\sqrt{32dKH^2\log(\frac{2}{\delta})}+KH\zeta 
$$

By taking the union bound, we achieve the result.
\end{proof}

\begin{lemma} \label{A.7}
(Bound of bonus parameter).
With probability at least $1-\delta$, for all $(k,h)\in [K]\times[H]$, it holds that 

\[
\lambda_h^k \leq \beta_k
\]
where
\[
\beta_k = c_{\beta}\left(4\sqrt{kd}\zeta+\sqrt{(\lambda+1)d^2\log\left(\frac{4dKH}{\delta}\right)}\right)H
\]
\end{lemma}

\begin{proof}
For any fixed $(k,h) \in [K]\times[H] $, we have
\begin{equation}
    \begin{aligned}
        (\lambda_h^k)^2 =  \left(4H\sqrt{\lambda d}+C\cdot dH\sqrt{\log\left[2(c_{\beta}+1)dKH/\delta\right]}+\sqrt{d}\sqrt{\sum\limits_{\tau=1}^{k-1}(\epsilon_h^{\tau})^2}   \right)^2 \\ \leq 2\left( \left(4H\sqrt{\lambda d}+C\cdot dH\sqrt{\log\left[2(c_{\beta}+1)dKH/\delta\right]}\right)^2 + d\sum\limits_{\tau=1}^{k-1}(\epsilon_h^{\tau})^2  \right) 
    \end{aligned}
\end{equation}

Notice that 
\begin{equation}
    \begin{aligned}
        \sum\limits_{\tau=1}^{k-1}(\epsilon_h^{\tau})^2 &\leq \sum\limits_{\tau=1}^{k-1}\left(\eta_h(s_h^{\tau},a_h^{\tau})+H\cdot \xi_h(s_h^{\tau}, a_h^{\tau})\right)^2 \\ & \leq 2 \sum\limits_{\tau=1}^{k-1} \left(\eta_h^2(s_h^{\tau},a_h^{\tau}) +H^2 \xi_h^2(s_h^{\tau}, a_h^{\tau})\right)
    \end{aligned}
\end{equation}

\paragraph{Case 1 $k \geq \frac{64d^2\log(\frac{4}{\delta})}{\zeta^4}$}
By applying Azuma-Hoeffding's inequality, with probability at least $1-\delta/2$, we have
\begin{equation}
    \begin{aligned}
        \sum\limits_{\tau=1}^{k-1}\xi_h^2(s_h^{\tau},a_h^{\tau}) &\leq \sum\limits_{\tau=1}^{k-1}\mathbb{E}\left[\xi_h^2(s_h^{\tau},a_h^{\tau})|\mathcal{F}_{\tau-1}\right] +\sqrt{8k\log(\frac{4}{\delta})} \\ & \leq \sum\limits_{\tau=1}^k\mathbb{E}_{(s_h^{\tau},a_h^{\tau})\sim d_h^{\pi_{\tau}}}[\xi_h^2(s_h^{\tau},a_h^{\tau})] + \sqrt{8k\log(\frac{4}{\delta})} \\ &\leq k \zeta^2 + \sqrt{8k\log(\frac{4}{\delta})} \ \  (\text{By Assumption \ref{Ass1}})
    \end{aligned}
\end{equation}
In the same way, with probability at least $1-\delta/2$, we have
\begin{equation}
    \begin{aligned}
        \sum\limits_{\tau=1}^{k-1}\eta_h^2(s_h^{\tau},a_h^{\tau}) \leq k\zeta^2 +\sqrt{128kd^2\log(\frac{4}{\delta})}
    \end{aligned}
\end{equation}

Therefore, with probability at least $1-\delta$,

\begin{equation}
    \begin{aligned}
        \sum\limits_{\tau=1}^{k-1}(\epsilon_h^{\tau})^2 &\leq 2 \left(k\zeta^2 +\sqrt{128kd^2\log(\frac{4}{\delta})}\right) + 2H^2 \left(k \zeta^2 + \sqrt{8k\log(\frac{4}{\delta})}\right) \\ &\leq 4kH^2\zeta^2 + 32dH^2\sqrt{k\log(\frac{4}{\delta})} \\ &\leq 8kH^2\zeta^2 
    \end{aligned}
\end{equation}

\paragraph{Case 2 $k < \frac{64d^2\log(\frac{4}{\delta})}{\zeta^4}$}
We denote $X_{\tau} = \xi_h^2(s_h^{\tau}, a_h^{\tau}) - \mathbb{E}[\xi_h^2(s_h^{\tau},a_h^{\tau})|\mathcal{F}_{\tau-1}]$, and $Y_0 = 0$, $Y_{\tau} = \sum\limits_{i=1}^{\tau}X_i$, $\tau = 1,2,\cdots, k-1$.

Notice that $\{Y_{\tau}\}_{\tau=1}^{k-1}$ is a martingale.
\begin{equation}
    \begin{aligned}
        \sum\limits_{\tau=1}^{k-1}\mathbb{E}[X_{\tau}^2|\mathcal{F}_{\tau-1}] & = \sum\limits_{\tau=1}^{k-1}\mathbb{E}\left[\left( \xi_h^2(s_h^{\tau}, a_h^{\tau}) - \mathbb{E}[\xi_h^2(s_h^{\tau},a_h^{\tau})|\mathcal{F}_{\tau-1}]\right)^2|\mathcal{F}_{\tau-1}\right] \\ &\leq \sum\limits_{\tau =1}^{k-1}\mathbb{E}\left[\xi_h^4(s_h^{\tau},a_h^{\tau})+\left(\mathbb{E}[\xi_h^2(s_h^{\tau},a_h^{\tau})|\mathcal{F}_{\tau-1}]\right)^2|\mathcal{F}_{\tau-1}\right] \\ &= \sum\limits_{\tau=1}^{k-1}\mathbb{E}\left[\xi_h^4(s_h^{\tau},a_h^{\tau})|\mathcal{F}_{\tau-1}\right]+\sum\limits_{\tau=1}^{k-1}\left(\mathbb{E}[\xi_h^2(s_h^{\tau},a_h^{\tau})|\mathcal{F}_{\tau-1}]\right)^2 \\ &\leq 2\sum\limits_{\tau=1}^{k-1}\mathbb{E}\left[\xi_h^4(s_h^{\tau},a_h^{\tau})|\mathcal{F}_{\tau-1}\right] \ \  (\text{By Jensen's inequality}) \\ & \leq 2k\zeta^4 \ \ (\text{By Assumption \ref{Ass1}})
    \end{aligned}
\end{equation}

By applying Freedman's inequality (Lemma \ref{C.7}), we have for any $t\geq 0$, that
$$
\mathbb{P}\left(|Y_k-Y_0|\geq t\right) \leq 2\exp\left\{-\frac{t^2/2}{2k\zeta^4+2t/3}\right\} \leq 2\exp \left\{ -\frac{t^2/2}{128d^2\log\left(\frac{4}{\delta}\right)+2t/3}\right\}
$$
We let the rightmost term in the above formula to be $\delta/2$, and by solving the quadratic equation with respect to $t$, we get $t = C_1\cdot d \log(\frac{4}{\delta})$, where $C_1$ is some constant. Therefore, we have with probability at least $1-\delta/2$,  
$$
\sum\limits_{\tau=1}^{k-1}\xi_h^2(s_h^{\tau},a_h^{\tau}) \leq k\zeta^2 + C_1\cdot d \log(\frac{4}{\delta})
$$
Similarly, with probability at least $1-\delta/2$, 
$$
\sum\limits_{\tau=1}^{k-1}\eta_h^2(s_h^{\tau},a_h^{\tau}) \leq k\zeta^2 + C_2\cdot d \log(\frac{4}{\delta})
$$
where $C_2$ is some constant.
Therefore, in this case, with probability at least $1-\delta$,
\[
 \sum\limits_{\tau=1}^{k-1}(\epsilon_h^{\tau})^2 \leq 4kH^2\zeta^2 + C'dH^2\log(\frac{4}{\delta})
\]
where $C' = 2(C_1+C_2)$ is also some constant.

Combining two cases, we have for any fixed $(k,h)\in [K]\times [H]$,  with probability at least $1-\delta$,
\begin{equation}
    \begin{aligned}
        \sum\limits_{\tau=1}^{k-1}(\epsilon_h^{\tau})^2 \leq 8kH^2\zeta^2 + C'dH^2\log(\frac{4}{\delta})
    \end{aligned}
\end{equation}
Finally, by taking the union bound of all $(k,h)\in [K]\times[H]$, we have with probability at least $1-\delta$, 
\begin{equation}
    \begin{aligned}
        (\lambda_h^k)^2 &\leq 32dH^2\lambda + 4C^2d^2H^2\log\left(\frac{2(c_{\beta}+1)dKH}{\delta}\right)+ 16dkH^2\zeta^2 + 2C'd^2H^2\log(\frac{4KH}{\delta}) \\ &\leq 16dkH^2\zeta^2 + 32(\lambda+1)(C+C')^2d^2H^2\log\left(\frac{4(c_{\beta}+1)dKH}{\delta}\right)
    \end{aligned}
\end{equation}

This means that
\[
\lambda_h^k \leq 4\sqrt{kd}H\zeta + \sqrt{32(\lambda+1)(C+C')^2}dH\sqrt{\log\left(\frac{4(c_{\beta}+1)dKH}{\delta}\right)}
\]
By choosing an appropriate $c_{\beta}$, we have
\[
\lambda_h^k \leq c_{\beta}\left(4\sqrt{kd}\zeta+\sqrt{(\lambda+1)d^2\log\left(\frac{4dKH}{\delta}\right)}\right)H = \beta_k
\]

\end{proof}

\begin{lemma} \label{A.8}
(Near-Optimism)
Given K initial points $\{s_1^{k}\}_{k=1}^K$, we use $\{(s_h^{k*},a_h^{k*})\}_{(k,h)\in [K]\times[H]}\ $ (where $s_1^{k*} = s_1^k$, $\forall k \in [K]$) $\ $ to represent the dataset sampled by the optimal policy $\pi^*$ in the true environment  (This dataset is impossible to obtain in reality, but it can be used for  our analysis), and we denote 
$\delta_h^{k*} = V_h^*(s_h^{k*}) - V_h^k(s_h^{k*})$, and $\zeta_h^{k*} = \mathbb{E}\left[\delta_{h+1}^{k*}|s_h^{k*},a_h^{k*}\right] - \delta_{h+1}^{k*}$. Then conditioned on the event in Lemma \ref{A.7}, with probability at least $1-\delta$, we have 
\begin{equation}
    \begin{aligned}
         \sum\limits_{k=1}^K[V_1^*(s_1^{k})-V_1^k(s_1^{k})] & \leq 4KH^2\zeta +12H^2\sqrt{dK\log(\frac{8}{\delta})}
    \end{aligned}
\end{equation}

\end{lemma}

\begin{proof}
First of all, by the definition of $V_1^k$, we have
\[
V_1^k(s_1^{k*}) = \max_{a\in\mathcal{A}}Q_1^k(s_1^{k*},a) \geq Q_1^k(s_1^{k*},a_1^{k*})
\]
Therefore, 
\begin{equation} \label{20}
    \begin{aligned}
        \sum\limits_{k=1}^K[V_1^*(s_1^{k*})-V_1^k(s_1^{k*})] &\leq \sum\limits_{k=1}^K \left( Q_1^*(s_1^{k*},a_1^{k*}) - Q_1^k(s_1^{k*},a_1^{k*}) \right)
    \end{aligned}
\end{equation}
Notice that if we have $Q_1^k(s_1^{k*}, a_1^{k*}) = H$ for some $k$, then $ Q_1^*(s_1^{k*},a_1^{k*}) \leq Q_1^k(s_1^{k*},a_1^{k*})$, which means we achieve an optimistic estimate, and this term will less than or equal to $0$ in Eq.(\ref{20}). Therefore, we only need to consider the situation when $Q_1^k(s_1^{k*},a_1^{k*})<H$, $\forall k \in [K]$. 

In this case,
\[
Q_1^k(s_1^{k*},a_1^{k*}) = \left\langle \phi(s_1^{k*},a_1^{k*}), \textbf{w}_1^k \right\rangle + \beta_k\sqrt{\phi(s_1^{k*},a_1^{k*})^{\top}(\Lambda_1^k)^{-1}\phi(s_1^{k*},a_1^{k*})} 
\]

Then we have
\begin{equation}
    \begin{aligned}
         &\ \ \ \ \sum\limits_{k=1}^K \left( Q_1^*(s_1^{k*},a_1^{k*}) - Q_1^k(s_1^{k*},a_1^{k*}) \right) \\ &=  \sum\limits_{k=1}^K \left( Q_1^*(s_1^{k*},a_1^{k*}) - \left\langle \phi(s_1^{k*},a_1^{k*}), \textbf{w}_1^k \right\rangle - \beta_k\sqrt{\phi(s_1^{k*},a_1^{k*})^{\top}(\Lambda_1^k)^{-1}\phi(s_1^{k*},a_1^{k*})} \right)  \\ &\leq \sum\limits_{k=1}^K \left( (\lambda_h^k - \beta_k)\sqrt{\phi(s_1^{k*},a_1^{k*})^{\top}(\Lambda_1^k)^{-1}\phi(s_1^{k*},a_1^{k*}) } +3H\xi_1(s_1^{k*},a_1^{k*}) + \eta_1(s_1^{k*},a_1^{k*}) - \mathbb{P}_1(V_2^k-V_2^*)(s_1^{k*},a_1^{k*}) \right)
    \end{aligned}
\end{equation}
where the last inequality is derived by Lemma \ref{A.4}.

By conditioning on the event in Lemma \ref{A.7}, we know that $\lambda_h^k \leq \beta_k$, $\forall (k,h) \in [K] \times [H]$.

In addition,
\begin{equation}
    \begin{aligned}
        \mathbb{P}_1(V_2^*-V_2^{k})(s_1^{k*},a_1^{k*}) = \mathbb{E}_{s_2^{k*}\sim \mathbb{P}_1(\cdot|s_1^{k*},a_1^{k*})}\left[V_2^*(s_2^{k*})-V_2^k(s_2^{k*})\right]
    \end{aligned}
\end{equation}
and
\[
V_2^*(s_2^{k*})-V_2^k(s_2^{k*}) \leq Q_2^*(s_2^{k*},a_2^{k*})-Q_2^k(s_2^{k*},a_2^{k*})
\]

Similar to Eq.(\ref{20}), we only need to consider the case when $Q_2^k(s_2^{k*},a_2^{k*})<H$, $\forall k \in [K]$. In this way, we can recursively use Lemma \ref{A.4}.

Therefore, 
\begin{equation}
    \begin{aligned}
        &\sum\limits_{k=1}^K \left( Q_1^*(s_1^{k*},a_1^{k*}) - Q_1^k(s_1^{k*},a_1^{k*}) \right) \\  &\leq \sum\limits_{k=1}^K \left( 3H \cdot \xi_1(s_1^{k*},a_1^{k*}) + \eta_1(s_1^{k*},a_1^{k*}) + \delta_2^{k*} + \zeta_{1}^{k*} \right) \\ & \leq \sum\limits_{k=1}^K \left( 3H \cdot \xi_1(s_1^{k*},a_1^{k*}) + \eta_1(s_1^{k*},a_1^{k*}) + 3H\cdot \xi_2(s_2^{k*},a_2^{k*}) +\eta_2(s_2^{k*},a_2^{k*})+ \delta_3^{k*}+ \zeta_2^{k*}+\zeta_{1}^{k*} \right) \\ & \leq \cdots \leq  3H\sum\limits_{k=1}^K\sum\limits_{h=1}^H \xi_h(s_h^{k*},a_h^{k*}) + \sum\limits_{k=1}^K\sum\limits_{h=1}^H \eta_h(s_h^{k*},a_h^{k*})  + \sum\limits_{k=1}^K\sum\limits_{h=1}^H  \zeta_h^{k*}
    \end{aligned}
\end{equation}

Similar to Lemma \ref{A.6}, We define $X_{k}^* = \sum\limits_{h=1}^H \xi_h (s_h^{k*}, a_h^{k*})$, $Z_0^*=0$, $Z_k^* = \sum\limits_{i=1}^k X_i^*-\sum\limits_{i=1}^k\mathbb{E}[X_i^*|\mathcal{F}_{i-1}]$, $k = 1,2, \cdots, K$.
Notice that $\{Z_k^*\}_{k=1}^K$ is a martingale, and $|Z_k^*-Z_{k-1}^*| = |X_k^* - \mathbb{E}[X_k^* | \mathcal{F}_{k-1}]|\leq 2H$, $\forall {k}\in [K]$. 

Then by Azuma-Hoeffding's inequality: For any $\epsilon >0$,

$$
\mathbb{P}\left(|Z_K^*-Z_0^*|\geq \epsilon\right) \leq 2\ \text{exp}\left\{\frac{-\epsilon^2}{2K\cdot 4H^2}\right\}
$$ 

which means that, with probability at least $1-\delta$, 
$$
|\sum\limits_{k=1}^K X_k^*- \sum\limits_{k=1}^K\mathbb{E}[X_k^*|\mathcal{F}_{k-1}]| \leq \sqrt{8KH^2\log(\frac{2}{\delta})}
$$
For the term $\sum\limits_{k=1}^K \mathbb{E}[X_k^*|\mathcal{F}_{k-1}]$, notice that

$$
\mathbb{E}[X_k^*|\mathcal{F}_{k-1}] = \sum\limits_{h=1}^H \mathbb{E}[\xi_h(s_h^{k*},a_h^{k*})|\mathcal{F}_{k-1}] = \sum\limits_{h=1}^H \mathbb{E}_{(s,a)\sim d_h^{\pi^*}}[\xi_h(s,a)] \leq H\zeta \ \ (\text{By Assumption \ref{Ass1}})
$$
Therefore, 
$$
\sum\limits_{k=1}^K \mathbb{E}[X_k^*|\mathcal{F}_{k-1}] \leq KH \zeta
$$
Then,  with probability at least $1-\delta$,

$$
 \sum\limits_{k=1}^K\sum\limits_{h=1}^H \xi_h(s_h^{k*},a_h^{k*}) \leq\sqrt{8KH^2\log(\frac{2}{\delta})}+KH\zeta 
$$
and similarly, with probability at least $1-\delta$,

$$
\sum\limits_{k=1}^K\sum\limits_{h=1}^H \eta_h(s_h^{k*},a_h^{k*}) \leq\sqrt{32dKH^2\log(\frac{2}{\delta})}+KH\zeta 
$$

For the last term $\sum\limits_{k=1}^K\sum\limits_{h=1}^H  \zeta_h^{k*}$, notice that $\{\zeta_h^{k*}\}$ is a martingale difference sequence, and each term is upper bounded by $2H$. By using Azuma-Hoeffding's inequality, with probability at least $1-\delta/4$, the following inequality holds:
\begin{equation} \label{zeta}
    \begin{aligned}
        \sum\limits_{k=1}^K\sum\limits_{h=1}^H \zeta_h^{k*} \leq \sqrt{8KH^3\cdot \log(8/\delta)}
    \end{aligned}
\end{equation}

Finally, with probability at least $1-\delta$, 
\begin{equation}
    \begin{aligned}
        \sum\limits_{k=1}^K[V_1^*(s_1^{k})-V_1^k(s_1^{k})] &\leq 3H\sum\limits_{k=1}^K\sum\limits_{h=1}^H \xi_h(s_h^{k*},a_h^{k*}) + \sum\limits_{k=1}^K\sum\limits_{h=1}^H \eta_h(s_h^{k*},a_h^{k*})  + \sum\limits_{k=1}^K\sum\limits_{h=1}^H  \zeta_h^{k*} \\ &\leq 4KH^2\zeta +12H^2\sqrt{dK\log(\frac{8}{\delta})}
    \end{aligned}
\end{equation}

\end{proof}

\begin{theorem} \label{Thm A.9}
\textbf{(Regret Bound under $\zeta$-Average-Approximate Linear MDP)}.

Under our Assumption \ref{Ass1}, \ref{Ass2} , for any fixed $\delta \in (0,1)$, with probability at least $1-\delta$,  the total regret of the algorithm Robust-LSVI (Algorithm \ref{Algorithm known}) is at most $\widetilde{O}\left(dKH^2\zeta+ \sqrt{d^3KH^4} \right)$. 
\end{theorem}

\begin{proof}
First of all, we do the following decomposition.
\begin{equation}
    \begin{aligned}
        \text{Regret}(K) = \sum\limits_{k=1}^K[V_1^*(s_1^k)-V_1^{\pi_k}(s_1^k)] = \underbrace{\sum\limits_{k=1}^K[V_1^*(s_1^k)-V_1^k(s_1^k)]}_{A}+\underbrace{\sum\limits_{k=1}^K[V_1^k(s_1^k)-V_1^{\pi_k}(s_1^k)]}_{B}    
    \end{aligned}
\end{equation}

For the term A, from Lemma \ref{A.8} (Near-Optimism), we have with probability at least $1-\frac{\delta}{2}$,
\begin{equation} \label{Eq 27}
    \begin{aligned}
        \sum\limits_{k=1}^K[V_1^*(s_1^k)-V_1^k(s_1^k)] &\leq  4KH^2\zeta +12H^2\sqrt{dK\log(\frac{16}{\delta})}
    \end{aligned}
\end{equation}

For the term B, from Lemma \ref{A.5} (Recursive formula), we have
\begin{equation} \label{term B}
    \begin{aligned}
        &\sum\limits_{k=1}^K[V_1^k(s_1^k)-V_1^{\pi_k}(s_1^k)] = \sum\limits_{k=1}^K \delta_1^k \\ &\leq \sum\limits_{k=1}^K\sum\limits_{h=1}^H \zeta_h^k + \sum\limits_{k=1}^K\beta_k\sum_{h=1}^H \sqrt{(\phi_h^k)^{\top}(\Lambda_h^k)^{-1}\phi_h^k} + \sum\limits_{k=1}^K\sum\limits_{h=1}^H \lambda_h^k \sqrt{(\phi_h^k)^{\top}(\Lambda_h^k)^{-1}\phi_h^k} \\ &+3H \sum\limits_{k=1}^K\sum\limits_{h=1}^H \xi_h(s_h^k,a_h^k) + \sum\limits_{k=1}^K\sum\limits_{h=1}^H \eta_h(s_h^k,a_h^k)
    \end{aligned}
\end{equation}
Notice that $\{\zeta_h^k\}$ is a martingale difference sequence, and each term is upper bounded by $2H$. By using Azuma-Hoeffding's inequality, with probability at least $1-\delta/4$, the following inequality holds:
\begin{equation} \label{zeta}
    \begin{aligned}
        \sum\limits_{k=1}^K\sum\limits_{h=1}^H \zeta_h^k \leq \sqrt{8KH^3\cdot \log(8/\delta)}
    \end{aligned}
\end{equation}

By Lemma \ref{A.7}, we have
\begin{equation} \label{Eq 29}
    \begin{aligned}
        \sum\limits_{k=1}^K\sum\limits_{h=1}^H \lambda_h^k \sqrt{(\phi_h^k)^{\top}(\Lambda_h^k)^{-1}\phi_h^k} \leq \sum\limits_{k=1}^K\beta_k\sum_{h=1}^H \sqrt{(\phi_h^k)^{\top}(\Lambda_h^k)^{-1}\phi_h^k}
    \end{aligned}
\end{equation}

and by using Cauchy-Schwarz inequality, we have
\begin{equation} \label{Eq 30}
    \begin{aligned}
         \sum\limits_{k=1}^K\beta_k \sqrt{(\phi_h^k)^{\top}(\Lambda_h^k)^{-1}\phi_h^k} \leq \left[\sum\limits_{k=1}^K \beta_k^2\right]^{\frac{1}{2}}\cdot \left[\sum\limits_{k=1}^K (\phi_h^k)^{\top}(\Lambda_h^k)^{-1}\phi_h^k\right]^{\frac{1}{2}}
    \end{aligned}
\end{equation}
and
\begin{equation} \label{31}
    \begin{aligned}
        \sum\limits_{k=1}^K \beta_k^2 &= \sum\limits_{k=1}^K \left(c_{\beta}^2\left(4\sqrt{kd}\zeta+\sqrt{(\lambda+1)d^2\log\left(\frac{4dKH}{\delta}\right)}\right)^2H^2\right) \\ &\leq 2\sum\limits_{k=1}^K c_{\beta}^2H^2\left(16kd\zeta^2+ (\lambda+1)d^2\log\left(\frac{4dKH}{\delta}\right)\right) \\ &\leq 32c_{\beta}^2K^2H^2d\zeta^2 + 2c_{\beta}^2(\lambda+1)KH^2d^2\log\left(\frac{4dKH}{\delta}\right)
    \end{aligned}
\end{equation}
Therefore, 
\begin{equation} \label{Eq 32}
    \begin{aligned}
        \left[\sum\limits_{k=1}^K \beta_k^2\right]^{\frac{1}{2}} \leq 8c_{\beta} \sqrt{d}KH\zeta + 2c_{\beta}(\lambda+1)Hd \sqrt{K\log\left(\frac{4dKH}{\delta}\right)}
    \end{aligned}
\end{equation}
By Lemma \ref{C.2}, we have for any $h \in [H]$, 
\[
\sum\limits_{k=1}^K (\phi_h^k)^{\top}(\Lambda_h^k)^{-1}\phi_h^k \leq 2\log\left[\frac{\text{det}(\Lambda_h^{k+1})}{\text{det}(\Lambda_h^1)}\right] \leq 2d\log\left[\frac{\lambda+k}{\lambda}\right] \leq 2d\log\left(\frac{2dKH}{\delta}\right)
\]
thus
\begin{equation} \label{Eq 33}
    \begin{aligned}
        \left[\sum\limits_{k=1}^K (\phi_h^k)^{\top}(\Lambda_h^k)^{-1}\phi_h^k\right]^{\frac{1}{2}} \leq \sqrt{2d\log\left(\frac{2dKH}{\delta}\right)}
    \end{aligned}
\end{equation}
By combining (\ref{Eq 30}), (\ref{Eq 32}), and (\ref{Eq 33}), we have
\begin{equation} \label{Eq 34}
    \begin{aligned}
       \sum\limits_{k=1}^K\beta_k\sum_{h=1}^H \sqrt{(\phi_h^k)^{\top}(\Lambda_h^k)^{-1}\phi_h^k}\leq 16 c_{\beta}dKH^2 \zeta\sqrt{\log\left(\frac{2dKH}{\delta}\right)} + 4c_{\beta}(\lambda+1)\sqrt{Kd^3H^4}\cdot \log\left(\frac{4dKH}{\delta}\right)
    \end{aligned}
\end{equation}

Using Lemma \ref{A.6} (Bound of cumulative misspecification error), (\ref{zeta}), (\ref{Eq 29}), and (\ref{Eq 34}), we can give a bound of (\ref{term B}), that with probability at least $1-\delta/2$, 
\begin{equation} \label{Eq 36}
    \begin{aligned}
        &\sum\limits_{k=1}^K[V_1^k(s_1^k)-V_1^{\pi_k}(s_1^k)]  \leq 36 c_{\beta}dKH^2 \zeta\sqrt{\log\left(\frac{2dKH}{\delta}\right)} + 20c_{\beta}(\lambda+1)\sqrt{Kd^3H^4}\cdot \log\left(\frac{4dKH}{\delta}\right)
    \end{aligned}
\end{equation}
Finally, by combining (\ref{Eq 27}) and (\ref{Eq 36}), we get the regret bound:
\begin{equation}
    \begin{aligned}
        \text{Regret}(K) \leq 40 c_{\beta}dKH^2 \zeta\sqrt{\log\left(\frac{2dKH}{\delta}\right)} + 32c_{\beta}(\lambda+1)\sqrt{Kd^3H^4}\cdot \log\left(\frac{4dKH}{\delta}\right)
    \end{aligned}
\end{equation}

This indicates that our regret bound is $\widetilde{O}(dKH^2\zeta+ \sqrt{d^3KH^4} )$, which finishes our proof. 

\end{proof}

\section{Analysis of Robust Value-based Algorithm with General Function Approximation for Locally-bounded Misspecified MDP} \label{Sec: general value}

\begin{assumption} (General function approximation with locally-bounded misspecification error) \label{Ass:general-value_appendix}

Given the MDP $M$ with the transition model $P$, we assume that there exists a function class $\mathcal{F} \subset \{f:\mathcal{S}\times\mathcal{A} \rightarrow [0,H]\}$ and  a real number $\zeta \in [0,1]$, such that for any $V : \mathcal{S} \rightarrow [0,H]$, there exists $\bar{f}_V \in \mathcal{F}$, which satisfies :  $\forall \beta \in [4]$,

$$
\sup_{\pi} \mathbb{E}_{(s,a)\sim d^{\pi}}\left|\bar{f}_V(s,a) - \left(r(s,a)+\mathbb{P}V(s,a)\right)\right|^{\beta} \leq \zeta^{\beta}
$$

\end{assumption}
To simplify the notation, for a fixed $k \in [K]$, we let $\mathcal{Z}^k = \{(s_h^{k'},a_h^{k'})\}_{(k',h) \in [k-1]\times [H]}$. 

For any $V: \mathcal{S} \rightarrow [0,H]$, define 
    $$
    \mathcal{D}_V^k : = \{(s_h^{k'},a_h^{k'},r_h^{k'}+ V(s_{h+1}^{k'}))\}_{(k',h) \in [k-1]\times [H]}
    $$ 

and the accordingly minimizer
    $$
    \widehat{f}_V: = \text{arg}\min\limits_{f\in \mathcal{F}} || f||_{\mathcal{D}_V^k}^2
    $$

\begin{lemma} \label{single-f-error(general value)}
For any fixed $f \in \mathcal{F}$, and fixed $V: \mathcal{S} \rightarrow [0,H]$, with probability at least $1-\delta$,  we have for all  $k \in [K]$ that

\begin{equation}
    \begin{aligned}
        ||\bar{f}_V||_{\mathcal{D}_V^k}^2- ||f||_{\mathcal{D}_V^k}^2-\frac{2}{3} H^2\log(\frac{1}{\delta})- ||f-\bar{f}_V ||_{\mathcal{Z}^k}\cdot  \sqrt{kH\zeta^2 + C'H\cdot \log(\frac{4KH}{\delta})} + \frac{1}{4} ||f-\bar{f}_V ||_{\mathcal{Z}^k}^2 \leq 0
    \end{aligned}
\end{equation}

\begin{proof}
        For each $(k,h)\in [K]\times [H]$, and a fixed $f \in \mathcal{F}$, we define 
    $$
    \mathcal{Z}_h^k = \left(\bar{f}_V(s_h^k,a_h^k)-r(s_h^k,a_h^k) -V(s_{h+1}^k)\right)^2 - \left(f(s_h^k,a_h^k)-r(s_h^k,a_h^k) -V(s_{h+1}^k)\right)^2
    $$
    $$
    \epsilon_h^k = V(s_{h+1}^k) - \mathbb{P}V(s_h^k,a_h^k), \ \ \ \xi_h^k = r(s_h^k,a_h^k) + \mathbb{P}V(s_h^k,a_h^k) -\bar{f}_V(s_h^k,a_h^k)
    $$
    and set $\mathbb{F}_h^k$ be the $\sigma$-algebra generated by $\{(s_{h'}^{k'},a_{h'}^{k'})\}_{(h',k')\in [H]\times [k-1]} \cup \{(s_1^k,a_1^k),(s_2^k,a_2^k),\cdots, (s_h^k,a_h^k)\}$.

Actually, $\sum\limits_{k'=1}^{k-1}\sum\limits_{h=1}^H \mathcal{Z}_h^{k'} = ||\bar{f}_V||_{\mathcal{D}_V^k}^2- ||f||_{\mathcal{D}_V^k}^2$, and after a simple calculation, we have

$$
\mathcal{Z}_h^k = -\left(f(s_h^k,a_h^k) - \bar{f}_V(s_h^k,a_h^k)\right)^2 + 2\left(f(s_h^k,a_h^k) - \bar{f}_V(s_h^k,a_h^k)\right)\left(\underbrace{r(s_h^k,a_h^k)+V(s_{h+1}^k)-\bar{f}_V(s_h^k,a_h^k)}_{\epsilon_h^k + \xi_h^k}\right)
$$

    Notice that $\mathbb{E}[\epsilon_h^k | \mathbb{F}_h^k] = 0$, and since $\epsilon_h^k$ is bounded in $[-H,H]$, hence, $\epsilon_h^k$ is $H$-subguassian. That is to say, for any $\lambda \in \mathbb{R}$, we have $\mathbb{E}[\exp\{\lambda \epsilon_h^k\}| \mathbb{F}_h^k] \leq \exp\{\frac{\lambda^2H^2}{8}\}$.

Moreover, under Assumption \ref{Ass:general-value}, using the same argument in Lemma \ref{A.7}, with probability at least $1-\delta/2$, for all $(k,h) \in [K]\times [H]$, we have

\begin{equation} \label{union bound for miss error(general value)}
    \begin{aligned}
        \sum\limits_{k'=1}^k\sum\limits_{h=1}^H (\xi_h^{k'})^2 \leq kH\zeta^2 + C'H\cdot \log(\frac{8KH}{\delta})
    \end{aligned}
\end{equation}

where $C'$ is some constant.

Therefore, the conditional mean and the conditional cumulant generating function of the centered random variable can be calculated.

$$
\mu_h^k = \mathbb{E}[\mathcal{Z}_h^k| \mathbb{F}_h^k] = -\left(f(s_h^k,a_h^k) - \bar{f}_V(s_h^k,a_h^k)\right)^2 + 2\left(f(s_h^k,a_h^k) - \bar{f}_V(s_h^k,a_h^k)\right)\xi_h^k
$$
and
\begin{equation}
    \begin{aligned}
        \phi_h^k(\lambda) &= \log \mathbb{E}\left[\exp\{\lambda(\mathcal{Z}_h^k -\mu_h^k)\}| \mathbb{F}_h^k\right] \\ &= \log \mathbb{E}\left[\exp \{2\lambda \left(f(s_h^k,a_h^k) - \bar{f}_V(s_h^k,a_h^k)\right)\epsilon_h^k\}|\mathbb{F}_h^k\right] \\ &\leq  \frac{\lambda^2 \left(f(s_h^k,a_h^k) - \bar{f}_V(s_h^k,a_h^k)\right)^2 H^2}{2}
\end{aligned}
\end{equation}

By using Lemma \ref{russo}, we have for any $x\geq 0$, $\lambda \geq 0$,
\begin{equation} \label{E.5 (general value)}
    \begin{aligned}
    \mathbb{P}\Big(\lambda \sum\limits_{k'=1}^{k-1}\sum\limits_{h=1}^H \mathcal{Z}_h^{k'} &\leq x -\lambda \sum\limits_{k'=1}^{k-1}\sum\limits_{h=1}^H \left(f(s_h^{k'},a_h^{k'}) - \bar{f}_V(s_h^{k'},a_h^{k'})\right)^2  \\ & +2\lambda \sum\limits_{k'=1}^{k-1}\sum\limits_{h=1}^H \left(f(s_h^{k'},a_h^{k'}) - \bar{f}_V(s_h^{k'},a_h^{k'})\right)\xi_h^{k'} \\ & + \frac{\lambda^2H^2}{2}\sum\limits_{k'=1}^{k-1}\sum\limits_{h=1}^H \left(f(s_h^{k'},a_h^{k'}) - \bar{f}_V(s_h^{k'},a_h^{k'})\right)^2, \forall (k,h) \in [K]\times [H]\Big) \geq  1-e^{-x}
    \end{aligned}
\end{equation}
After setting $x = \log(\frac{2}{\delta})$, $\lambda = \frac{3}{2H^2}$, and conditioned the event that Eq.(\ref{union bound for miss error(general value)}) holds, that is to say, 
\begin{equation}
    \begin{aligned}
        &\ \ \ \ \ \sum\limits_{k'=1}^{k-1}\sum\limits_{h=1}^H \left(f(s_h^{k'},a_h^{k'}) - \bar{f}_V(s_h^{k'},a_h^{k'})\right)\xi_h^{k'} \\ &\leq  \sqrt{\sum\limits_{k'=1}^{k-1}\sum\limits_{h=1}^H \left(f(s_h^{k'},a_h^{k'}) - \bar{f}_V(s_h^{k'},a_h^{k'})\right)^2} \cdot  \sqrt{\sum\limits_{k'=1}^{k-1}\sum\limits_{h=1}^H (\xi_h^{k'})^2} \\ & \leq  \sqrt{\sum\limits_{k'=1}^{k-1}\sum\limits_{h=1}^H \left(f(s_h^{k'},a_h^{k'}) - \bar{f}_V(s_h^{k'},a_h^{k'})\right)^2} \cdot  \sqrt{kH\zeta^2 + C'H\cdot \log(\frac{8KH}{\delta})}
    \end{aligned}
\end{equation}

Then we can derive the following result by using Eq.(\ref{E.5 (general value)}) :

\begin{equation} \label{C.7(gerneal value)}
    \begin{aligned}
        &\mathbb{P}\Big(||\bar{f}_V||_{\mathcal{D}_V^k}^2- ||f||_{\mathcal{D}_V^k}^2-\frac{2}{3} H^2\log(\frac{2}{\delta})- ||f-\bar{f}_V ||_{\mathcal{Z}^k}\cdot  \sqrt{kH\zeta^2 + C'H\cdot \log(\frac{8KH}{\delta})} \\ &+ \frac{1}{4} ||f-\bar{f}_V ||_{\mathcal{Z}^k}^2 \leq 0, \forall k \in [K] \Big) \geq 1-\delta
    \end{aligned}
\end{equation}

\end{proof}
\end{lemma}

\begin{lemma} \label{Discretization error (general value function)}

(Discretization error)   
If $g \in \mathcal{C}(\mathcal{F},1/T)$ satisfies $||f-g ||_{\infty} \leq 1/T$, then 
\begin{equation} \label{D-error-equation(general value)}
    \begin{aligned}
         &\Big| \frac{1}{4} ||g-\bar{f}_V ||_{\mathcal{Z}^k}^2 -\frac{1}{4} ||f-\bar{f}_V||_{\mathcal{Z}^k}^2 +||f-\bar{f}_V||_{\mathcal{Z}^k} \cdot  \sqrt{kH\zeta^2 + C'H\cdot \log(\frac{8KH\mathcal{N}(\mathcal{F},1/T)}{\delta})} \\ &- ||g-\bar{f}_V||_{\mathcal{Z}^k} \cdot  \sqrt{kH\zeta^2 + C'H\cdot \log(\frac{8KH\mathcal{N}(\mathcal{F},1/T)}{\delta})} + ||f||_{\mathcal{D}_V^k}^2 - ||g||_{\mathcal{D}_V^k}^2 \Big| \\ &\leq 5(H+1)+ 2\sqrt{kH^2\zeta^2 + C'H^2\cdot \log(\frac{8KH\mathcal{N}(\mathcal{F},1/T)}{\delta})}
    \end{aligned}
\end{equation}

\begin{proof}

For any $(s,a) \in \mathcal{S}\times \mathcal{A}$,
\begin{equation}
    \begin{aligned}
        &\ \ \ \ \left|\left(g(s,a)-\bar{f}_V(s,a)\right)^2-\left(f(s,a)-\bar{f}_{V}(s,a)\right)^2\right| \\ &\leq \left|\left[g(s,a)+f(s,a)-2\bar{f}_V(s,a)\right]\cdot\left[g(s,a)-f(s,a)\right]\right| \\ & \leq 4H \cdot \frac{1}{T}
    \end{aligned}
\end{equation}

Therefore,
\begin{equation} \label{Discretization error (general value function)-1}
    \begin{aligned}
        \left| ||f-\bar{f}_V||_{\mathcal{Z}^k}^2 - ||g-\bar{f} ||_{\mathcal{Z}^k}^2\right| \leq 4H\cdot \frac{1}{T} \cdot |\mathcal{Z}^k| = 4H
    \end{aligned}
\end{equation}

\begin{equation}\label{Discretization error (general value function)-2}
    \begin{aligned}
        \left|||f||_{\mathcal{D}_V^k}^2 - ||g||_{\mathcal{D}_V^k}^2\right| &= \left|\sum\limits_{k'=1}^{k-1}\sum\limits_{h=1}^H \left(f(s_h^{k'},a_h^{k'})-r(s_h^{k'},a_h^{k'}) -V(s_{h+1}^{k'})\right)^2 - \sum\limits_{k'=1}^{k-1}\sum\limits_{h=1}^H \left(g(s_h^{k'},a_h^{k'})-r(s_h^{k'},a_h^{k'}) -V(s_{h+1}^{k'})\right)^2 \right|\\ &\leq \left|\sum\limits_{k'=1}^{k-1}\sum\limits_{h=1}^H \left(f(s_h^{k'},a_h^{k'}) + g(s_h^{k'},a_h^{k'}) -2r(s_h^{k'},a_h^{k'}) -2V(s_{h+1}^{k'})\right) \left(f(s_h^{k'},a_h^{k'})-g(s_h^{k'},a_h^{k'})\right)\right| \\& \leq 4(H+1) \cdot |\mathcal{Z}^k|\cdot \frac{1}{T} = 4(H+1)
    \end{aligned}
\end{equation}

Moreover, 

\begin{equation}\label{Discretization error (general value function)-3}
    \begin{aligned}
        \left| ||f-\bar{f}_V||_{\mathcal{Z}^k} - ||g-\bar{f}_V ||_{\mathcal{Z}^k}\right| \leq \sqrt{\left| ||f-\bar{f}_V||_{\mathcal{Z}^k}^2 - ||g-\bar{f} ||_{\mathcal{Z}^k}^2\right|} \leq 2\sqrt{H}
    \end{aligned}
\end{equation}

By combining (\ref{Discretization error (general value function)-1}), (\ref{Discretization error (general value function)-2}) and (\ref{Discretization error (general value function)-3}), we have (\ref{D-error-equation(general value)}).

\end{proof}

\end{lemma}

\begin{lemma} \label{fix V all f}
For a fixed $V: \mathcal{S} \rightarrow  [0,H]$, with probability at least $1-\delta$, for all $k \in [K]$ and all $f \in \mathcal{F}$, that

\begin{equation} \label{C.12 all f in F}
    \begin{aligned}
        ||f||_{\mathcal{D}_V^k}^2- ||\bar{f}_V||_{\mathcal{D}_V^k}^2 &\geq \frac{1}{4}||f-\bar{f}_V||_{\mathcal{Z}^k}^2 - ||f-\bar{f}_V||_{\mathcal{Z}^k} \cdot  \sqrt{kH\zeta^2 + C'H\cdot \log(\frac{8KH\mathcal{N}(\mathcal{F},1/T)}{\delta})} \\ &-\frac{2}{3}H^2\log(\frac{2\mathcal{N}(\mathcal{F},1/T)}{\delta})-5(H+1)- 2\sqrt{kH^2\zeta^2 + C'H^2\cdot \log(\frac{8KH\mathcal{N}(\mathcal{F},1/T)}{\delta})}
    \end{aligned}
\end{equation}

We define the above event to be $\varepsilon_{V,\delta}$.

\begin{proof}
        For any $f \in \mathcal{F}$, there exists a $g \in \mathcal{C}(\mathcal{F},1/T)$, such that $||f-g||_{\infty} \leq \frac{1}{T}$. By taking a union bound on all $g \in \mathcal{C}(\mathcal{F},1/T)$ and using the result from Lemma \ref{single-f-error(general value)}, we have, with probability at least $1-\delta$, for all $g \in \mathcal{C}(\mathcal{F},1/T)$, any $k \in [K]$, that
\begin{equation}
    \begin{aligned}
       ||\bar{f}_V||_{\mathcal{D}_V^k}^2- ||g||_{\mathcal{D}_V^k}^2-\frac{2}{3} H^2\log(\frac{2\mathcal{N}(\mathcal{F},1/T)}{\delta})- ||g-\bar{f}_V ||_{\mathcal{Z}^k}\cdot  \sqrt{kH\zeta^2 + C'H\cdot \log(\frac{8KH\mathcal{N}(\mathcal{F},1/T)}{\delta})} + \frac{1}{4} ||g-\bar{f}_V ||_{\mathcal{Z}^k}^2 \leq 0
    \end{aligned}
\end{equation}
Therefore, with probability at least $1-\delta$, for all $k \in [K]$, and all $f \in \mathcal{F}$, we have:
\begin{equation} \label{E.13}
    \begin{aligned}
         ||f||_{\mathcal{D}_V^k}^2- ||\bar{f}_V||_{\mathcal{D}_V^k}^2&\geq \frac{1}{4}||f-\bar{f}_V||_{\mathcal{Z}^k}^2 - \frac{2}{3}H^2\log(\frac{2\mathcal{N}(\mathcal{F},1/T)}{\delta})- ||f-\bar{f}_V||_{\mathcal{Z}^k} \cdot  \sqrt{kH\zeta^2 + C'H\cdot \log(\frac{8KH\mathcal{N}(\mathcal{F},1/T)}{\delta})} \\ &+ \Big\{\frac{1}{4} ||g-\bar{f}_V ||_{\mathcal{Z}^k}^2 -\frac{1}{4} ||f-\bar{f}_V||_{\mathcal{Z}^k}^2 \\ &+||f-\bar{f}_V||_{\mathcal{Z}^k} \cdot  \sqrt{kH\zeta^2 + C'H\cdot \log(\frac{8KH\mathcal{N}(\mathcal{F},1/T)}{\delta})} \\ &- ||g-\bar{f}_V||_{\mathcal{Z}^k} \cdot  \sqrt{kH\zeta^2 + C'H\cdot \log(\frac{8KH\mathcal{N}(\mathcal{F},1/T)}{\delta})} \\ &+ ||f||_{\mathcal{D}_V^k}^2 - ||g||_{\mathcal{D}_V^k}^2\Big\} \\ &\geq \frac{1}{4}||f-\bar{f}_V||_{\mathcal{Z}^k}^2 - \frac{2}{3}H^2\log(\frac{2\mathcal{N}(\mathcal{F},1/T)}{\delta})- ||f-\bar{f}_V||_{\mathcal{Z}^k} \cdot  \sqrt{kH\zeta^2 + C'H\cdot \log(\frac{8KH\mathcal{N}(\mathcal{F},1/T)}{\delta})} \\ &-5(H+1)- 2\sqrt{kH^2\zeta^2 + C'H^2\cdot \log(\frac{8KH\mathcal{N}(\mathcal{F},1/T)}{\delta})} \ \ \ \ (\text{By} \ \ (\ref{D-error-equation(general value)}))
    \end{aligned}
\end{equation}
\end{proof}
\end{lemma}

\begin{lemma} \label{expansion to neighborhood}
Conditioned on the event $\mathcal{\varepsilon}_{V,\delta}$ (defined in Lemma \ref{fix V all f}), then for any $V': \mathcal{S} \rightarrow [0,H]$ with $||V'-V ||_{\infty} \leq \frac{1}{T}$, we have for any $k \in [K]$, 
    \begin{equation}
        \begin{aligned}
            ||\widehat{f}_{V'} - \bar{f}_{V} ||_{\mathcal{Z}^k} \leq C'\cdot \sqrt{kH\zeta^2 + H^2\log(\frac{KH\mathcal{N}(\mathcal{F},1/T)}{\delta})}
        \end{aligned}
    \end{equation}

\begin{proof}

For a fixed $V: \mathcal{S} \rightarrow [0,H]$, we consider any $V' : \mathcal{S} \rightarrow [0,H]$ with $||V'-V||_{\infty} \leq \frac{1}{T}$.

Then for any $f \in \mathcal{F}$,
\begin{equation} \label{C.10 (general function)}
    \begin{aligned}
        ||f||_{\mathcal{D}_{V'}^k}^2 - ||\bar{f}_{V}||_{\mathcal{D}_{V'}^k}^2 &= ||f-\bar{f}_{V}||_{\mathcal{Z}^k}^2 + 2\sum\limits_{k'=1}^{k-1}\sum\limits_{h=1}^H \left(f(s_h^{k'},a_h^{k'}) - \bar{f}_{V}(s_h^{k'},a_h^{k'})\right)\cdot \left(\bar{f}_{V}(s_h^{k'},a_h^{k'}) - r_{h}^{k'} -V'(s_{h+1}^{k'})\right)  \\ &= ||f-\bar{f}_{V}||_{\mathcal{Z}^k}^2+ \underbrace{2\sum\limits_{k'=1}^{k-1}\sum\limits_{h=1}^H \left(f(s_h^{k'},a_h^{k'}) - \bar{f}_{V}(s_h^{k'},a_h^{k'})\right)\cdot \left(\bar{f}_{V}(s_h^{k'},a_h^{k'}) - r_{h}^{k'} -V(s_{h+1}^{k'})\right)}_{\text{term A}} \\ &- \underbrace{2\sum\limits_{k'=1}^{k-1}\sum\limits_{h=1}^H \left(f(s_h^{k'},a_h^{k'}) - \bar{f}_{V}(s_h^{k'},a_h^{k'})\right)\left(V'(s_{h+1}^{k'})-V(s_{h+1}^{k'})\right)}_{\text{term B}}
    \end{aligned}
\end{equation}

For the term A, by using (\ref{C.12 all f in F}), we have
\begin{equation}
    \begin{aligned}
        \text{term A} &= ||f||_{\mathcal{D}_V^k}^2 - ||\bar{f}_V||_{\mathcal{D}_V^k}^2 - ||f-\bar{f}_V||_{\mathcal{Z}^k}^2 \\ &\geq \frac{1}{4}||f-\bar{f}_V||_{\mathcal{Z}^k}^2 - ||f-\bar{f}_V||_{\mathcal{Z}^k} \cdot  \sqrt{kH\zeta^2 + C'H\cdot \log(\frac{8KH\mathcal{N}(\mathcal{F},1/T)}{\delta})} \\ &-\frac{2}{3}H^2\log(\frac{2\mathcal{N}(\mathcal{F},1/T)}{\delta})-5(H+1)- 2\sqrt{kH^2\zeta^2 + C'H^2\cdot \log(\frac{8KH\mathcal{N}(\mathcal{F},1/T)}{\delta})}-||f-\bar{f}_V||_{\mathcal{Z}^k}^2  
    \end{aligned}
\end{equation}

By using Cauchy-Schwarz's inequality, we can derive the upper bound for term B.
\begin{equation}
    \begin{aligned}
        \text{term B} &\leq 2 ||f-\bar{f}_V ||_{\mathcal{Z}^k} \cdot \sqrt{\sum\limits_{k'=1}^{k-1}\sum\limits_{h=1}^H \left(V'(s_{h+1}^{k'})-V(s_{h+1}^{k'})\right)^2} \leq  2 ||f-\bar{f}_V ||_{\mathcal{Z}^k} 
    \end{aligned}
\end{equation}

Therefore, Eq.(\ref{C.10 (general function)}) can be bounded by
\begin{equation}
    \begin{aligned}
        ||f||_{\mathcal{D}_{V'}^k}^2 - ||\bar{f}_{V}||_{\mathcal{D}_{V'}^k}^2 &\geq ||f-\bar{f}_{V}||_{\mathcal{Z}^k}^2 +\frac{1}{4}||f-\bar{f}_V||_{\mathcal{Z}^k}^2 - ||f-\bar{f}_V||_{\mathcal{Z}^k} \cdot  \sqrt{kH\zeta^2 + C'H\cdot \log(\frac{8KH\mathcal{N}(\mathcal{F},1/T)}{\delta})} \\ &-\frac{2}{3}H^2\log(\frac{2\mathcal{N}(\mathcal{F},1/T)}{\delta})-5(H+1)- 2\sqrt{kH^2\zeta^2 + C'H^2\cdot \log(\frac{8KH\mathcal{N}(\mathcal{F},1/T)}{\delta})}-||f-\bar{f}_V||_{\mathcal{Z}^k}^2 \\ & -2||f-\bar{f}_V ||_{\mathcal{Z}^k} 
    \end{aligned}
\end{equation}

Notice that 
$$
\widehat{f}_{V'}: = \text{arg}\min\limits_{f\in \mathcal{F}} || f||_{\mathcal{D}_{V'}^k}^2
$$
, and by solving the quadratic equation for $||\widehat{f}_{V'} - \bar{f}_{V} ||_{\mathcal{Z}^k}$, we have
\begin{equation}
    \begin{aligned}
        ||\widehat{f}_{V'} - \bar{f}_{V} ||_{\mathcal{Z}^k} \leq C'\cdot \sqrt{kH\zeta^2 + H^2\log(\frac{KH\mathcal{N}(\mathcal{F},1/T)}{\delta})}
    \end{aligned}
\end{equation}
where $C'$ is an absolute constant.
    
\end{proof}
\end{lemma}

\begin{lemma} \label{important confidence set}
 Let $\mathcal{F}_h^k$ be the confidence region defined as 
$$
\mathcal{F}_h^k = \left\{f \in \mathcal{F}: ||f-f_h^k ||_{\mathcal{Z}_h^k} \leq \beta(\mathcal{F},\delta)\right\}
$$
where 
$$
\beta(\mathcal{F},\delta) = C'\cdot\sqrt{kH\zeta^2 + H^2\left(\log(\frac{4T^2}{\delta}) + 2\log \mathcal{N}(\mathcal{F},1/T)+ \log|\mathcal{W}|+1\right)}
$$
Then with probability at least $1-\frac{\delta}{4}$, we have for all $(k,h)\in [K]\times [H]$,
\begin{equation} \label{key event 1}
    \begin{aligned}
        \bar{f}_{(V_{h+1}^k)^{\dagger}} \in \mathcal{F}_h^k
    \end{aligned}
\end{equation}
where $(V)^{\dagger}$ denotes the closest function to $V$ in the set $\mathcal{V}$ (the $1/T$-net of $\{V_h^k\}$).

\begin{proof}
    We denote
\begin{equation}
    \begin{aligned}
        \mathcal{Q} : = \left\{\min\{f(\cdot,\cdot) + \omega(\cdot,\cdot),H\}| \ \omega \in \mathcal{W}, f \in \mathcal{C}(\mathcal{F},1/T)\cup \{0\}\right\}
    \end{aligned}
\end{equation}
Notice that $\mathcal{Q}$ is a $(1/T)$-cover of $Q_{h+1}^k(\cdot,\cdot)$. This implies that 
\begin{equation}
    \begin{aligned}
        \mathcal{V}: = \left\{\max_{a\in \mathcal{A}}q(\cdot,a)| q \in \mathcal{Q}\right\}
    \end{aligned}
\end{equation}
is also a $(1/T)$-cover of $V_{h+1}^k$, and we have $\log(|\mathcal{V}|) \leq \log|\mathcal{W}| + \log\mathcal{N}(\mathcal{F},1/T) +1$.

By taking the union bound for all $V \in \mathcal{V}$ in the event defined in Lemma \ref{fix V all f},  we have $Pr(\bigcap_{V\in\mathcal{V}} \varepsilon_{V,\delta/(4|\mathcal{V}|T)}) \geq 1-\delta/(4T)$. We condition on $\bigcap_{V\in\mathcal{V}} \varepsilon_{V,\delta/(4|\mathcal{V}|T)})$ in the rest part of the proof.

Recall that $f_h^k$ is the minimizer of the empirical loss, i.e.,  $f_h^k = \argmin_{f \in \mathcal{F}}||f||_{\mathcal{D}_h^k}^2$. Let $(V_{h+1}^k)^{\dagger} \in \mathcal{V}$ such that $||V_{h+1}^k -(V_{h+1}^k)^{\dagger}||_{\infty} \leq 1/T$. Then, by lemma \ref{expansion to neighborhood}, we have
\begin{equation}
    \begin{aligned}
        ||f_h^k - \bar{f}_{(V_{h+1}^k)^{\dagger}} ||_{\mathcal{Z}^k} &\leq C'\cdot \sqrt{kH\zeta^2 + H^2\log(\frac{4\mathcal{N}(\mathcal{F},1/T)|\mathcal{V}|T^2}{\delta})} \\ &\leq C'\cdot\sqrt{kH\zeta^2 + H^2\left(\log(\frac{4T^2}{\delta}) + 2\log \mathcal{N}(\mathcal{F},1/T)+ \log|\mathcal{W}|+1\right)}
    \end{aligned}
\end{equation}

This completes the proof.

\end{proof}
\end{lemma}

\begin{lemma} \label{width lemma} (Proposition 2 in \citep{wang2020reinforcement})
With probability at least $1-\delta/8$, for all $(s,a) \in \mathcal{S} \times \mathcal{A}$,

$$
\omega(\mathcal{F}_h^k,s,a) \leq b_h^k(s,a)
$$
\end{lemma}

\begin{lemma} \label{bonus bound lemma} (Lemma 10 in \citep{wang2020reinforcement})
With probability at least $1-\delta/8$,
\begin{equation} \label{bonus bound formula}
    \begin{aligned}
        \sum\limits_{k=1}^K\sum\limits_{h=1}^H b_h^k(s_h^k,a_h^k) \leq 1+4H^2 \text{dim}_E(\mathcal{F},1/T) + \sqrt{c\cdot \text{dim}_E(\mathcal{F},1/T)\cdot T}\cdot \beta(\mathcal{F},\delta)
    \end{aligned}
\end{equation}

\end{lemma}

\begin{theorem}
(Regret bound of robust value-based methods)

Under our Assumption \ref{ass:cover} and \ref{Ass:general-value}, for any fixed $\delta \in (0,1)$, with probability at least $1-\delta$,  the total regret of Algorithm \ref{Algorithm general known} is at most $\widetilde{O}\left(\sqrt{d_EH^3}K\zeta \log(1/\delta)+ \sqrt{d_E^2KH^3} \log(1/\delta) \right)$, where $d_E$ represents the eluder dimension of the function class.
    
\begin{proof}
    First of all, we do the following decomposition.
\begin{equation}
    \begin{aligned}
        \text{Regret}(K) = \sum\limits_{k=1}^K[V_1^*(s_1^k)-V_1^{\pi_k}(s_1^k)] = \underbrace{\sum\limits_{k=1}^K[V_1^*(s_1^k)-V_1^k(s_1^k)]}_{A}+\underbrace{\sum\limits_{k=1}^K[V_1^k(s_1^k)-V_1^{\pi_k}(s_1^k)]}_{B}    
    \end{aligned}
\end{equation}

We denote $\mathcal{E}$ to be the event that (\ref{key event 1}) holds, and $\mathcal{E}'$ to be the event that for all $(k,h) \in [K] \times [H]$, all $(s,a) \in \mathcal{S}\times \mathcal{A}$, $b_h^k(s,a) \geq \omega(\mathcal{F}_h^k,s,a)$. From Lemma \ref{important confidence set} and Lemma \ref{width lemma}, we have $Pr(\mathcal{E}\cap\mathcal{E}') \geq 1-\frac{\delta}{2}$. For the rest of the proof, we condition on the above event. 

Note that 
$$
\max_{f\in \mathcal{F}_h^k} |f(s,a) - f_h^k(s,a)| \leq \omega(\mathcal{F}_h^k,s,a) \leq b_h^k(s,a)
$$

Since $\bar{f}_{(V_{h+1}^k)^{\dagger}} \in \mathcal{F}_h^k$ for all $(k,h) \in [K] \times [H]$, we have for all  $(s,a) \in \mathcal{S} \times \mathcal{A}$, all $(k,h) \in [K] \times [H]$,

\begin{equation} \label{c26}
    \begin{aligned}
        |\bar{f}_{(V_{h+1}^k)^{\dagger}}(s,a) - f_h^k(s,a)| \leq \omega(\mathcal{F}_h^k,s,a) \leq b_h^k(s,a)
    \end{aligned}
\end{equation}

To simplify our notation, for each $(k,h) \in [K] \times [H]$, we denote 
$$
\xi_{V_{h+1}^k}^{\dagger}(s,a) := \bar{f}_{(V_{h+1}^k)^{\dagger}}(s,a) - r(s,a)-\mathbb{P}V_{h+1}^k(s,a)
$$, 

$$
\zeta_{h+1}^k = \mathbb{P}(V_{h+1}^k - V_{h+1}^{\pi_k})(s_h^k,a_h^k) - \left(V_{h+1}^k(s_{h+1}^k) - V_{h+1}^{\pi_k}(s_{h+1}^k)\right) 
$$
and
$$
\zeta_{h+1}^{k*} = \mathbb{P}(V_{h+1}^k - V_{h+1}^{\pi_k})(s_h^{k*},a_h^{k*}) - \left(V_{h+1}^k(s_{h+1}^k) - V_{h+1}^{\pi_k}(s_{h+1}^{k*})\right) 
$$

For the term $\sum\limits_{k=1}^K\sum\limits_{h=1}^H  \zeta_{h+1}^{k}$ and $\sum\limits_{k=1}^K\sum\limits_{h=1}^H  \zeta_{h+1}^{k*}$, notice that $\{\zeta_{h+1}^{k*}\}$ and $\{\zeta_{h+1}^{k}\}$ are martingale difference sequences, and each term is upper bounded by $2H$. By using Azuma-Hoeffding's inequality, with probability at least $1-\delta/4$, the following inequality holds:

\begin{equation} \label{last martingale 1}
    \begin{aligned}
        \sum\limits_{k=1}^K\sum\limits_{h=1}^H \zeta_{h+1}^{k} \leq \sqrt{8KH^3\cdot \log(8/\delta)}
    \end{aligned}
\end{equation}

\begin{equation} \label{last martingale 2}
    \begin{aligned}
        \sum\limits_{k=1}^K\sum\limits_{h=1}^H \zeta_{h+1}^{k*} \leq \sqrt{8KH^3\cdot \log(8/\delta)}
    \end{aligned}
\end{equation}

Notice that for all $(s,a) \in \mathcal{S} \times \mathcal{A} $,
\begin{equation}
    \begin{aligned}
        |\xi_{V_{h+1}^k}^{\dagger}(s,a)| &= \left|\bar{f}_{(V_{h+1}^k)^{\dagger}}(s,a) - r(s,a)-\mathbb{P}(V_{h+1}^k)^{\dagger}(s,a) +  \mathbb{P}(V_{h+1}^k)^{\dagger}(s,a) - \mathbb{P}V_{h+1}^k(s,a) \right| \\ &\leq \left|\xi_{(V_{h+1}^k)^{\dagger}}(s,a)\right| + 1/T
    \end{aligned}
\end{equation}

Therefore,
\begin{equation} \label{C28}
    \begin{aligned}
        \left|\sum\limits_{k=1}^K\sum\limits_{h=1}^H  \xi_{V_{h+1}^k}^{\dagger}(s_h^k,a_h^k) \right| \leq \sum\limits_{k=1}^K\sum\limits_{h=1}^H  \left|\xi_{(V_{h+1}^k)^{\dagger}}(s_h^k,a_h^k)\right| +1
    \end{aligned}
\end{equation}

By using Assumption \ref{Ass:general-value} and Azuma-Hoeffding's inequality, we have with probability at least $1-\delta/4$,

\begin{equation}
    \begin{aligned} \label{C29}
        \sum\limits_{k=1}^K\sum\limits_{h=1}^H  \left|\xi_{(V_{h+1}^k)^{\dagger}}(s_h^k,a_h^k)\right| \leq \sqrt{8KH^2 \log(\frac{16}{\delta})} + KH\zeta 
    \end{aligned}
\end{equation}

\begin{equation}
    \begin{aligned} \label{C30}
        \sum\limits_{k=1}^K\sum\limits_{h=1}^H  \left|\xi_{(V_{h+1}^k)^{\dagger}}(s_h^{k*},a_h^{k*})\right| \leq \sqrt{8KH^2 \log(\frac{16}{\delta})} + KH\zeta 
    \end{aligned}
\end{equation}

For the term $A$, we only need to consider when $V_h^k = f_h^k + b_h^k$, for all $(k,h) \in [K] \times [H]$.

\begin{equation}
    \begin{aligned}
        A &= \sum\limits_{k=1}^K \left(r(s_1^{k*},a_1^{k*}) + \mathbb{P}V_2^*(s_1^{k*},a_1^{k*}) - f_1^k(s_1^{k*},a_1^{k*}) - b_1^k(s_1^{k*},a_1^{k*})\right) \\ & = \sum\limits_{k=1}^K \Big(r(s_1^{k*},a_1^{k*}) + \mathbb{P}V_2^k(s_1^{k*},a_1^{k*}) - \bar{f}_{(V_2^k)^{\dagger}}(s_1^{k*},a_1^{k*}) \\ &+ \bar{f}_{(V_2^k)^{\dagger}}(s_1^{k*},a_1^{k*}) - f_1^k(s_1^{k*},a_1^{k*}) -b_1^k(s_1^{k*},a_1^{k*}) + \mathbb{P}(V_2^* - V_2^k)(s_1^{k*},a_1^{k*})\Big) \\ &\leq \sum\limits_{k=1}^K \left(-\xi_{V_2^k}^{\dagger}(s_1^{k*},a_1^{k*}) + \mathbb{P}(V_2^* - V_2^k)(s_1^{k*},a_1^{k*})\right) \\ &\leq \sum\limits_{k=1}^K \left(-\xi_{V_2^k}^{\dagger}(s_1^{k*},a_1^{k*}) + V_2^*(s_2^{k*}) - V_2^k(s_2^{k*}) + \zeta_2^{k*}\right) \leq \cdots \\ &\leq \sum\limits_{k=1}^K\sum\limits_{h=1}^H  |\xi_{V_{h+1}^k}^{\dagger}(s_h^{k*},a_h^{k*})|+ \sum\limits_{k=1}^K\sum\limits_{h=1}^H \zeta_{h+1}^{k*}
    \end{aligned}
\end{equation}

For the term B, 
\begin{equation}
    \begin{aligned}
        B &= \sum\limits_{k=1}^K \left(Q_1^k(s_1^k,a_1^k) - Q_1^{\pi_k}(s_1^k,a_1^k)\right) \\ &\leq \sum\limits_{k=1}^K\left(f_1^k(s_1^k,a_1^k) + b_1^k(s_1^k,a_1^k) - r(s_1^k,a_1^k) -\mathbb{P}V_2^{\pi_k}(s_1^k,a_1^k)\right) \\ &\leq \sum\limits_{k=1}^K \left(\bar{f}_{(V_{2}^k)^{\dagger}}(s_1^k,a_1^k) + 2b_1^1(s_1^k,a_1^k)- r(s_1^k,a_1^k) -\mathbb{P}V_2^{\pi_k}(s_1^k,a_1^k)\right) \ \ \ (\text{By Eq}.(\ref{c26})) \\ &= \sum\limits_{k=1}^K \left(\bar{f}_{(V_{2}^k)^{\dagger}}(s_1^k,a_1^k) - r(s_1^k,a_1^k)-\mathbb{P}V_2^k(s_1^k,a_1^k) + \mathbb{P}V_2^k(s_1^k,a_1^k) - \mathbb{P}V_2^{\pi_k}(s_1^k,a_1^k) + 2b_1^1(s_1^k,a_1^k)\right) \\ & = \sum\limits_{k=1}^K \left(\xi_{V_2^k}^{\dagger}(s_1^k,a_1^k) + 2b_1^1(s_1^k,a_1^k) + V_2^k(s_2^k)-V_2^{\pi_k}(s_2^k) + \zeta_2^k\right) \leq \cdots \\ & \leq \sum\limits_{k=1}^K\sum\limits_{h=1}^H  |\xi_{V_{h+1}^k}^{\dagger}(s_h^k,a_h^k)| + 2\sum\limits_{k=1}^K\sum\limits_{h=1}^H b_h^k(s_h^k,a_h^k) + \sum\limits_{k=1}^K\sum\limits_{h=1}^H \zeta_{h+1}^k
    \end{aligned}
\end{equation}

By using (\ref{bonus bound formula}), (\ref{last martingale 1}), (\ref{last martingale 2}), (\ref{C28}), (\ref{C29}), and (\ref{C30}), we complete our proof.

\end{proof}

\end{theorem}

\section{Analysis of Robust Model-based Algorithm with General Function Approximation for Locally-bounded Misspecified MDP} \label{general_model_based}

In this section, we will provide the theoretical analysis for Algorithm \ref{Algorithm model-based} under locally-bounded misspecification error assumption (Assumption \ref{Ass4}).

\textbf{Confidence sets for non-linear regression with locally-bounded misspecification error}

Let $\mathcal{V}$ be the set of optimal value functions under some model in $\mathcal{P}$: $\mathcal{V} = \{V_{P'}^* : P' \in \mathcal{P} \}$.
We define $\mathcal{X} = \mathcal{S}\times \mathcal{A} \times \mathcal{V}$, and choose 
\begin{equation} \label{definition of function class}
    \begin{aligned}
        \mathcal{F} = \left\{ f: \mathcal{X} \rightarrow  \mathbb{R}: \exists \widetilde{P} \in \mathcal{P}\ \  \text{s.t.} \ \ f(s,a,V) = \widetilde{\mathbb{P}}V(s,a), \ \ \forall (s,a,V) \in \mathcal{X}\right\}
    \end{aligned}
\end{equation}

Let $\phi : \mathcal{P} \rightarrow \mathcal{F}$ be the natural surjection to $\mathcal{F}$: $\phi(P) = f$, such that $f (s,a,V) = \mathbb{P}V(s,a), \ \ \forall (s,a,V) \in \mathcal{X}$. In fact, $\phi$ is a bijection, and for convenience to the reader, we denote $f_P = \phi(P)$.

\paragraph{} For any $f \in \mathcal{F}$, we define the empirical loss as 
$$
L_{2,k}(f) = \sum\limits_{k'=1}^k\sum\limits_{h=1}^H \left(f(s_h^{k'},a_h^{k'},V_{h+1}^{k'})-V_{h+1}^{k'}(s_{h+1}^{k'})\right)^2
$$
and the minimizer $\widehat{f}_{k+1} = \argmin_{f \in \mathcal{F}}L_{2,k}(f)$. Since $\phi$ is a bijection, $\widehat{f}_{k+1} = \phi(\widehat{P}^{(k+1)}) = f_{\widehat{P}^{(k+1)}}$, where $\widehat{P}^{(k+1)}$ is defined in (\ref{v-t-1}).

We also define the norm 

$$
||f||_{D_h^k} = \sqrt{\sum\limits_{k'=1}^k\sum\limits_{h'=1}^h\left(f(s_{h'}^{k'},a_{h'}^{k'},V_{h'+1}^{k'})\right)^2}
$$

Now we are able to define the confidence set for each episode, which is also introduced in (\ref{v-t-2}).
$$
B_k = \{\widetilde{P} \in \mathcal{P}: L_k(\widetilde{P}, \widehat{P}^{(k)}) \leq \beta_k^2\} = \{\phi^{-1}(f): f\in \mathcal{F} \ \ \text{and}\ \ ||f-\widehat{f}_k ||_{D_H^k} \leq \beta_k\}
$$

We set $\mathbb{F}_h^k$ to be the $\sigma$-algebra generated by $\{(s_{h'}^{k'},a_{h'}^{k'})\}_{(h',k')\in [H]\times [k-1]} \cup \{(s_1^k,a_1^k),(s_2^k,a_2^k),\cdots, (s_h^k,a_h^k)\}$. For each $(h,k) \in [H] \times [K]$, we define $\mathcal{Z}_h^k = \left( \bar{f}(s_h^k,a_h^k,V_{h+1}^k) - V_{h+1}^k (s_{h+1}^k)\right)^2 - \left(f(s_h^k,a_h^k,V_{h+1}^k) - V_{h+1}^k(s_{h+1}^k)\right)^2$

\begin{lemma} \label{single-f-error}
Under Assumption \ref{Ass4}, for any fixed $f \in \mathcal{F}$, with probability at least $1-\delta$,  we have for all  $k \in [K]$ that 
\begin{equation}
    \begin{aligned}
         L_{2,k}(\bar{f})-L_{2,k}(f)-H^2\log(\frac{1}{\delta})- ||f-\bar{f}||_{D_H^k} \cdot  \sqrt{kH\zeta^2 + C'H\cdot \log(\frac{4KH}{\delta})}  + \frac{1}{2}||f-\bar{f}||_{D_H^k}^2 \leq 0
    \end{aligned}
\end{equation}

\begin{proof}
According to the definition of $\mathcal{Z}_h^k$, $\sum\limits_{k'=1}^k\sum\limits_{h=1}^H \mathcal{Z}_h^{k'} = L_{2,k}(\bar{f}) - L_{2,k}(f)$. After a simple calculation, we have 
$$
\mathcal{Z}_h^k = -\left(f(s_h^k,a_h^k,V_{h+1}^k) - \bar{f}(s_h^k,a_h^k,V_{h+1}^k)\right)^2 + 2\left(f(s_h^k,a_h^k,V_{h+1}^k) - \bar{f}(s_h^k,a_h^k,V_{h+1}^k)\right)\left(\underbrace{V_{h+1}^k(s_{h+1}^k)-\bar{f}(s_h^k,a_h^k,V_{h+1}^k)}_{\epsilon_h^k + \xi_h^k}\right)
$$
where $\epsilon_h^k = V_{h+1}^k(s_{h+1}^k) - \mathbb{P}_hV_{h+1}^k(s_h^k,a_h^k)$, $\xi_h^k = \mathbb{P}_hV_{h+1}^k(s_h^k,a_h^k) - \bar{f}(s_h^k,a_h^k,V_{h+1}^k)$.

Notice that $\mathbb{E}[\epsilon_h^k| \mathbb{F}_h^k] = 0$, and since $\epsilon_h^k$ is bounded in $[0,H]$, hence, $\epsilon_h^k$ is $\frac{H}{2}$-subgaussian.  That is to say, for any $\lambda \in \mathbb{R}$, we have $\mathbb{E}[\exp\{\lambda \epsilon_h^k\}| \mathbb{F}_h^k] \leq \exp\{\frac{\lambda^2H^2}{8}\}$.

Moreover, under Assumption \ref{Ass4}, using the same argument in Lemma \ref{A.7}, with probability at least $1-\delta$, for all $(k,h) \in [K]\times [H]$, we have

\begin{equation} \label{union bound for miss error}
    \begin{aligned}
        \sum\limits_{k'=1}^k\sum\limits_{h=1}^H (\xi_h^{k'})^2 \leq kH\zeta^2 + C'H\cdot \log(\frac{4KH}{\delta})
    \end{aligned}
\end{equation}

where $C'$ is some constant.

Therefore,

$$
\mu_h^k = \mathbb{E}[\mathcal{Z}_h^k| \mathbb{F}_h^k] = -\left(f(s_h^k,a_h^k,V_{h+1}^k) - \bar{f}(s_h^k,a_h^k,V_{h+1}^k)\right)^2 + 2\left(f(s_h^k,a_h^k,V_{h+1}^k) - \bar{f}(s_h^k,a_h^k,V_{h+1}^k)\right)\xi_h^k
$$
and
\begin{equation}
    \begin{aligned}
        \phi_h^k(\lambda) &= \log \mathbb{E}\left[\exp\{\lambda(\mathcal{Z}_h^k -\mu_h^k)\}| \mathbb{F}_h^k\right] \\ &= \log \mathbb{E}\left[\exp \{2\lambda \left(f(s_h^k,a_h^k,V_{h+1}^k) - \bar{f}(s_h^k,a_h^k,V_{h+1}^k)\right)\epsilon_h^k\}|\mathbb{F}_h^k\right] \\ &\leq  \frac{\lambda^2 \left(f(s_h^k,a_h^k,V_{h+1}^k) - \bar{f}(s_h^k,a_h^k,V_{h+1}^k)\right)^2 H^2}{2}
\end{aligned}
\end{equation}

By using Lemma \ref{russo}, we have for any $x\geq 0$, $\lambda \geq 0$,
\begin{equation} \label{E.5}
    \begin{aligned}
    \mathbb{P}\Big(\lambda \sum\limits_{k'=1}^k\sum\limits_{h=1}^H \mathcal{Z}_h^{k'} &\leq x -\lambda \sum\limits_{k'=1}^k\sum\limits_{h=1}^H \left(f(s_h^{k'},a_h^{k'},V_{h+1}^{k'}) - \bar{f}(s_h^{k'},a_h^{k'},V_{h+1}^{k'})\right)^2  \\ & +2\lambda \sum\limits_{k'=1}^k\sum\limits_{h=1}^H \left(f(s_h^{k'},a_h^{k'},V_{h+1}^{k'}) - \bar{f}(s_h^{k'},a_h^{k'},V_{h+1}^{k'})\right)\xi_h^{k'} \\ & + \frac{\lambda^2H^2}{2}\sum\limits_{k'=1}^k\sum\limits_{h=1}^H \left(f(s_h^{k'},a_h^{k'},V_{h+1}^{k'}) - \bar{f}(s_h^{k'},a_h^{k'},V_{h+1}^{k'})\right)^2, \forall (k,h) \in [K]\times [H]\Big) \geq  1-e^{-x}
    \end{aligned}
\end{equation}
After setting $x = \log(\frac{1}{\delta})$, $\lambda = \frac{1}{H^2}$, and conditioned the event that Eq.(\ref{union bound for miss error}) holds, that is to say, 
\begin{equation}
    \begin{aligned}
        &\ \ \ \ \ \sum\limits_{k'=1}^k\sum\limits_{h=1}^H \left(f(s_h^{k'},a_h^{k'},V_{h+1}^{k'}) - \bar{f}(s_h^{k'},a_h^{k'},V_{h+1}^{k'})\right)\xi_h^{k'} \\ &\leq  \sqrt{\sum\limits_{k'=1}^k\sum\limits_{h=1}^H \left(f(s_h^{k'},a_h^{k'},V_{h+1}^{k'}) - \bar{f}(s_h^{k'},a_h^{k'},V_{h+1}^{k'})\right)^2} \cdot  \sqrt{\sum\limits_{k'=1}^k\sum\limits_{h=1}^H (\xi_h^{k'})^2} \\ & \leq  \sqrt{\sum\limits_{k'=1}^k\sum\limits_{h=1}^H \left(f(s_h^{k'},a_h^{k'},V_{h+1}^{k'}) - \bar{f}(s_h^{k'},a_h^{k'},V_{h+1}^{k'})\right)^2} \cdot  \sqrt{kH\zeta^2 + C'H\cdot \log(\frac{4KH}{\delta})}
    \end{aligned}
\end{equation}

Then we can derive the following result by using Eq.(\ref{E.5}) :

\begin{equation}
    \begin{aligned}
        &\mathbb{P}\Big(L_{2,k}(\bar{f})-L_{2,k}(f)-H^2\log(\frac{1}{\delta})- \sqrt{\sum\limits_{k'=1}^k\sum\limits_{h=1}^H \left(f(s_h^{k'},a_h^{k'},V_{h+1}^{k'}) - \bar{f}(s_h^{k'},a_h^{k'},V_{h+1}^{k'})\right)^2} \cdot  \sqrt{kH\zeta^2 + C'H\cdot \log(\frac{4KH}{\delta})} \\ &+ \frac{1}{2}\sum\limits_{k'=1}^k\sum\limits_{h=1}^H \left(f(s_h^{k'},a_h^{k'},V_{h+1}^{k'}) - \bar{f}(s_h^{k'},a_h^{k'},V_{h+1}^{k'})\right)^2 \leq 0, \forall k \in [K] \Big) \geq 1-\delta
    \end{aligned}
\end{equation}

which finishes the proof.

\end{proof}

\end{lemma}

\begin{lemma} \label{Discretization error}

(Discretization error)  

We denote $\mathcal{F}^{\alpha}$ as the $\alpha$-cover of function class $\mathcal{F}$. If $f^{\alpha} \in \mathcal{F}^{\alpha}$ satisfies $||f-f^{\alpha} ||_{\infty} \leq \alpha$, then 
\begin{equation} \label{D-error-equation}
    \begin{aligned}
        &\Big|\frac{1}{2} ||f^{\alpha}-\bar{f} ||_{D_H^k}^2 -\frac{1}{2} ||f-\bar{f}||_{D_H^k}^2 +||f-\bar{f}||_{D_H^k} \cdot  \sqrt{kH\zeta^2 + C'H\cdot \log(\frac{4KH |\mathcal{F}^{\alpha}|}{\delta})} \\ &- ||f^{\alpha}-\bar{f}||_{D_H^k} \cdot  \sqrt{kH\zeta^2 + C'H\cdot \log(\frac{4KH |\mathcal{F}^{\alpha}|}{\delta})} + L_{2,k}(f) - L_{2,k}(f^{\alpha})\Big| \\ & \leq  4 \alpha k H^2 + \sqrt{4\alpha k H^2}\cdot  \sqrt{kH\zeta^2 + C'H\cdot \log(\frac{4KH |\mathcal{F}^{\alpha}|}{\delta})} \ \ \ \ \ \ , \forall (k,h) \in [K] \times [H]
    \end{aligned}
\end{equation}

\begin{proof}
    If $f^{\alpha} \in \mathcal{F}^{\alpha}$ satisfies $||f-f^{\alpha} ||_{\infty} \leq \alpha$, then for any $(s,a,V) \in \mathcal{S} \times \mathcal{A} \times \mathcal{V}$, we have

\begin{equation}
    \begin{aligned}
        \left|(f^{\alpha})^2(s,a,V) - (f)^2(s,a,V)\right| \leq  2\alpha H
    \end{aligned}
\end{equation}

This implies that 

\begin{equation}
    \begin{aligned}
        &\ \ \ \ \ \left| \left(f^{\alpha}(s,a,V) - \bar{f}(s,a,V)\right)^2 - \left(f(s,a,V)-\bar{f}(s,a,V)\right)^2\right| \\ &= \left|[(f^{\alpha})(s,a,V)^2 - f(s,a,V)^2] + 2\bar{f}(s,a,V)\left(f(s,a,V)-f^{\alpha}(s,a,V)\right)\right| \\ &\leq 2\alpha H + 2\alpha H = 4 \alpha H
    \end{aligned}
\end{equation}

and for any $(k,h) \in [K] \times [H]$,

\begin{equation}
    \begin{aligned}
        &\ \ \ \ \ \left|(V_{h+1}^k(s_{h+1}^k)-f(s,a,V))^2 - (V_{h+1}^k(s_{h+1}^k)-f^{\alpha}(s,a,V))^2\right| \\ &= \left|2V_{h+1}^k(s_{h+1}^k) \left(f^{\alpha}(s,a,V) - f(s,a,V)\right) + f(s,a,V)^2 - f^{\alpha}(s,a,V)^2\right| \\ &\leq 2\alpha H + 2\alpha H = 4\alpha H
    \end{aligned}
\end{equation}

Moreover, 

\begin{equation}
    \begin{aligned}
        \left| ||f-\bar{f}||_{D_H^k} - ||f^{\alpha}-\bar{f} ||_{D_H^k}\right| \leq \sqrt{\left| ||f-\bar{f}||_{D_H^k}^2 - ||f^{\alpha}-\bar{f} ||_{D_H^k}^2\right|}
    \end{aligned}
\end{equation}

By taking the sum over $k$ and $H$, we can find that the left hand side of (\ref{D-error-equation}) is bounded by 

$$
4 \alpha k H^2 + \sqrt{4\alpha k H^2}\cdot  \sqrt{kH\zeta^2 + C'H\cdot \log(\frac{4KH |\mathcal{F}^{\alpha}|}{\delta})} 
$$

\end{proof}

\end{lemma}

\begin{lemma} \label{Least Square Bound}

With probability at least $1-\delta$, for all $k \in [K]$, we have

\begin{equation}
    \begin{aligned}
        ||\widehat{f}_{k+1} - \bar{f} ||_{D_H^k} \leq \beta_k
    \end{aligned}
\end{equation}

where

\begin{equation}
    \begin{aligned}
        \beta_k = 3\sqrt{kH}\zeta + 5\sqrt{C'H^2\cdot \log(\frac{4KH |\mathcal{F}^{\alpha}|}{\delta})} + 4 \sqrt{\alpha k H^2}
    \end{aligned}
\end{equation}

\begin{remark}
    Since the mapping $\phi: \mathcal{P} \rightarrow \mathcal{F}$ is a bijection, $\bar{P} = \phi^{-1}(\bar{f})$ satisfies 

\begin{equation} \label{high-prob-event-confidence-set}
    \begin{aligned}
        \mathbb{P}\left(\bar{P} \in \bigcap_{k \in [K]}B_k\right) \geq 1-\delta
    \end{aligned}
\end{equation}
\end{remark}

\begin{proof}
Let $\mathcal{F}^{\alpha}\subset \mathcal{F} $ be an $\alpha$-cover of $\mathcal{F}$ in the sup-norm. In other words, for any $f \in \mathcal{F}$, there is an $f^{\alpha} \in \mathcal{F}^{\alpha}$, such that $||f^{\alpha} -f||_{\infty} \leq \alpha$. By a union bound and from Lemma \ref{single-f-error}, with probability at least $1-\delta$, we have for any $f^{\alpha} \in \mathcal{F}^{\alpha}$, any $k \in [K]$, that
\begin{equation}
    \begin{aligned}
       L_{2,k}(f^{\alpha}) - L_{2,k}(\bar{f}) \geq -H^2\log(\frac{|\mathcal{F}^{\alpha}|}{\delta})- ||f^{\alpha}-\bar{f}||_{D_H^k} \cdot  \sqrt{kH\zeta^2 + C'H\cdot \log(\frac{4KH |\mathcal{F}^{\alpha}|}{\delta})} + \frac{1}{2}||f^{\alpha}-\bar{f}||_{D_H^k}^2 
    \end{aligned}
\end{equation}
Therefore, with probability at least $1-\delta$, for all $k \in [K]$, and all $f \in \mathcal{F}$, we have:
\begin{equation} \label{E.13}
    \begin{aligned}
        L_{2,k}(f) - L_{2,k}(\bar{f}) &\geq \frac{1}{2}||f-\bar{f}||_{D_H^k}^2 - H^2\log(\frac{|\mathcal{F}^{\alpha}|}{\delta})- ||f-\bar{f}||_{D_H^k} \cdot  \sqrt{kH\zeta^2 + C'H\cdot \log(\frac{4KH |\mathcal{F}^{\alpha}|}{\delta})} \\ &+ \Big\{\frac{1}{2} ||f^{\alpha}-\bar{f} ||_{D_H^k}^2 -\frac{1}{2} ||f-\bar{f}||_{D_H^k}^2 \\ &+||f-\bar{f}||_{D_H^k} \cdot  \sqrt{kH\zeta^2 + C'H\cdot \log(\frac{4KH |\mathcal{F}^{\alpha}|}{\delta})}- ||f^{\alpha}-\bar{f}||_{D_H^k} \cdot  \sqrt{kH\zeta^2 + C'H\cdot \log(\frac{4KH |\mathcal{F}^{\alpha}|}{\delta})} \\ &+ L_{2,k}(f) - L_{2,k}(f^{\alpha})\Big\}
    \end{aligned}
\end{equation}

For the last term in (\ref{E.13}), it can be bounded by Lemma \ref{Discretization error}. Moreover, since $\widehat{f}_{k+1} = \argmin_{f \in \mathcal{F}}L_{2,k}(f)$, $L_{2,k}(\widehat{f}_{k+1})-L_{2,k}(\bar{f}) \leq 0$.

Therefore we have: with probability at least $1-\delta$, for all $k \in [K]$ and all $f \in \mathcal{F}$, that

\begin{equation}
    \begin{aligned}
        &\frac{1}{2} ||\widehat{f}_{k+1}-\bar{f}||_{D_H^k}^2  - \sqrt{kH\zeta^2 + C'H\cdot \log(\frac{4KH |\mathcal{F}^{\alpha}|}{\delta})} \cdot ||\widehat{f}_{k+1}-\bar{f}||_{D_H^k}  \\ & -H^2\log(\frac{|\mathcal{F}^{\alpha}|}{\delta})  - 4 \alpha k H^2 - \sqrt{4\alpha k H^2}\cdot  \sqrt{kH\zeta^2 + C'H\cdot \log(\frac{4KH |\mathcal{F}^{\alpha}|}{\delta})} \leq 0
    \end{aligned}
\end{equation}

By solving the quadratic equation for $||\widehat{f}_{k+1}-\bar{f}||_{D_H^k}$, we can get

\begin{equation}
    \begin{aligned}
        &\ \ \ \ ||\widehat{f}_{k+1}-\bar{f}||_{D_H^k} \leq \sqrt{kH\zeta^2 + C'H\cdot \log(\frac{4KH |\mathcal{F}^{\alpha}|}{\delta})} \\ &+ \sqrt{kH\zeta^2 + C'H\cdot \log(\frac{4KH |\mathcal{F}^{\alpha}|}{\delta}) + 2H^2\log(\frac{|\mathcal{F}^{\alpha}|}{\delta}) + 8 \alpha k H^2 + 4 \sqrt{\alpha k H^2}\cdot  \sqrt{kH\zeta^2 + C'H\cdot \log(\frac{4KH |\mathcal{F}^{\alpha}|}{\delta})}} \\ & \leq \sqrt{kH}\zeta + \sqrt{C'H\cdot \log(\frac{4KH |\mathcal{F}^{\alpha}|}{\delta})} + H \sqrt{2\log(\frac{|\mathcal{F}^{\alpha}|}{\delta})} + \sqrt{2\left(kH\zeta^2 + C'H\cdot \log(\frac{4KH |\mathcal{F}^{\alpha}|}{\delta}) + 8 \alpha k H^2 \right)}\\ &\leq \sqrt{kH}\zeta + \sqrt{C'H\cdot \log(\frac{4KH |\mathcal{F}^{\alpha}|}{\delta})} + H \sqrt{2\log(\frac{|\mathcal{F}^{\alpha}|}{\delta})} + 2\sqrt{kH}\zeta + 2\sqrt{C'H\cdot \log(\frac{4KH |\mathcal{F}^{\alpha}|}{\delta})} + 4 \sqrt{\alpha k H^2} \\ &\leq 3\sqrt{kH}\zeta + 5\sqrt{C'H^2\cdot \log(\frac{4KH |\mathcal{F}^{\alpha}|}{\delta})} + 4 \sqrt{\alpha k H^2}
    \end{aligned}
\end{equation}

\end{proof}

\end{lemma}

\begin{lemma} \label{E.1}
    (Near-optimism)  Given K initial points $\{s_1^{k}\}_{k=1}^K$, we use $\{(s_h^{k*},a_h^{k*})\}_{(k,h)\in [K]\times[H]}\ $ (where $s_1^{k*} = s_1^k$, $\forall k \in [K]$) $\ $ to represent the dataset sampled by the optimal policy $\pi^*$ in the true model. For each $(k,h)\in [K] \times [H-1]$, we define $\zeta_{h+1}^{k*} = \mathbb{P}_h\left((V_{h+1}^*-V_{\bar{P},h+1}^*)(s_h^{k*},a_h^{k*})\right)-\left((V_{h+1}^*-V_{\bar{P},h+1}^*)(s_{h+1}^{k*})\right)$, and $\xi_h^{k*} = \mathbb{P}_hV_{\bar{P},h+1}^*(s_h^{k*},a_h^{k*}) - \bar{f}_h(s_h^{k*},a_h^{k*},V_{\bar{P},h+1}^*)$.  Conditioned on the event that $\bar{P} \in \cap_{k=1}^{K} B_k$ (\ref{high-prob-event-confidence-set}) holds. Then we have
$$
\sum_{k=1}^K \left(V_1^*(s_1^{k*}) - V_1^k(s_1^{k*})\right) \leq \sum\limits_{k=1}^K\sum\limits_{h=1}^{H-1} \zeta_{h+1}^{k*} + \sum\limits_{k=1}^K\sum\limits_{h=1}^H \xi_h^{k*}
$$

\begin{proof}
    First, according to the definition of $P^{(k)}$ (defined in \ref{optimisic model selection}), and since we condition on Eq.(\ref{high-prob-event-confidence-set}), $\bar{P} \in \cap_{k=1}^{K} B_k$, we have $V_1^k (s_1^{k}) \geq V_{\bar{P},1}^*(s_1^k) $, $\forall k \in [K]$.

\begin{equation}
    \begin{aligned}
        &\ \ \ \sum\limits_{k=1}^K \left(V_1^*(s_1^{k*}) - V_1^k(s_1^{k*})\right) \\ &\leq \sum\limits_{k=1}^K \left(V_1^*(s_1^{k*}) - V_{\bar{P},1}^*(s_1^{k*}) \right) \\ & \leq \sum\limits_{k=1}^K \left(Q_1^*(s_1^{k*},a_1^{k*}) - Q_{\bar{P},1}^*(s_1^{k*},a_1^{k*})\right) \\ &= \sum\limits_{k=1}^K \left(r_1(s_1^{k*},a_1^{k*}) + \mathbb{P}_1V_2^*(s_1^{k*},a_1^{k*})-\left[r_1(s_1^{k*},a_1^{k*}) + \mathbb{\bar{P}}_1V_{\bar{P},2}^*(s_1^{k*},a_1^{k*})\right]\right) \\ &= \sum\limits_{k=1}^K\left(\mathbb{P}_1V_2^*(s_1^{k*},a_1^{k*}) - \mathbb{P}_1V_{\bar{P},2}^*(s_1^{k*},a_1^{k*})+ \mathbb{P}_1V_{\bar{P},2}^*(s_1^{k*},a_1^{k*}) -\mathbb{\bar{P}}_1V_{\bar{P},2}^*(s_1^{k*},a_1^{k*})\right) \\ &= \sum\limits_{k=1}^K \left(V_2^*(s_2^{k*})-V_{\bar{P},2}^*(s_2^{k*})+\zeta_2^{k*} +  \mathbb{P}_1V_{\bar{P},2}^*(s_1^{k*},a_1^{k*}) - \bar{f}_1(s_1^{k*},a_1^{k*},V_{\bar{P},2}^*)\right) \leq \cdots \\ & \leq \sum\limits_{k=1}^K\sum\limits_{h=1}^{H-1} \zeta_{h+1}^{k*} + \sum\limits_{k=1}^K\sum\limits_{h=1}^H \xi_h^{k*}
    \end{aligned}
\end{equation}

\end{proof}

\end{lemma}

\begin{lemma} \label{E.2}

For each $(k,h)\in [K] \times [H-1]$, we define $\zeta_{h+1}^{k} = \mathbb{P}_h\left((V_{h+1}^k-V_{h+1}^{\pi_k})(s_h^{k},a_h^{k})\right)-\left((V_{h+1}^k-V_{h+1}^{\pi_k})(s_{h+1}^{k})\right)$, and $W_k = \sup_{\widetilde{\mathbb{P}}^k\in B_k}\sum\limits_{h=1}^{H-1} \left(\widetilde{\mathbb{P}}_h^k-\mathbb{P}_h\right) V_{h+1}^k (s_h^{k},a_h^{k})$. Then we have
$$
\sum_{k=1}^K \left(V_1^k(s_1^{k}) - V_1^{\pi_k}(s_1^{k})\right) \leq \sum\limits_{k=1}^K\sum\limits_{h=1}^{H-1} \zeta_{h+1}^{k} + \sum\limits_{k=1}^K W_k
$$

\begin{proof}
    First, for any $k \in [K]$, we have the following decomposition:
\begin{equation}
    \begin{aligned}
    &\ \ \ V_1^k(s_1^{k}) - V_1^{\pi_k}(s_1^{k}) \\ &= r_1(s_1^k,a_1^k) + \mathbb{P}_1^{(k)}V_2^k(s_h^k,a_h^k) - \left(r_1(s_1^{k},a_1^{k})+ \mathbb{P}_1V_{2}^{\pi_k}(s_1^k,a_1^k)\right) \\ &= (\mathbb{P}_1^{(k)}-\mathbb{P}_1)V_2^k(s_1^k,a_1^k)+ \mathbb{P}_1 (V_{2}^k-V_{2}^{\pi_k})(s_1^k,a_1^k) \\ &= (\mathbb{P}_1^{(k)}-\mathbb{P}_1)V_2^k(s_1^k,a_1^k) + \zeta_2^k + V_2^k(s_2^{k}) - V_2^{\pi_k}(s_2^{k})=\cdots \\ &= \sum\limits_{h=1}^{H-1} \left(\mathbb{P}_h^{(k)}-\mathbb{P}_h\right) V_{h+1}^k(s_h^k,a_h^k) + \sum\limits_{h=1}^{H-1} \zeta_{h+1}^k \\ &\leq W_k + \sum\limits_{h=1}^{H-1} \zeta_{h+1}^k
    \end{aligned}
\end{equation}

Therefore, we have
$$
\sum_{k=1}^K \left(V_1^k(s_1^{k}) - V_1^{\pi_k}(s_1^{k})\right) \leq \sum\limits_{k=1}^K\sum\limits_{h=1}^{H-1} \zeta_{h+1}^{k} + \sum\limits_{k=1}^K W_k
$$

\end{proof}

\end{lemma}

\begin{lemma} \label{W_k_lemma}
Let $\alpha >0$, and $d: = \operatorname{dim}_{E}(\mathcal{F}, \alpha)$, where $\mathcal{F}$ is the function class the algorithm used to approximate the ground-truth model (\ref{definition of function class}). Then for any non-decreasing sequences $(\beta_k^2)_{k=1}^K$, conditioned on the event that $\bar{P} \in \bigcap_{k \in [K]} B_k$, where $B_k = \{\widetilde{P} \in \mathcal{P}: L_k(\widetilde{P}, \widehat{P}^{(k)}) \leq \beta_k^2\}$, we have:

\begin{equation}
    \begin{aligned}
        \sum\limits_{k=1}^K W_k \leq \alpha + H(d\land K(H-1)) + 4 \beta_K \sqrt{dK(H-1)} + \sqrt{8KH^2 \log(\frac{2}{\delta})} + KH\zeta
    \end{aligned}
\end{equation}

\begin{proof}
First, we denote
\begin{equation}
    \begin{aligned}
        \mathcal{F}_t(\beta_k) = \{f\in \mathcal{F}: ||f-\widehat{f}_k ||_{D_H^k} \leq \beta_k\} = \{\phi(P): P \in B_k\}
    \end{aligned}
\end{equation}

and for the convenience of notation, we denote
$$
\widetilde{\mathcal{F}}_t = \mathcal{F}_t(\beta_k) \ \ \text{for}\ \  t \in [(k-1)(H-1)+1, k(H-1)]
$$
and
\begin{equation}
    \begin{aligned}
        &x_1 = (s_1^1,a_1^1,V_2^1), x_2 = (s_2^1,a_2^1,V_3^1), \cdots , x_{H-1} = (s_{H-1}^1,a_{H-1}^1,V_H^1) \\ & x_H = (s_1^2,a_1^2,V_2^2), x_{H+1} = (s_2^2,a_2^2,V_3^2), \cdots , x_{2H} = (s_{H-1}^2,a_{H-1}^2,V_H^2) \\ &\cdots \cdots \\ & x_{(K-1)(H-1)+1} = (s_1^K,a_1^K,V_2^K), x_{(K-1)(H-1)+2} = (s_2^K,a_2^K,V_3^K), \cdots , x_{K(H-1)} = (s_{H-1}^K,a_{H-1}^K,V_H^K)
    \end{aligned}
\end{equation}

According to the definition of $W_k$, we have
\begin{equation} \label{sum_W_k}
    \begin{aligned}
        &\ \ \ \ \sum\limits_{k=1}^K W_k \leq \sum\limits_{k=1}^K \sup_{\widetilde{\mathbb{P}}^k\in B_k}\sum\limits_{h=1}^{H-1} \left(\widetilde{\mathbb{P}}_h^k-\mathbb{P}_h\right) V_{h+1}^k (s_h^{k},a_h^{k}) \\ &= \sum\limits_{k=1}^K \sup_{\widetilde{\mathbb{P}}^k\in B_k}\sum\limits_{h=1}^{H-1} \left(f_{\widetilde{\mathbb{P}}^k}(s_h^k,a_h^k,V_{h+1}^k)-\bar{f}(s_h^k,a_h^k,V_{h+1}^k) + \bar{f}(s_h^k,a_h^k,V_{h+1}^k) - \mathbb{P}_hV_{h+1}^k(s_h^k,a_h^k)\right) \\ & \leq \sum\limits_{k=1}^K \sup_{\widetilde{\mathbb{P}}^k\in B_k}\sum\limits_{h=1}^{H-1} \left(f_{\widetilde{\mathbb{P}}^k}(s_h^k,a_h^k,V_{h+1}^k)-\bar{f}(s_h^k,a_h^k,V_{h+1}^k)\right) + \sum\limits_{k=1}^K\sum\limits_{h=1}^{H-1}\left(\bar{f}(s_h^k,a_h^k,V_{h+1}^k) - \mathbb{P}_hV_{h+1}^k(s_h^k,a_h^k)\right) \\ & \leq \sum\limits_{t=1}^{K(H-1)} w_{\mathcal{F}_t}(x_t) + \sum\limits_{k=1}^K\sum\limits_{h=1}^{H-1}\xi_h(s_h^k,a_h^k)
    \end{aligned}
\end{equation}

For the last line of (\ref{sum_W_k}), $w_{\widetilde{\mathcal{F}}}(x) = \sup\limits_{f_1,f_2 \in \widetilde{\mathcal{F}}}\left(f_1(x) - f_2(x)\right)$  represents the width function.

By using Lemma \ref{russo_2}, the first term of the above equation can be upper bounded by
\begin{equation}
    \begin{aligned} \label{sum_W_k_1}
    \sum\limits_{t=1}^{K(H-1)} w_{\mathcal{F}_t}(x_t)  \leq \alpha + H(d \land K(H-1)) + 4\beta_K\sqrt{d K(H-1)}
    \end{aligned}
\end{equation}

For the second term, by using Assumption \ref{Ass4} and Azuma-Hoeffding's inequality, we have with probability at least $1-\delta$,

\begin{equation}
    \begin{aligned} \label{sum_W_k_2}
        \sum\limits_{k=1}^K\sum\limits_{h=1}^{H-1}| \xi_h(s_h^k,a_h^k) | \leq \sqrt{8KH^2 \log(\frac{2}{\delta})} + KH\zeta 
    \end{aligned}
\end{equation}

By combining (\ref{sum_W_k}), (\ref{sum_W_k_1}), (\ref{sum_W_k_2}), we have
\begin{equation}
    \begin{aligned}
        \sum\limits_{k=1}^K W_k \leq \alpha + H(d\land K(H-1)) + 4 \beta_K \sqrt{dK(H-1)} + \sqrt{8KH^2 \log(\frac{2}{\delta})} + KH\zeta 
    \end{aligned}
\end{equation}

\end{proof}

\end{lemma}

Now we are able to analyze the regret bound of Robust-UCRL-VTR. We define the regret of the algorithm as 
$$
R_K = \sum\limits_{k=1}^K (V_1^*(s_1^k) - V_1^{\pi_k}(s_1^k)).
$$

\begin{theorem} \label{Thm model-based}
(Regret bound of robust model-based methods)

Let Assumption \ref{ass:cover2} and \ref{Ass4} hold and $\alpha \in (0,1)$. For each $k \in [K]$, let $\beta_k$ be 
\begin{equation} \label{beta's definition}
    \begin{aligned}
        \beta_k = 3\sqrt{kH}\zeta + 5\sqrt{C'H^2\cdot \log\left(\frac{4KH \mathcal{N}(\mathcal{F},\alpha)}{\delta}\right)} + 4 \sqrt{\alpha k H^2}
    \end{aligned}
\end{equation}
then with probability at least $1-\delta$, the total regret of Algorithm \ref{Algorithm model-based} is at most
$\widetilde{O}\left(\sqrt{d_E}KH\zeta \log(1/\delta) + \sqrt{d_E^2KH^3} \log(1/\delta)\right)$, where $d_E$ represents the eluder dimension of the function class.

\begin{proof}
    First, for any $k \in [K]$, and $h \in [H-1]$, $\zeta_{h+1}^k \in [-H,H]$, and $\{\zeta_{h+1}^{k}\}_{(k,h) \in [K] \times [H-1]}$ is a martingale difference sequence. Thus, with probability at least $1-\delta/2$, $\sum\limits_{k=1}^K\sum\limits_{h=1}^{H-1} \zeta_{h+1}^{k} \leq \sqrt{2KH^3\log(\frac{2}{\delta})}$. In the same way, with probability at least $1-\delta/2$, $\sum\limits_{k=1}^K\sum\limits_{h=1}^{H-1} \zeta_{h+1}^{k*} \leq \sqrt{2KH^3\log(\frac{2}{\delta})}$. Then conditioned on the event in Lemma \ref{Least Square Bound}, we can obtain the regret bound by applying Lemma \ref{E.1}, \ref{E.2}, and \ref{W_k_lemma}:

\begin{equation}
    \begin{aligned}
        R_K  &= \sum\limits_{k=1}^K (V_1^*(s_1^k) - V_1^{\pi_k}(s_1^k)) \\ &\leq \sum\limits_{k=1}^K (V_1^*(s_1^k)-V_1^k(s_1^k)) + \sum\limits_{k=1}^K (V_1^k(s_1^k) - V_1^{\pi_k}(s_1^k)) \\ & \leq \sum\limits_{k=1}^K\sum\limits_{h=1}^{H-1} \zeta_{h+1}^{k*} + \sum\limits_{k=1}^K\sum\limits_{h=1}^H \xi_h^{k*} + \sum\limits_{k=1}^K\sum\limits_{h=1}^{H-1} \zeta_{h+1}^{k} + \sum\limits_{k=1}^K W_k \\ &\leq 2 \sqrt{KH^3\log(\frac{2}{\delta})} + \alpha + H(d\land K(H-1)) +  4 \beta_K \sqrt{dK(H-1)} + 2 \left(\sqrt{8KH^2 \log(\frac{2}{\delta})} + KH\zeta \right) \\ &\leq \alpha +  H(d\land K(H-1))+  4\beta_K \sqrt{dK(H-1)} + 2 KH\zeta + 8 \sqrt{KH^3\log(\frac{2}{\delta})}
    \end{aligned}
\end{equation}
By applying the definition of $\beta_K$ in (\ref{beta's definition}), we complete our proof.

\end{proof}

\end{theorem}

\section{Proof of Theorem \ref{Thm 3.2}} \label{Sec: meta algorithm}
In this section, we present the complete proof of Theorem \ref{Thm 3.2}. Prior to providing the proof of the theorem itself (Theorem \ref{Thm A.11}), we first establish the groundwork by introducing the following lemmas.

\begin{lemma} \label{single-LSVI}
    (Estimation error of Single-epoch-Algorithm)
    
    For the Single-epoch-Algorithm (Algorithm \ref{Algorithm Single}),  with probability at least $1-\delta$, we have
    \[
    |\overline{V_1}(s_1) - V_1^{\pi_{\text{ave}}}(s_1)| \leq \sqrt{\frac{8H^2\log(\frac{2}{\delta})}{K}}
    \]
    where $\pi_{\text{ave}} = \text{Unif}\ \{\pi^1,\pi^2, \cdots, \pi^K\}$.
\end{lemma}
\begin{proof}
     For $k=1,2,\cdots, K$, $R_k = \sum\limits_{h=1}^H r_h(s_h^k,a_h^k)$, where $\{(s_h^k,a_h^k)\}_{h\in [H]}$ is sampled under policy $\pi^k$. Next we define $Z_0=0$, $Z_l = \sum\limits_{k=1}^l R_k - \sum\limits_{k=1}^l V^{\pi^k}$, $l=1,2,\cdots,K$. Then we have
     \begin{equation}
         \begin{aligned}
              \mathbb{E}[{Z_l}| \mathcal{F}_{l-1}] &= \sum\limits_{k=1}^{l-1}R_k + \mathbb{E}[R_l|\mathcal{F}_{l-1}]-\sum\limits_{k=1}^l V^{\pi^k} \\ &= \sum\limits_{k=1}^{l-1}R_k + V^{\pi^l}-\sum\limits_{k=1}^l V^{\pi^k} \\ &= \sum\limits_{k=1}^{l-1}R_k - \sum\limits_{k=1}^{l-1}V^{\pi^k} = Z_{l-1}
         \end{aligned}
     \end{equation}
    This shows that $\{Z_l\}_{l=1}^K$ is a martingale. Moreover, $|Z_l-Z_{l-1}|\leq 2H$, $\forall l \in [K]$. By Azuma-Hoeffding's inequality, we have
    for any $\epsilon \geq 0$, 
    \[
    \mathbb{P}\left(|\frac{1}{K}\sum\limits_{k=1}^K R_k - \frac{1}{K}\sum\limits_{k=1}^K V^{\pi_k}|\geq \frac{\epsilon}{K}\right)\leq 2\exp\left\{-\frac{\epsilon^2}{8KH^2}\right\}
    \]
    By using the fact that $\overline{V_1}(s_1)=\frac{1}{K}\sum\limits_{k=1}^K R_k$, and $ V_1^{\pi_{\text{ave}}}(s_1) = \frac{1}{K}\sum\limits_{k=1}^K V^{\pi_k}$, we complete our proof.
\end{proof}

For the analysis of Algorithm \ref{Algorithm unknown}, we need the following high probability events to represent that we get a good policy $\pi^{(i)}$ in epochs with correct misspecified parameter setting (i.e. $\zeta^{(i)} \geq \zeta $): 

For each epoch $i \in \left\{0,1,2,\cdots,  \lfloor \log_2\left(\frac{1}{\zeta}\right) \rfloor \land \lfloor \log_2\left(\sqrt{3K+1}\right) \rfloor \right\}$, we define:
\[
\mathcal{E}_i (\delta) = \left\{V_1^*(s_1) \geq V_1^{\pi^{(i)}}(s_1) \geq V_1^*(s_1) -L(d,H,\delta)\cdot \left(\frac{d^{\alpha}H^{\beta}}{\sqrt{K^{(i)}}}+ d^{\alpha}H^{\beta}\zeta^{(i)}\right)\right\}
\]
where $\pi^{(i)}$ is the uniform mixture of policies gained from the $i$-th epoch, and $L(d,H,\delta)$ is a function of logarithmic order on $d,H,\delta$.

We know that for any $i \in \left\{0,1,2,\cdots,  \lfloor \log_2\left(\frac{1}{\zeta}\right) \rfloor \land \lfloor \log_2\left(\sqrt{3K+1}\right) \rfloor \right\}$, $\zeta^{(i)} = \frac{1}{2^i} \geq \zeta$, then according to the property of input base algorithm, i.e., it has a regret bound of $\widetilde{O}\left(d^{\alpha}H^{\beta}(\sqrt{K}+K\cdot\zeta)\right)$ if the input misspecified parameter is $\zeta$, we know that $\mathcal{E}_i(\delta)$ happens with probability at least $1-\delta$. Next, we define the intersection of all these events:
\begin{equation} \label{good event 1}
    \begin{aligned}
        \mathcal{E}(\zeta, \delta) = \bigcap\limits_{i=0}^{\lfloor \log_2\left(\frac{1}{\zeta}\right) \rfloor \land \lfloor \log_2\left(\sqrt{3K+1}\right) \rfloor} \mathcal{E}_i\left(\frac{\delta}{2 (\lceil \log_2(\frac{1}{\zeta}) \rceil \land \lfloor \log_2\left(\sqrt{3K+1}\right) \rfloor )}\right)
    \end{aligned}
\end{equation}

By taking the union bound, we have
\[
\mathbb{P}\left( \mathcal{E}(\zeta, \delta) \right) \geq 1-\delta/2
\]

Moreover, we define the following events to represent we get a good estimator of value function for each epoch $i \in \left\{0,1,2, \cdots,\lfloor \log_2\left(\sqrt{3K+1}\right) \rfloor \right\}$.

\[
\mathcal{G}_i(\delta) = \left\{|\overline{V_1}^{(i)}(s_1) - V_1^{\pi^{(i)}}(s_1)| \leq \sqrt{\frac{8H^2\log(\frac{2}{\delta})}{K^{(i)}}} \right\}
\]

Although the last few epochs are executed with the same policy, this process can still be regared as a martingale, and Lemma \ref{single-LSVI} still holds.
From Lemma \ref{single-LSVI}, we know that $\mathcal{G}_i(\delta)$ happens with probability at least $1-\delta$. Next, we define the intersection of all these events:
\begin{equation} \label{good event 2}
    \begin{aligned}
        \mathcal{G}(\delta) = \bigcap\limits_{i=0}^{\lfloor \log_2\left(\sqrt{3K+1}\right) \rfloor} \mathcal{G}_i\left(\frac{\delta}{2 \lceil \log_2\left(\sqrt{3K+1}\right) \rceil}\right)
    \end{aligned}
\end{equation}
By taking the union bound, we have
\[
\mathbb{P}\left(\mathcal{G}(\delta)\right) \geq 1-\delta/2
\]
Therefore, 
\begin{equation} \label{high_prob_event}
    \begin{aligned}
       \mathbb{P}\left(\mathcal{E}(\zeta,\delta) \bigcap \mathcal{G}(\delta)\right) \geq 1-\delta 
    \end{aligned}
\end{equation}

\begin{theorem} \label{Thm A.11}
\textbf{(Regret bound under locally-bounded misspecified MDP with unknown misspecified parameter $\zeta$)}

Suppose the input base algorithm \textsl{Alg.} that needs to know the locally-bounded misspecified parameter $\zeta$ has a regret bound of $\widetilde{O}\left(d^{\alpha}H^{\beta}(\sqrt{K}+K\cdot\zeta)\right)$, then conditioned on the high probability event $\mathcal{E}(\zeta,\delta) \bigcap \mathcal{G}(\delta)$ in (\ref{high_prob_event}), the total regret of our meta-algorithm (Algorithm \ref{Algorithm unknown}) is still $\widetilde{O}\left(d^{\alpha}H^{\beta}(\sqrt{K}+K\cdot\zeta)\right)$.
\end{theorem}

\begin{proof}
    Conditioned on the event $\mathcal{E}(\zeta,\delta) \bigcap \mathcal{G}(\delta)$ (The definition is in (\ref{good event 1}) and (\ref{good event 2})), we have a claim here: for all $i$ such that $\zeta^{(i)} \geq \zeta$,
    \begin{equation} \label{Claim}
        \begin{aligned}
            |\overline{V_1}^{(i)} - \overline{V_1}^{(i-1)}| \leq  C(d,H,\delta) \cdot \zeta^{(i)}
        \end{aligned}
    \end{equation}
     where 
    \begin{equation} \label{Important constant}
        \begin{aligned}
             C(d,H,\delta) = 3\sqrt{8H^2\log\left(\frac{2\lceil \log_2\left(\sqrt{3K+1}\right) \rceil}{\delta}\right)} + 6 L\left(d,H,\frac{2\lceil \log_2\left(\sqrt{3K+1}\right) \rceil}{\delta}\right) \cdot d^{\alpha}H^{\beta}
        \end{aligned}
    \end{equation}
   is a function of $(d,H,\delta)$, which has an order of $\widetilde{O}(d^{\alpha}H^{\beta})$.

    This is because 
    \begin{equation}
        \begin{aligned}
            &|\overline{V_1}^{(i)} - \overline{V_1}^{(i-1)}| \leq |\overline{V_1}^{(i)} - V_1^{\pi^{(i)}}| + |V_1^{\pi^{(i)}} - V_1^{\pi^{(i-1)}}| + 
            |V_1^{\pi^{(i-1)}} - \overline{V_1}^{(i-1)}| \\ &\leq \sqrt{\frac{8H^2\log(\frac{2\lceil \log_2\left(\sqrt{3K+1}\right) \rceil}{\delta})}{K^{(i)}}} \\  &+ L\left(d,H,\frac{2\lceil \log_2\left(\sqrt{3K+1}\right) \rceil}{\delta}\right)\cdot\left(\frac{d^{\alpha} H^{\beta}}{\sqrt{K^{(i)}}} + d^{\alpha}H^{\beta}\cdot \zeta^{(i)}+\frac{d^{\alpha} H^{\beta}}{\sqrt{K^{(i-1)}}} + d^{\alpha}H^{\beta}\cdot \zeta^{(i-1)}\right)\\&+\sqrt{\frac{8H^2\log(\frac{2\lceil \log_2\left(\sqrt{3K+1}\right) \rceil}{\delta})}{K^{(i-1)}}} \\ &\leq  C(d,H,\delta) \cdot \zeta^{(i)}
        \end{aligned}
    \end{equation}
    The above inequality is derived by using (\ref{good event 1}), (\ref{good event 2}), and the fact that $K^{(i)} = \frac{1}{(\zeta^{(i)})^2}$
    
 Then we discuss following two cases with respect to $\zeta$.
\paragraph{ Case 1 $0<\zeta < \frac{1}{2^{\lfloor \log_2\left(\sqrt{3K+1}\right) \rfloor}}$}

In this case, for all $i \in \left\{0,1,2, \cdots,\lfloor \log_2\left(\sqrt{3K+1}\right) \rfloor \right\}$, $\zeta^{(i)} \geq \zeta$.
This means that the algorithm will not violate the condition $|\overline{V_1}^{(i)} - \overline{V_1}^{(i-1)}| \leq C(d,H,\delta) \cdot \zeta^{(i)}$ for all $i \in \left\{1,2, \cdots,\lfloor \log_2\left(\sqrt{3K+1}\right) \rfloor \right\}$.
Then
\begin{equation}
    \begin{aligned}
        \text{Regret}(K) &= \sum\limits_{i=0}^{\lfloor \log_2\left(\sqrt{3K+1}\right) \rfloor} \widetilde{O}\left(d^{\alpha}H^{\beta}(\sqrt{K^{(i)}}+K^{(i)}\cdot\zeta^{(i)})\right) \\ &\leq \sum\limits_{i=0}^{\lfloor \log_2\left(\sqrt{3K+1}\right) \rfloor} \widetilde{O}\left(d^{\alpha}H^{\beta}\sqrt{K^{(i)}}\right) \ \ \left(\zeta^{(i)} = \sqrt{\frac{1}{K^{(i)}}}\right) \\ &=\widetilde{O}(d^{\alpha}H^{\beta})\cdot  \sum\limits_{i=0}^{\lfloor \log_2\left(\sqrt{3K+1}\right) \rfloor} 2^i \\ & \leq \widetilde{O}(d^{\alpha}H^{\beta})  \cdot \left(\sqrt{3K+1}\right) \\ &= \widetilde{O}\left(\sqrt{K}d^{\alpha}H^{\beta}\right)
    \end{aligned}
\end{equation}

\paragraph{ Case 2 $\frac{1}{2^{\lfloor \log_2\left(\sqrt{3K+1}\right) \rfloor}} \leq \zeta \leq 1$}

We have for any $i \geq 1$ such that $\zeta^{(i)}\geq \zeta$, $|\overline{V_1}^{(i)} - \overline{V_1}^{(i-1)}| \leq C(d,H,\delta) \cdot \zeta^{(i)}$. We denote $j$ to be the first epoch number that violates the condition. This means that 
\begin{equation}
    \begin{aligned}
        |\overline{V_1}^{(j)} - \overline{V_1}^{(j-1)}| >  C(d,H,\delta) \cdot \zeta^{(j)}
    \end{aligned}
\end{equation}
while

\begin{equation} \label{gap}
    \begin{aligned}
         |\overline{V_1}^{(i)} - \overline{V_1}^{(i-1)}| \leq  C(d,H,\delta) \cdot \zeta^{(i)} ,\ \ \ \  \forall i = 1,\cdots, j-1
    \end{aligned}
\end{equation}

According to our claim (\ref{Claim}), we know that $\zeta^{(j)} < \zeta$. Moreover, according to our exponentially decreasing $\{\zeta^{(i)}\}$, there must exist a $\zeta^{(s)}$, such that $ \zeta \leq \zeta^{(s)} < 2\zeta$. 
For the gap between $\overline{V_1}^{(j-1)}$ and $\overline{V_1}^{(s)}$, from (\ref{gap}) we have
\begin{equation}
    \begin{aligned}
        |\overline{V_1}^{(j-1)} - \overline{V_1}^{(s)} | &\leq |\overline{V_1}^{(j-1)} - \overline{V_1}^{(j-2)} | + |\overline{V_1}^{(j-2)} - \overline{V_1}^{(j-3)} | + \cdots + |\overline{V_1}^{(s+1)} - \overline{V_1}^{(s)} | \\ &\leq C(d,H,\delta) \cdot \left(\frac{1}{2^{j-1}} + \frac{1}{2^{j-2}}+ \cdots + \frac{1}{2^{s+1}}\right) \\ &\leq  C(d,H,\delta) \cdot \frac{1}{2^s} = C(d,H,\delta) \cdot \zeta^{(s)}
    \end{aligned}
\end{equation}

Then we can bound the gap between $V_1^{\pi^{(j-1)}}$ and $V_1^{\pi^{(s)}}$
\begin{equation}
    \begin{aligned}
    |V_1^{\pi^{(j-1)}} - V_1^{\pi^{(s)}} | &\leq |V_1^{\pi^{(j-1)}} - \overline{V_1}^{(j-1)}| + |\overline{V_1}^{(j-1)} - \overline{V_1}^{(s)} | + |\overline{V_1}^{(s)} - V_1^{\pi^{(s)}}| \\ &\leq C(d,H,\delta) \cdot \zeta^{(j-1)} + C(d,H,\delta) \cdot \zeta^{(s)} + C(d,H,\delta) \cdot \zeta^{(s)} \\ &\leq 3 C(d,H,\delta) \cdot \zeta^{(s)}
    \end{aligned}
\end{equation}
Similarly, for any $s+1 \leq i \leq j-1$, we have
\[
|V_1^{\pi^{(i)}} - V_1^{\pi^{(s)}} | \leq 3 C(d,H,\delta) \cdot \zeta^{(s)}
\]
Therefore, for any $s+1 \leq i \leq j-1$, we have
\begin{equation} \label{50}
    \begin{aligned}
        V_1^* - V_1^{\pi^{(i)}} &=  V_1^* - V_1^{\pi^{(s)}} + V_1^{\pi^{(s)}} -V_1^{\pi^{(i)}} \\ &\leq C(d,H,\delta)\cdot \zeta^{(s)} + 3 C(d,H,\delta) \cdot \zeta^{(s)} \\ &= 4 C(d,H,\delta) \cdot \zeta^{(s)}
    \end{aligned}
\end{equation}

Next, we will give the regret bound in this case.
\begin{equation}
    \begin{aligned}
        \text{Regret}(K) &= \sum\limits_{i=0}^{s} \widetilde{O}\left(d^{\alpha}H^{\beta}(\sqrt{K^{(i)}}+K^{(i)}\cdot\zeta^{(i)})\right) + \sum\limits_{i=s+1}^{j-1} K^{(i)}(V_1^*-V_1^{\pi^{(i)}}) +(K - \sum\limits_{i=0}^{j-1} K^{(i)})(V_1^* - V_1^{\pi^{(j-1)}}) \\ &\leq \widetilde{O}(d^{\alpha}H^{\beta})\cdot \sum\limits_{i=0}^{s} 2^i + \left[ \sum\limits_{i=s+1}^{j-1} K^{(i)} + (K - \sum\limits_{i=0}^{j-1}K^{(i)}) \right] \cdot 4 C(d,H,\delta) \cdot \zeta^{(s)}  \ \ \ (\text{By}\  (\ref{50})) \\ &\leq \widetilde{O} (d^{\alpha}H^{\beta}) \cdot \sum\limits_{i=0}^{\lfloor \log_2\left(\sqrt{3K+1}\right) \rfloor} 2^i + 4 K \cdot C(d,H,\delta) \cdot \zeta^{(s)} \\ &\leq  \widetilde{O}(d^{\alpha}H^{\beta}) \cdot \left(\sqrt{3K+1}\right) + 4 K \cdot C(d,H,\delta) \cdot \zeta^{(s)} \\ &\leq \widetilde{O}\left(\sqrt{K}d^{\alpha}H^{\beta}\right) + 4K \cdot C(d,H,\delta) \cdot 2\zeta \leq \widetilde{O}\left(d^{\alpha}H^{\beta}(\sqrt{K}+K\cdot\zeta)\right)
    \end{aligned}
\end{equation}
This completes the proof.

\end{proof}

\section{Comparison with Transfer Error in Policy-Based Methods} \label{sec: comparsion} 
\vspace{-2mm}

In those policy-based methods \citep{agarwal2020pc,feng2021provably,zanette2021cautiously,li2023lowswitching}, they use a notion called transfer error to measure the model misspecification. They assume that the minimizer $\theta^*$ of the misspecification error with respect to state-action function with some policy $\pi$, $Q^{\pi}$,    under the distribution of policy cover  
has a bounded transfer error when transferred to an arbitrary distribution $d^{\pi}$ induced by a policy $\pi$. 
Formally, they define:
$
\theta^* = \text{argmin}_{||\theta|| \leq W}\mathbb{E}_{(s,a)\sim\rho_{\text{cover}}}\left[\phi(s,a)^{\top} \theta - Q^{\pi}(s,a)\right]^2
$
then assume that for any policy $\pi$, 
$$
\mathbb{E}_{(s,a)\sim d^{\pi}}\left[\phi(s,a)^{\top} \theta^* - Q^{\pi}(s,a)\right]^2 \leq \zeta^2
$$

While a direct comparison between the bounded transfer error assumption and our Assumption \ref{Ass:general-value} and \ref{Ass4} is not feasible, they share a common characteristic. Both assumptions measure model misspecification error based on the average sense of the policy-induced distribution, rather than considering the maximum misspecification error across all state-action pairs. We consider this shared attribute to    be a crucial step in establishing a connection between value-based (or model-based) and policy-based approaches regarding model misspecification.

\section{Auxiliary Lemmas} \label{Sec C}
In this section, we provide the necessary auxiliary lemmas that we will utilize in our proof.

\paragraph{Notations} $\mathcal{N}_{\epsilon}$ denotes the $\epsilon$-covering number of the class $\mathcal{V}$ with respect to distance $d(V_1,V_2): = \sup_{s\in\mathcal{S}}[V_1(s)-V_2(s)]$.

\begin{lemma} \label{C.1}
    Let $\Lambda_t = \lambda {I} +\sum\limits_{i=1}^{t}\phi_i\phi_i^{\top}$ where $\phi_i \in \mathbb{R}^d$ and $\lambda>0$. Then:
\[
\sum\limits_{i=1}^{t} \phi_i^{\top}(\Lambda_t)^{-1}\phi_i \leq d
\]
\end{lemma}

\begin{lemma} \label{C.2}
    (\hyperlink{3}{Y. Abbasi-Yadkori et al.,2011}). Let $\{\phi_t\}_{t\geq 0}$ be a bounded sequence in $\mathbb{R}^d$ satisfying $\sup_{t\geq 0}||\phi_t || \geq 1$. Let $\Lambda_0 \in \mathbb{R}^{d\times d}$ be a positive definite matrix. For any $t\geq 0$, we define $\Lambda_t = \Lambda_0 +\sum\limits_{j=1}^{t}\phi_j\phi_j^{\top}$. Then if the smallest eigenvalue of $\Lambda_0$ satisfies $\lambda_{\text{min}}(\Lambda_0)\geq 1$, we have
\[
\log\left[\frac{\text{det}(\Lambda_t)}{\text{det}(\Lambda_0)}\right] \leq \sum\limits_{j=1}^{t}\phi_j^{\top}\Lambda_{j-1}^{-1}\phi_j \leq 2 \log\left[\frac{\text{det}(\Lambda_t)}{\text{det}(\Lambda_0)}\right]
\].
\end{lemma}

\begin{lemma} \label{C.3}
 (Lemma D.4 in \citep{jin2020provably}). Let $\{s_{\tau}\}_{\tau=1}^{\infty}$ be a stochastic process on state space $\mathcal{S}$ with corresponding filtration $\{\mathcal{F}_{\tau}\}_{\tau=0}^{\infty}$. Let $\{\phi_{\tau}\}_{\tau=0}^{\infty}$ be an $\mathbb{R}^d$-valued stochastic process where $\phi_{\tau} \in \mathcal{F}_{\tau-1}$, and $||\phi_{\tau}|| \leq 1$. Let $\Lambda_k = \lambda I_d + \sum\limits_{\tau=1}^{k-1}\phi_{\tau}\phi_{\tau}^{\top}$. Then with probablility at least $1-\delta$, for all $k\geq 0$ and $V \in \mathcal{V}$ such that $\sup_{s\in\mathcal{S}}|V(s)|\leq H$, we have

\[
\left|\left|\sum\limits_{\tau=1}^k \phi_{\tau}\left(V(s_{\tau})-\mathbb{E}[V(s_{\tau})|\mathcal{F}_{\tau-1}]\right)\right|\right|_{\Lambda_k^{-1}}^2 \leq 4H^2\left(\frac{d}{2} \log\left(\frac{k+\lambda}{\lambda}\right) +\log\left(\frac{\mathcal{N}_{\epsilon}}{\delta}\right)\right) + \frac{8k^2\epsilon^2}{\lambda}
\]

\end{lemma}

\begin{lemma}
For any $\epsilon >0$, the $\epsilon$-covering number of Euclidean ball in $\mathbb{R}^d$ with radius $R>0$ is upper bounded by $(1+2R/\epsilon)^d$.
\end{lemma}

\begin{lemma} \label{C.5}
(Lemma D.6 in \hyperlink{2}{[Jin et al., 2020]}). Let $\mathcal{V}$ denote a class of functions mapping from $\mathcal{S}$ to $\mathbb{R}$ with the following form
\[
  V(\cdot) = \min\{\max_{a\in \mathcal{A}}\{w^{\top}\phi(\cdot,a)+\beta\sqrt{\phi(\cdot,a)^{\top}\Lambda^{-1}\phi(\cdot,a)}\}, H\}
\]    
where the parameters $(\bf{\omega},\beta,\Lambda)$ satisfy $||\bf{w}||\leq L$, $\beta \in [0,B]$, and the minimum eigenvalue satisfies $\lambda_{\text{min}}(\Lambda) \geq \lambda$. Assume for all $(s,a)$, we have $||\phi(s,a)|| \leq 1$, and let $\mathcal{N}_{\epsilon}$ be the $\epsilon$-covering number of $\mathcal{V}$ with respect to distance $\text{dist}(V,V') = \sup_s|V(s)-V'(s)|$. Then
\[
\log \mathcal{N}_{\epsilon} \leq d\log(1+4L/\epsilon) + d^2\log[1+8d^{1/2}B^2/(\lambda \epsilon^2)]
\]
\end{lemma}

\begin{lemma} \label{C.7}
(\citet{freedman1975tail}).
Consider a real-valued martingale $\left\{Y_{k}: k=0,1,2, \cdots\right\}$ with difference sequence $\left\{X_{k}: k=0,1,2, \cdots\right\}$, which is adapted to the filtration $\{\mathcal{F}_k: k=0,1,2,\cdots\}$. Assume that the difference sequence is uniformly bounded:

$$
\left|X_{k}\right| \leq R \quad \text { almost surely for } \quad k=1,2,3, \cdots
$$
For a fixed $n \in \mathbb{N}$, assume that

$$
\sum_{k=1}^{n} \mathbb{E}\left[X_{k}^{2}|\mathcal{F}_{k-1}\right] \leq \sigma^{2}
$$

almost surely. Then for all $t \geq 0$,

$$
P\left\{\left|Y_{n}-Y_{0}\right| \geq t\right\} \leq 2 \exp \left\{-\frac{t^{2} / 2}{\sigma^{2}+R t / 3}\right\}
$$

\end{lemma}

\begin{lemma} \label{russo}
(Lemma 4 in \citet{russo2013eluder}).

Consider random variables $(\mathcal{Z}_n | n\in \mathbb{N})$ adapted to the filtration $(\mathcal{F}_n: n=0,1,\cdots)$. Assume $\mathbb{E}[\exp\{\lambda\mathcal{Z}_i\}]$ is finite for all $\lambda$. We define the conditional mean $\mu_i = \mathbb{E}[\mathcal{Z}_i|\mathcal{F}_{i-1}]$ and the conditional cumulant generating function of the centered random variable $[\mathcal{Z}_i -\mu_i]$ by $\phi_i(\lambda)  = \log \mathbb{E}\left[\exp\{\lambda ([\mathcal{Z}_i -\mu_i])\}| \mathcal{F}_{i-1}\right]$. Then we have:

For all $x \geq 0$, and $\lambda \geq 0$, 
$$
\mathbb{P} \left(\sum\limits_{i=1}^n \lambda \mathcal{Z}_i \leq x + \sum\limits_{i=1}^n \left[\lambda \mu_i + \phi_i(\lambda)\right], \forall n \in \mathbb{N}\right) \geq 1-e^{-x}
$$
\end{lemma}

\begin{lemma} \label{russo_2}
    (Lemma 2 in \citet{russo2013eluder}).
    
Let $\mathcal{F} \subset B_{\infty}(\mathcal{X},C)$ be a set of functions bounded by $C>0$. We define the width of a subset $\widetilde{\mathcal{F}} \subset \mathcal{F}$ at $x \in \mathcal{X}$ by $w_{\widetilde{\mathcal{F}}}(x) = \sup\limits_{f_1,f_2 \in \widetilde{\mathcal{F}}}\left(f_1(x) - f_2(x)\right)$. If $(\beta_t \geq 0 | t \in \mathbb{N})$ is a non-decreasing sequence, and $\{x_t\}_{t\geq 1}$ be the sequences in $\mathcal{X}$. For all $t \in \mathbb{N}$, $\mathcal{F}_t : = \left\{f\in \mathcal{F}: \sup\limits_{f_1,f_2 \in \mathcal{F}}||f_1 - f_2 ||_{2,E_t} \leq 2\sqrt{\beta_t}\right\}$, where the empirical 2-norm $||\cdot||_{2,E_t}$ is defined by $||g||_{2,E_t}^2 = \sum\limits_{k=1}^{t-1} g^2(x_k)$. Then for all $T \in \mathbb{N}$, we have

\begin{equation}
    \begin{aligned}
        \sum\limits_{t=1}^T w_{\mathcal{F}_t}(x_t) \leq \alpha + C(d \land T) + 4\sqrt{d\beta_T T}
    \end{aligned}
\end{equation}
\end{lemma}

where $d = \text{dim}_{E}(\mathcal{F},\alpha)$.

\end{document}